\documentclass[11pt]{article} 
\usepackage{etoolbox}

\usepackage{graphicx} 
\usepackage{}

\usepackage{prettyref}

\usepackage[letterpaper, left=1in, right=1in, top=1in, bottom=1in]{geometry} 
 
\pdfoutput=1

\usepackage[dvipsnames]{xcolor}
\usepackage[colorlinks=true, linkcolor=blue, citecolor=blue]{hyperref}
\usepackage{microtype} 

\usepackage{mathtools}
\usepackage{amsmath}
\usepackage{bbm}
\usepackage{amsfonts}

\usepackage{amssymb}

\usepackage{xpatch}

\usepackage{pifont}

\usepackage{array}
\usepackage{booktabs}

\usepackage{floatrow}
\newfloatcommand{capbtabbox}{table}[][\FBwidth]

\usepackage{blindtext}
\usepackage{caption}

\usepackage{natbib}
\bibpunct{(}{)}{;}{a}{,}{,}

\usepackage{amsthm}
\usepackage{mathtools}
\usepackage{amsmath}
\usepackage{bbm}
\usepackage{amsfonts}
\usepackage{amssymb}

\usepackage{MnSymbol} 

\usepackage{xpatch}

\theoremstyle{definition}  

\theoremstyle{plain}
\newtheorem{assumption}{Assumption}

\newtheorem{example}{Example}
\newtheorem{lemma}{Lemma}
\newtheorem{theorem}{Theorem}
\newtheorem{corollary}{Corollary}
\newtheorem{definition}{Definition}

\xpatchcmd{\proof}{\itshape}{\normalfont\proofnameformat}{}{}
\newcommand{\proofnameformat}{\bfseries}

\usepackage{prettyref}
\newcommand{\pref}[1]{\prettyref{#1}}

\newcommand{\savehyperref}[2]{\texorpdfstring{\hyperref[#1]{#2}}{#2}}
\newrefformat{eq}{\savehyperref{#1}{\textup{(\ref*{#1})}}}
\newrefformat{ineq}{\savehyperref{#1}{\textup{(\ref*{#1})}}}
\newrefformat{eqn}{\savehyperref{#1}{Equation~\ref*{#1}}}
\newrefformat{obs}{\savehyperref{#1}{Observation~\ref*{#1}}}
\newrefformat{lem}{\savehyperref{#1}{Lemma~\ref*{#1}}}
\newrefformat{def}{\savehyperref{#1}{Definition~\ref*{#1}}}
\newrefformat{thm}{\savehyperref{#1}{Theorem~\ref*{#1}}}
\newrefformat{fact}{\savehyperref{#1}{Fact~\ref*{#1}}}
\newrefformat{tab}{\savehyperref{#1}{Table~\ref*{#1}}}
\newrefformat{table}{\savehyperref{#1}{Table~\ref*{#1}}}
\newrefformat{line}{\savehyperref{#1}{Line~\ref*{#1}}}
\newrefformat{corr}{\savehyperref{#1}{Corollary~\ref*{#1}}}
\newrefformat{cor}{\savehyperref{#1}{Corollary~\ref*{#1}}}
\newrefformat{sec}{\savehyperref{#1}{Section~\ref*{#1}}}
\newrefformat{app}{\savehyperref{#1}{Appendix~\ref*{#1}}}
\newrefformat{ass}{\savehyperref{#1}{Assumption~\ref*{#1}}}
\newrefformat{ex}{\savehyperref{#1}{Example~\ref*{#1}}}
\newrefformat{fig}{\savehyperref{#1}{Figure~\ref*{#1}}}
\newrefformat{alg}{\savehyperref{#1}{Algorithm~\ref*{#1}}}
\newrefformat{rem}{\savehyperref{#1}{Remark~\ref*{#1}}}
\newrefformat{conj}{\savehyperref{#1}{Conjecture~\ref*{#1}}}
\newrefformat{prop}{\savehyperref{#1}{Proposition~\ref*{#1}}}
\newrefformat{proto}{\savehyperref{#1}{Protocol~\ref*{#1}}}
\newrefformat{prob}{\savehyperref{#1}{Problem~\ref*{#1}}}
\newrefformat{claim}{\savehyperref{#1}{Claim~\ref*{#1}}}
\newrefformat{tbl}{\savehyperref{#1}{Table~\ref*{#1}}}

\newcommand\numberthis{\addtocounter{equation}{1}\tag{\theequation}}
\allowdisplaybreaks

\usepackage{}

\usepackage{mathtools}
\usepackage{amsmath}
\usepackage{bbm}
\usepackage{amsfonts}
\usepackage{amssymb}

\DeclarePairedDelimiter{\abs}{\lvert}{\rvert} 
\DeclarePairedDelimiter{\brk}{[}{]}
\DeclarePairedDelimiter{\crl}{\{}{\}}
\DeclarePairedDelimiter{\prn}{(}{)}
\DeclarePairedDelimiter{\nrm}{\|}{\|}
\DeclarePairedDelimiter{\tri}{\langle}{\rangle}

\let\Pr\undefined

\DeclareMathOperator{\En}{\mathbb{E}}

\DeclareMathOperator{\Pr}{Pr}

\DeclareMathOperator*{\argmin}{argmin} 
\DeclareMathOperator*{\argmax}{argmax}

\newcommand{\indicator}[1]{\mathbbm{1}\crl*{#1}}    

\newcommand{\eps}{\epsilon}

\newcommand{\ldef}{\vcentcolon=}

\newcommand{\wt}[1]{\widetilde{#1}}
\newcommand{\wh}[1]{\widehat{#1}}
\newcommand{\mb}[1]{\boldsymbol{#1}}

\def\ddefloop#1{\ifx\ddefloop#1\else\ddef{#1}\expandafter\ddefloop\fi}
\def\ddef#1{\expandafter\def\csname bb#1\endcsname{\ensuremath{\mathbb{#1}}}}
\ddefloop ABCDEFGHIJKLMNOPQRSTUVWXYZ\ddefloop
\def\ddefloop#1{\ifx\ddefloop#1\else\ddef{#1}\expandafter\ddefloop\fi}
\def\ddef#1{\expandafter\def\csname b#1\endcsname{\ensuremath{\mathbf{#1}}}}
\ddefloop ABCDEFGHIJKLMNOPQRSTUVWXYZ\ddefloop
\def\ddef#1{\expandafter\def\csname c#1\endcsname{\ensuremath{\mathcal{#1}}}}
\ddefloop ABCDEFGHIJKLMNOPQRSTUVWXYZ\ddefloop
\def\ddef#1{\expandafter\def\csname h#1\endcsname{\ensuremath{\widehat{#1}}}}
\ddefloop ABCDEFGHIJKLMNOPQRSTUVWXYZabcdefghijklmnopqrsuvwxyz\ddefloop    
\def\ddef#1{\expandafter\def\csname hc#1\endcsname{\ensuremath{\widehat{\mathcal{#1}}}}}
\ddefloop ABCDEFGHIJKLMNOPQRSTUVWXYZ\ddefloop
\def\ddef#1{\expandafter\def\csname t#1\endcsname{\ensuremath{\widetilde{#1}}}}
\ddefloop ABCDEFGHIJKLMNOPQRSTUVWXYZ\ddefloop
\def\ddef#1{\expandafter\def\csname tc#1\endcsname{\ensuremath{\widetilde{\mathcal{#1}}}}}
\ddefloop ABCDEFGHIJKLMNOPQRSTUVWXYZ\ddefloop

\newcommand{\KL}[2]{\mathrm{KL}{\prn*{#1 \| #2}}}

\newcommand{\diag}{\textrm{diag}}
\newcommand{\rank}{\mathrm{rank}}

\usepackage{tikz}

\newcommand{\proman}[1]{\prn*{\romannumeral #1}}
\newcommand{\overleq}[1]{\overset{ #1}{\leq{}}}
\newcommand{\overgeq}[1]{\overset{#1}{\geq{}}}
\newcommand{\overeq}[1]{\overset{#1}{=}}

\usepackage{mathrsfs} 
\newcommand{\ind}[1]{^{#1}} 

\providecommand{\uniform}{\text{Uniform}}

\newcommand{\transfer}{P} 
\newcommand{\valestimate}{\textsf{ValEstimate}}
\newcommand{\adavalestimate}{\textsf{AdaValEstimate}} 
\newcommand{\ignore}[1]{}
\renewcommand{\epsilon}{\varepsilon}

\newcommand{\mdpfamily}{\mathscr{M}}
\newcommand{\obsmap}{\phi} 
\newcommand{\obserr}[1]{\Delta(#1)} 

\usepackage[suppress]{color-edits} 
\addauthor{as}{red}  
\addauthor{ks}{blue} 
\addauthor{cd}{purple}
\addauthor{ym}{black}

\usepackage{longtable} 
\usepackage{import}
\usepackage{enumitem}
\usepackage{algorithm}
\usepackage[noend]{algpseudocode} 

\algblockdefx[class]{Class}{EndClass}[1]{\textbf{object} \textsc{#1}:}{\textbf{end object}}
\makeatletter 
\ifthenelse{\equal{\ALG@noend}{t}}
  {\algtext*{EndClass}}
  {}
\makeatother
\usepackage{amsmath}
\usepackage{makecell}
\usepackage{enumitem}
\usepackage{xspace}
\usepackage[scaled=.9]{helvet}
\usepackage{booktabs}
\usepackage[export]{adjustbox}
\usepackage{times}

\usepackage{amsmath, amsthm, thm-restate}

\usepackage{mathtools} 

\title{Agnostic Reinforcement Learning with \\ 
Low-Rank MDPs and Rich Observations} 

\usepackage{authblk}

\author[2]{Christoph Dann$^{1}$, Yishay Mansour$^{1, 3}$, Mehryar Mohri$^{1,4}$ \\ Ayush Sekhari$^{2}$  and  Karthik Sridharan}
\affil[1]{Google Research} 
\affil[2]{Cornell University}
\affil[3]{Tel Aviv University}
\affil[4]{Courant Institute of Mathematical Sciences}

\date{} 

\usepackage[toc,page,header]{appendix} 
\usepackage{mathrsfs} 

\begin{document}

\maketitle 
\begin{abstract}
There have been many recent advances on provably efficient Reinforcement Learning (RL) in problems with rich observation spaces. However, all these works share a strong realizability assumption about the optimal value function of the true MDP. Such realizability assumptions are often too strong to hold in practice. In this work, we consider the more realistic setting of agnostic RL with rich observation spaces and a fixed class of policies $\Pi$ that may not contain any near-optimal policy. We provide an algorithm for this setting whose error is bounded in terms of the rank $d$ of the underlying MDP.  Specifically, our algorithm enjoys a sample complexity bound of $\wt{O}\left((H^{4d} K^{3d} \log |\Pi|)/\epsilon^2\right)$ where $H$ is the length of episodes, $K$ is the number of actions and $\epsilon>0$ is the desired sub-optimality.  We also provide a nearly matching lower bound for this agnostic setting that shows that the exponential dependence on rank is unavoidable, without further assumptions. 
\end{abstract} 

\let\thefootnote\relax\footnotetext{Correspondence  to: Ayush Sekhari $<$as3663@cornell.edu$>$}

\section{Introduction}
Reinforcement Learning (RL) has achieved several remarkable empirical successes in the last decade, which include playing Atari 2600 video games at superhuman levels \citep{mnih2015human}, AlphaGo or AlphaGo Zero surpassing champions in Go  \citep{silver2018general}, AlphaStar's victory over top-ranked professional players in StarCraft \citep{vinyals2019grandmaster}, or practical self-driving cars. These applications all correspond to the setting of rich observations, where the state space is very large and where observations may be images, text or audio data.
In contrast, most provably efficient RL algorithms are still limited to the classical tabular setting where the state space is small \citep{kearns2002near, brafman2002r,azar2017minimax, dann2019policy} and do not scale to the rich observation setting. 

To derive guarantees for large state spaces, much of the existing work in
RL theory relies on a \emph{realizability} and 
a \emph{low-rank} assumption \citep{krishnamurthy2016pac, jiang2017contextual,dann2018oracle,du2019provably,misra2020kinematic, agarwal2020flambe}.
Different notions of rank have been adopted in the literature, including that of a low-rank transition matrix \citep{jin2019learning}, a low Bellman rank \citep{jiang2017contextual}, Wittness rank \citep{sun2019model}, Eluder dimension \citep{osband2014model}, Bellman-Eluder dimension \citep{jin2021bellman}, or bilinear classes \citep{du2021bilinear}.  These studies also show that learning without any such structural assumptions requires a sample size that grows exponentially in the time horizon of the MDP \citep{dann2015sample,krishnamurthy2016pac, du2019good}. The choice of the
most suitable and most general notion of rank is the topic of much active 
research in RL theory.

In comparison, the realizability assumption has received much less attention. This is the strong premise that the optimal value function belongs to the class of functions considered, which typically does not hold in practice.  In many applications, the optimal value function $Q^\star$ is highly complex and we cannot hope to accurately approximate it in the absence of some strong domain knowledge.  Can we relax the realizability assumption in RL?

Value-function realizability can be viewed as the analogue of the PAC-realizability assumption in classical statistical learning theory. That assumption rarely holds, which has motivated the development and analysis of numerous algorithms for the agnostic PAC learnability model. Those algorithms provably learn to predict as well as the best predictor in the given function class, independently of whether the Bayes predictor belongs to the class. The counterparts of such results in reinforcement learning are mostly unavailable, which prompts the following question: Can we derive a theory of agnostic reinforcement learning?

Here, we precisely initiate that study. In this agnostic setting, we adopt common structural assumptions, e.g.\ small rank of the transition matrix, but seek to learn to perform as well as the best policy in the given policy class, independently of how close this class represents the Bellman-optimal policy.  Specifically, we study agnostic Reinforcement Learning (RL) with a fixed policy class $\Pi$ in the episodic MDPs with rich observations. Provably sample-efficient algorithms for agnostic RL would be highly desirable but it is still unknown to what degree learning is possible in this setting. Our work provides new insights about learnability with structural assumptions in the absence of (approximate) realizability in RL.

Agnostic RL without any additional structural assumptions has been considered in the past. By evaluating each policy in the class individually, one can easily obtain a sample complexity upper bound of $O(|\Pi| / \epsilon^2)$. \citet{kearns2000approximate} also showed that an upper bound of $(K^H \log |\Pi|) / \epsilon^2$ is possible, where $K$ is the number of actions and $H$ is the time horizon. However, as discussed in prior work such as \citep{krishnamurthy2016pac}, bounds of this form are rather unsatisfactory as one of them admits a linear dependence on the size of the function class, which is prohibitively large, and the other one admits an exponential dependence on the length of the episodes $H$, which is typically long. Using existing constructions, one can derive a lower bound on the sample complexity of the form $\min\{|\Pi|, K^H\} / {\epsilon^2}$ in the rich observation setting. 
This further justifies our adoption of rank as a natural structural assumption. 
\ignore{
 that
the rank $d$ of the state transition matrix induced by any policy is small. 
We view this as 
and ask whether this alone is sufficient to learn more efficiently in the agnostic setting.}

\paragraph{Our Contributions:}  
The following highlights our main technical contributions, where
$d$ is the rank of the state transition matrix induced by any policy in the class $\Pi$, 
and is assumed to be small. 
\begin{enumerate}[label=$\bullet$, itemsep=0mm, leftmargin=8mm] 
\item We provide a uniform exploration-based algorithm that can find an $\epsilon$-sub-optimal policy w.r.t. the policy class $\Pi$ after collecting $O\prn{(H K/ d)^{4d} \log(d \abs{\Pi}) / \epsilon^2}$ samples in the MDP.
This bound shows that one can achieve a sample complexity that is polynomial in both $H$ and $\log |\Pi|$, while being exponential in rank $d$ only (which we assume is small). 

In addition to the sample complexity bound obtained here, the algorithmic techniques itself might be of independent interest and useful beyond this work. The algorithm is based on showing that for every policy, the expected rewards follows an autoregressive model of degree $d$. Thus obtaining samples of $O(d)$-length paths for a policy we show that one can extrapolate expected rewards for the entire episode.

\item We complement this upper bound with a sample complexity lower bound of $\Omega \prn[\big]{  \left(H / d\right)^{d /2 }  / \epsilon^2}$ (when $K=2$), thereby showing that the $H^{O(d)}$ term in the upper bound is unavoidable. The lower bound also highlights which structures in the policy class induce the $H^{O(d)}$ terms thus shedding light on what structural assumptions could help alleviate the exponential dependence on the rank.

\item Finally, we seek to improve upon the $H^d$ term and
provide an adaptive algorithm that admits a sample complexity that depends on the eigenspectrum of the transition matrix of the MDP; while in the worst case that bound matches the above one, it provides a significantly better guarantee when the eigenspectrum is more favorable.

\end{enumerate} 

However, we view the main benefit of our work to be the initiation of the study of agnostic reinforcement learning and the presentation of an in-depth analysis of a natural structural assumption within that setting. This can form the basis for future research in this domain with alternative and perhaps more favorable rank-type assumptions.

\section{Problem Setup}  
We consider an episodic Markov decision process with episode length $H \in \bbN$, observation space $\cX$ and action space $\cA \ldef{} \crl{1, \ldots, K}$. For ease of exposition, we assume that the observation space $\cX$ is finite (albeit extremely large), but our results can be readily extended to countably infinite and possibly uncountably infinite observation spaces.
Each episode is a sequence $((x_1, a_1, r_1), (x_2, a_2, r_2), \dots, (x_{H}, a_H, r_H)) \in (\cX \times \cA \times \bbR)^H$, where the initial observation $x_0$ is drawn from the initial distribution $\mu_0 \in \Delta(\cX)$, the actions are generated by the learning agent and all the following observations are sampled from the transition kernel $x_{h+1} \sim T( \cdot | x_h, a_h) \in \Delta(\cX)$ that depends on the previous observation and action. Finally, the rewards $r_h$ are drawn from a sub-Gaussian distribution with mean $r(x_h, a_h)$ where $r \colon \cX \times \cA \mapsto [0,1]$. The learning agent does not know the transition kernel $T$, the initial distribution $\mu_0$, or the reward function $r$. 

In our setting, the agent is given a policy class $\Pi \subseteq \crl*{\cX \mapsto \Delta(\cA)}$ consisting of  policies that map observations to distributions over the actions $\cA$.  
For any policy $\pi \in \Pi$, we denote by $T \ind{\pi} \in \bbR^{\cX \times \cX}$, the transition matrix induced by $\pi$, i.e., for any $x, x' \in \cX$, $$ \brk{T\ind{\pi}}_{(x', x)} = \En_{a \sim \pi(x)} T(x' \mid x, a).$$ 

\begin{assumption}[Low-rank transition] 
\label{ass:low_rank}
There exists $d \in \bbN$ such that $\rank(T\ind{\pi}) \leq d$ for all $\pi \in \Pi$. \end{assumption} 

For the main part of the paper, we assume that the learner knows the value of $d$, but later extend our results to the case where $d$ is unknown. We define $\lambda \ind{\pi} = (\lambda_1 \ind{\pi}, \ldots, \lambda_d \ind{\pi})^\top \in \bbC^d$ to denote the   eigenvalues of the transition matrix $T \ind{\pi}$ with rank at most $d$. Without any loss of generality, assume that $\abs{\lambda_1\ind{\pi}} \geq \abs{\lambda_2\ind{\pi}} \geq \ldots \geq \abs{\lambda_d\ind{\pi}}$. 

We denote by $\bbP\ind{\pi}$ the distribution over episodes when following policy $\pi$ and by $\En\ind{\pi}$ its expectation. We call the expected rewards obtained at time $h$ by policy $\pi$ \textbf{expected policy rewards}:
\begin{align*}
R\ind{\pi}_h &\ldef{} \En\ind{\pi} \brk*{r(x_h, a_h)}. \numberthis \label{eq:policy_step_rewards} 
\end{align*} 

The value function of $\pi$ at time $h$ is given by 
$
V\ind{\pi}_h(x) = \En^{\pi}\brk*{\sum_{h' = h}^H r(x_h, a_h) \mid x_h = x}
$.
Further, when using $V\ind{\pi}$ without a time index and arguments, we mean the value or expected $H$-step return:   
\begin{align*}
V\ind{\pi} &\ldef{} \En\brk*{V\ind{\pi}_0(x_0)} =  \En\ind{\pi} \brk[\Big]{ \sum_{h=1}^H r(x_h, a_h) } = \sum_{h=1}^H R\ind{\pi}_h,  \numberthis \label{eq:value_decomposition} 
\end{align*} 
the value function averaged over initial observations.

\paragraph{Learning objective.}  The goal of the learner is to return a policy $\wt \pi$, after interacting with the MDP for $n$ episodes of length $H$, such that the value of the returned policy is as close as possible to the value of the best policy in $\Pi$, that is, 
\begin{align*}
    V\ind{\wt \pi} \geq \max_{\pi \in \Pi} ~ V\ind{\pi} - \epsilon, 
\end{align*} 
where the error $\epsilon$ is a small as possible and may depend on $n$, the policy class $\Pi$ and the MDP. 
  
\section{Related work} \label{sec:related}
We give a brief overview of the most closely related works here, and defer a more detailed discussion to  \pref{app:related_work}. 

Recently there has been great interest in designing RL algorithms with general function approximation \citep{jiang2017contextual, dann2018oracle,  sun2019model, du2019provably, wang2020reinforcement, du2021bilinear}. In particular, \cite{jiang2017contextual} introduced the notion of Bellman rank, a measure of complexity that depends on the underlying environment and the value function class $\cF$, and provide statistically efficient algorithms for learning problems for which Bellman rank is bounded. This was later extended to model-based algorithms by \cite{sun2019model}. While these algorithms work across a variety of problem settings, their sample complexity scales with $\log(\abs{\cF})$. Furthermore, these algorithms also require the optimal value function $f^*$ to be realized in $\cF$. In our work, we do not assume that the learner has access to a value function class $\cF$. In fact, given a value function class $\cF$, we can construct the policy class $\Pi_{\cF}$ that corresponds to greedy policies induced by the class $\cF$. However, given just a policy class $\Pi$, one cannot construct a value function class, without additional knowledge of the underlying dynamics. 

 Our \pref{ass:low_rank} implies that for any policy $\pi \in \Pi$, the transition dynamics exhibits a low-rank decomposition with dimension $d$, that is  
$T\ind{\pi}(x' | x) = \langle \phi\ind{\pi}(x), \psi\ind{\pi}(x') \rangle$,  
for some d-dimensional feature maps $\phi\ind{\pi}, \psi\ind{\pi} \colon \cX \mapsto \bbR^d$. Low rank MDPs and linear transition models have recently gained a lot of  attention in the RL literature \citep{yang2020reinforcement, jin2020provably, modi2020sample, wang2019optimism}. The works most closely related to our setup are those of  \cite{jin2020provably} and \cite{yang2020reinforcement}, who give algorithms to find optimal policy in low rank MDPs with known feature maps $\phi$. Similarly, the other algorithms also assume that the learner either observes the feature $\phi(x)$, or the feature $\psi(x)$. \citet{agarwal2020flambe} and \citet{modi2021model} learn under weaker assumptions and only assume that the learner has access to a function class that realizes $\phi$.
However, in our setup, the learner neither observes the features $\phi \ind{\pi}, \psi \ind{\pi}$ nor has access to a realizable function class for them, and thus these methods are not applicable. 

Several of the works mentioned above recognize the issue of a strict realizability assumption and provide results only when the function class contains a good approximation to the optimal value function of model. However, the goal in our agnostic setting is more ambitious. We would like to find a policy that can compete with the best policy in the given class $\Pi$, independent of how close the best return in the class $\max_{\pi \in \Pi} V^\pi$ is to the return of the optimal policy $V^{\pi^\star}$ for that MDP.

There have also been several approaches for provably efficient RL with non-parametric function classes \citep{yang2020function, long20212, shah2020sample}. However, these approaches still aim to learn the optimal value function and their regret necessarily scales with the complexity of the optimal value function in the RKHS which can be very high. Instead, in our agnostic setting we would like to be able to quickly identify the best policy from the given policy class with low complexity containing a good but not necessarily optimal policy.

\section{Upper bound}  \label{sec:basic_upper_bounds}

In this section, we describe our main algorithm for finding a policy that is close to 
the best-in-class in $\Pi$. This algorithm presented in \pref{alg:policy_search}, is an instance of policy search with uniform exploration. Specifically, we first collect a dataset $\cD$ of $n$ episodes by picking actions uniformly at random and subsequently use those episodes to estimate the value of each policy in $\Pi$. The algorithm then simply returns the policy $\wt \pi$ with the highest estimated value. 

Our main technical innovation is a new estimation procedure for policy values in \pref{alg:value_estimation_extrapolation} that leverages the low-rank structure of the transition matrix. A straightforward 
way to estimate the policy value is to take the sum of the rewards on average across all episodes where all actions are consistent with the policy \citep{kearns2000approximate}. Unfortunately, this rejection sampling approach yields an error of $\Omega(\sqrt{K^H})$. Instead, our procedure only estimates the expected policy rewards for the first $3d$ steps. Specifically, when invoked with a given policy $\pi$, $\valestimate$ estimates the expected rewards for that policy by considering the subset of trajectories in $\cD$ where $\pi$ agrees with the chosen action till the first 3d steps, and by averaging the observed rewards in those trajectories. $\valestimate$ then predicts the future expected rewards for that policy by extrapolating these $3d$ estimated expected rewards. The prediction is computed by recognizing that the expected rewards for any policy $\pi$ satisfy an autoregressive relation of order $d$ as shown in \pref{lem:basic_recurrence_relation}.

In order to find the coefficients of this autoregression, $\valestimate$ computes $\wh \lambda \in \bbC^d$ by solving the optimization problem \pref{eq:optimization_problem}, where the coefficient $\alpha_k(\lambda)$ are the sum of degree $k$ monomials:
\begin{align*} \small  
   \alpha_{k}(\lambda) = \sum_{x \in \crl{0, 1}^d \text{~s.t.~} \nrm{x}_1 = k}  \lambda_1^{x_1}\lambda_2^{x_2} \ldots  \lambda_d^{x_d}.   \numberthis \label{eq:alpha_k_coeffs_main}  
\end{align*} 
After estimating $\wh \lambda$, $\valestimate$ then predicts the expected rewards for all future time steps for the policy $\pi$ by unfolding the autoregression model whose coefficients are given by $\alpha_k(\wh \lambda)$. The estimate for the value of the given policy $\pi$, denoted by $\wt V\ind{\pi}$, is then computed as the sum of the predicted expected rewards for $H$ steps.  

Finally, \pref{alg:policy_search}  returns the policy $\wt \pi$ whose estimated value function is highest amongst all the policies in $\Pi$. The following theorem characterizes the performance guarantee for the policy $\wt \pi$ returned by our algorithm. 
\begin{theorem}[Main Theorem] 
\label{thm:basic_sample_complexity_bound}
For a given $\delta \in (0,1)$, $d$-rank MDP, horizon $H \geq d$ and a finite policy class $\Pi$, after collecting $n$ episodes,  
\pref{alg:policy_search} 
returns a policy $\wt \pi$ that with probability at least $1 - \delta$ admits the following guarantee: 
\begin{align*} 
  V\ind{\wt \pi} \geq \max_{\pi \in \Pi}V \ind{\pi} -  O\prn[\Big]{d^3 \cdot  \prn[\Big]{\frac{H}{d}}^{2d} \sqrt{ \frac{K^{3d} \log(6 \Pi d/ \delta)}{n}}}.  
\end{align*} 
\end{theorem} 

\pref{thm:basic_sample_complexity_bound} implies that \pref{alg:policy_search}  can find an $\varepsilon$-optimal policy with probability $1 - \delta$ as long as the number of samples $n$ is larger than
\begin{align*}
 n \gtrsim   \prn[\Big]{\frac{H}{d}}^{4d} \frac{K^{3d} \log(6d \abs*{\Pi} / \delta)}{\epsilon^2}. \end{align*}

\begin{algorithm}[h]
\caption{Policy search algorithm}  
\begin{algorithmic}[1]
	\Require horizon $H$, rank $d$,  number of episodes $n$, finite policy class $\Pi$ 
	\State Collect the dataset $\cD = \crl{\prn{x\ind{t}_h, a\ind{t}_h, r\ind{t}_h}_{h=1}^H}_{t = 1}^n$ of $n$ trajectories by drawing actions from $\uniform(\cA)$.
	\For {\text{policy} $\pi \in \Pi$} 
\State \label{line:call_val_estimate}Estimate $\wt V\ind{\pi}$ by calling \valestimate($H, d, \cD, \pi$).
\EndFor  
\State\label{line:choose_optimal_policy}\textbf{Return: } policy $\wt \pi$ with best estimated value, i.e. $\wt \pi \in \argmax_{\pi \in \Pi} \wt V\ind{\pi}$.  
\end{algorithmic} 
\label{alg:policy_search}  
\end{algorithm}

\begin{algorithm}[h] 
\caption{Value estimation by autoregressive extrapolation}
\begin{algorithmic}[1]
\Function{\valestimate}{$H, d, \cD, \pi$}: 

\For {\text{time step} $h= 1, \ldots, 3d$}   
\State\label{line:reward_estimation1} Estimate expected rewards by importance sampling 
\vspace{-1mm}$$\wh R_h = \frac{1}{n} \sum_{t =1}^n r_h\ind{t} \prod_{h' \leq h} \prn*{K\indicator{\pi(x\ind{t}_{h'}) = a_{h'}\ind{t}}}$$   
\EndFor

\State\label{line:eig_optimization1}Estimate eigenvalues of the autoregression by solving the optimization problem: 
\begin{align*} 
   (\wh \lambda, \wh \Delta) \leftarrow \argmin_{\lambda \in \mathbb C^d, \Delta \in \bbR} &~ \Delta  \numberthis  \label{eq:optimization_problem}  ~~~~~~~~ \text{s.t.} ~~ |\lambda_{k}| \leq 1 ~~~~~ \textrm{for} ~ k =1, \dots, d \\
    \textrm{ and } ~~~&  \abs[\Big]{\sum_{k=1}^d (-1)^{k+1} \alpha_{k}(\lambda) \wh R_{h - k} - \wh R_{h} } \leq \Delta & \textrm{for } h = d + 1, \dots, 3d   
\end{align*} 
\State\label{line:reward_prediction1}Predict $\wt R_{h}$ as:  
{\small \begin{align*}
 	\wt R_{h} = 
\begin{cases}
\wh R_{h} & \text{for ~} 1 \leq h \leq d \\ 
\sum_{k=1}^d (-1)^{k+1} \alpha_{k}(\wh \lambda) \wt R_{h - k} & \text{for ~} d + 1 \leq h \leq H
\end{cases}.  \numberthis \label{eq:reward_prediction_autoregression} 
\end{align*} }
\State\label{line:value_prediction1}\textbf{return:} Estimate of the value $\wt V = \sum_{h=1}^H \wt R_h$. 
\EndFunction 
\end{algorithmic} 
\label{alg:value_estimation_extrapolation} 
\end{algorithm}

The key idea used in $\valestimate$, is that for any policy $\pi$ for which $\rank(T\ind{\pi}) \leq d$, the expected rewards satisfy an auto-regression of order $d$. The following lemma formalize this idea. 

\begin{restatable}[Autoregression on expected rewards]{lemma}{autoregressionlemma}
\label{lem:basic_recurrence_relation} 
Let $\pi$ be any policy for which the transition matrix  $T\ind{\pi}$ has rank at most $d$. Then, for any time step $h \geq d +1$, the expected reward for policy $\pi$ at time step $h$, denoted by $R\ind{\pi}_h$, satisfies the auto-regression 
\begin{align*} 
R\ind{\pi}_{h} = \sum_{k=1}^d (-1)^{k+1} \alpha_k(\lambda \ind{\pi}) R\ind{\pi}_{h-k},  \numberthis \label{eq:basic_autoregression_main} 
\end{align*} 
where $\lambda\ind{\pi} \in \bbC^{d}$ denotes the set of eigenvalues of the matrix $T\ind{\pi}$, and $\alpha_k(\lambda\ind{\pi})$ is as defined in \pref{eq:alpha_k_coeffs_main}. 
\end{restatable}

We defer the proof of \pref{lem:basic_recurrence_relation} to \pref{app:basic_recurrence_relation}. The proof uses the fact that for any policy $\pi \in \Pi$, the distribution over observations at time step $h$ is given by $\mu\ind{\pi}_h = (T\ind{\pi})^h \mu_0$, where $\mu_0$ denotes the distribution over the observation space at initialization. If $\rank(T\ind{\pi}) \leq d$, an application of the Cayley-Hamilton theorem implies that we can write $(T\ind{\pi})^h$ as a linear combination of $\prn{(T\ind{\pi})^{h-1}, \ldots, (T\ind{\pi})^{h-d}}$. This implies that $\mu\ind{\pi}_h$, and thus the expected rewards $R\ind{\pi}_h$, satisfy an auto-regression of order $d$. While the expected rewards $R\ind{\pi}_h$ satisfy an auto-regression for every policy $\pi$, note that we cannot hope for a similar relation between the instantaneous rewards  $r_h(s\ind{\pi}_h, \pi(s_h))$ observed when taking actions according to $\pi$. 

The following result shows that we can simultaneously estimate the expected rewards for the first $3d$ steps for every policy $\pi \in \Pi$.  Let $\cD$ be a dataset of $n$ episodes in the MDP collected by drawing actions uniformly at random from $\cA$. Then, for any policy $\pi$, there are approximately $n / K^{3d}$ episodes in $\cD$ where the actions taken during the first $3d$ time steps matches the predictions of $\pi$ on those observations. We compute $\wh R\ind{\pi}_h$ as the empirical average of the $h$th step reward in the corresponding $n/K^{3d}$ episodes that match with $\pi$ for the first $3d$ steps. 

\begin{restatable}[Importance sampling]{lemma}{ISbound} \label{lem:ss_ub} 
For any $\delta \in (0, 1)$, with probability at least $1-\delta$, for any policy $\pi \in \Pi$ and time step $h \in [3d]$, the estimates $\wh R\ind{\pi}_{h}$ computed using importance sampling satisfy the error bound 
\begin{align*} 
\abs{\wh R\ind{\pi}_{h} - R\ind{\pi}_h} &\leq { \sqrt{\frac{2 K^{3d} \log(6 d \abs*{\Pi}/ \delta)}{n}} + \frac{2K^{3d} \log(6 d \abs*{\Pi} / \delta)}{n}}. 
\end{align*} 
\end{restatable}  

For a given policy $\pi$, if we had access to the expected rewards $\crl{R\ind{\pi}_1, \ldots, R\ind{\pi}_{d}}$, we could have solved for the coefficients $\alpha_k(\lambda)$ exactly. However, we only have access to the empirical estimates $\crl{\wh R_1, \ldots, \wh R_d}$ of the expected rewards, and thus we compute the coefficients $\alpha_k(\wh \lambda)$ by solving the optimization problem in \pref{eq:optimization_problem}. We predict the future expected rewards by extrapolating using $\alpha_k(\wh \lambda)$. The following lemma bounds the error propagated due to this mismatch in our estimation. 

\begin{restatable}[Error propagation bound]{lemma}{firstEPbound} \label{lem:basic_ep_bound}
Let $\lambda, \wh \lambda \in \bbC^d$ be such that $\max\crl{\abs{\lambda_1}, \abs{\wh \lambda_1}} \leq 1$. 
Further, with the initial values  $R_1, \ldots, R_d$ and $\wt R_1, \ldots, \wt R_{d}$, let the sequence $\crl{R_h}$ and $\crl{\wt R_h}$ be given by 
\begin{align*}
	R_h = \sum_{k=1}^d (-1)^{k+1} \alpha_{k}(\lambda) R_{h - k } \qquad \text{and} \qquad  \wt R_h = \sum_{k=1}^d (-1)^{k+1} \alpha_{k}(\wh \lambda) \wt R_{h - k},  
\end{align*}
where the coefficients $\alpha_{k}(\lambda)$ and $\alpha_{k}(\wh \lambda)$ are define as in \pref{eq:alpha_k_coeffs_main}. Then, for all $h \geq 3d + 1$,  
\begin{align*} 
	\abs{\wt R_h - R_h} &\leq 2d \cdot  \prn[\Big]{\frac{16 e h}{d}}^{2d} \cdot \max_{h' \leq 3d} ~ \abs{ R_{h'} - \wt R_{h'}}. 
\end{align*} 
\end{restatable} 

We defer the proof of \pref{lem:basic_ep_bound} to \pref{app:appendix_basic_ep_bound}. The proof of \pref{thm:basic_sample_complexity_bound} follows from combining the above three technical results. \pref{lem:basic_recurrence_relation} suggests that for any policy $\pi \in \Pi$ for which $\rank(T\ind{\pi}) \leq d$, the expected per step rewards satisfy an auto-regression of order at most $d$. The error propagation bound in \pref{lem:basic_ep_bound} and the bound on the estimation of the expected rewards for the first $3d$ steps given in \pref{lem:basic_ep_bound} implies that, for every policy $\pi \in \Pi$, the estimated value $\wt V\ind{\pi}$ is close to the true value $V \ind{\pi}$. Specifically, the estimation error in the value of every policy in $\pi$ is bounded by $\wt{O}((H/d)^{2d} \sqrt{K^{3d} \log(\abs{\Pi})  / n})$. Thus, when $n = \wt O((H/d)^{4d} K^{3d} / \epsilon^2)$, we have that $\abs{\wt V \ind{\pi} - V \ind{\pi}} \leq \epsilon$ for every policy $\pi \in \Pi$ simultaneously. This implies that the returned policy , that maximize the estimated value $\wt V \ind{\pi}$, is  $2\epsilon$ sub-optimal w.r.t. the best policy in $\Pi$. We defer full details of the proof of \pref{thm:basic_sample_complexity_bound} to \pref{app:basic_sample_complexity_bound}.  
\section{Lower Bound} \label{sec:lower_bounds} 
After presenting an algorithm with sample-complexity bound of $\wt O \prn{{(H/d)^{2d} K^{3d} / \epsilon^2}}$, we now show through a lower-bound that the dependency on $H$ and $d$ cannot be improved significantly:

\begin{theorem}[Lower bound] 
\label{thm:basic_lower_bound}
Let $\epsilon \in (0, 1/26)$, $\delta \in (0, 1/2)$, $d \geq 4$, $K=2$ and $H \geq 219 d$. There exists a  policy class of size $(H/d)^d$ and a family of MDPs with rank at most $\Theta(d)$, finite observation space, horizon $H$ and two actions such that the optimal policy for each MDP in the family is contained in the policy class and the following holds:  Any algorithm that returns an $\epsilon$-optimal policy, with probability at least $1 - \delta$, for every MDP in this family has to collect at least 
\begin{align*}
    \Omega \prn[\Big]{\frac{1}{H \epsilon^2} \prn[\Big]{\frac{H}{41 d }}^{d/2} \log\prn[\Big]{\frac{1}{2 \delta} }}~.
\end{align*} 
episodes in expectation in some MDP in this family. 
\end{theorem}
The above lower bound shows that an exponential dependency on $d$ in the form of $(H/d)^d$ is unavoidable, even when a realizable policy class with $\pi^* \in \Pi$ and moderate size $\log|\Pi| = d \log(H/d)$ is given to the learner. We now provide a brief description of the problem class used in the proof of our lower bound but defer details of our construction and the proof to \pref{app:lower_bounds}.

The Markov decision processes in the proof of our lower bound bear some similarity to the so-called \emph{combination lock} constructions used in prior works \citep{krishnamurthy2016pac,du2019good}, where the algorithm only receives positive feedback after playing a certain sequence of actions. Modelling a combination lock typically requires $K^H$ states in MDPs and $\Theta(H)$ latent states in POMDP. In contrast, our contextual version of a combination lock uses a low-rank MDP  
with very large observation space but where the transition dynamics are governed by $\Theta(d) \ll H$ hidden states (and thus the rank is $\Theta(d)$). 
The latent state structure is shown in \pref{fig:lower_bound_structure}. The agent starts at the top left latent state and always progresses with probability $p = d / H$ to the right. As long as it chooses good actions (blue edges), it progresses in the top chain where it will eventually reach state $(d, g)$ with constant probability and receive a reward of $1$ with probability $1/2 + \epsilon$. If at any time before reaching state $(d, g)$, it chooses a bad action (red edges), then it moves to the lower chain where it eventually has a $1/2$ chance of receiving a reward of $1$.

If the latent states $s \in \cS$ were directly observable, an $\epsilon$-optimal policy could be learned with $O(dH/\epsilon^2)$ samples. However, in the latent state $s$, the agent only receives an observation drawn uniformly from a large set $\cX_s$. The sets $\{\cX_s\}_{s \in \cS}$ form a partition of the entire observation space $\cX$ and there is a mapping $\phi \colon \cX \mapsto \cS$ that identifies the latent state for each observation. Each MDP $M_{\pi^*, \phi}$ in our problem class is parameterized by the mapping $\phi$ and a policy $\pi^* \in \Pi$. The class of policies $\Pi$ can be arbitrary as long as each pair of policies differ on at least a constant fraction of $\cX$. In MDP $M_{\pi^*, \phi}$, only the action $\pi^*(x)$ is a good action (blue action) and allows the agent to stay in the top latent state chain. Thus, finding the $\Theta(\epsilon)$-best policy in $\Pi$ for $M_{\pi^*, \phi}$ is equivalent to identifying $\pi^*$.

Importantly, our problem class contains MDP $M_{\pi^*, \phi}$ for every possible $\pi^* \in \Pi$ and latent state mapping $\phi$. We pick the number of observations large enough so that observations become uninformative and it is virtually impossible for a learner to learn $\phi$. Instead it can only hope to learn $\pi^*$ by identifying the bias $\epsilon$ in the rewards. We can show that this requires number of samples that are not much smaller than collecting $\Theta(1/\epsilon^2 \ln(1 / \delta))$ episodes with each of the $\prn*{H/d}^d$ policies in $\Pi$.

While the lower bound in \pref{thm:basic_lower_bound} does not have a dependence on $\log(\abs{\Pi})$. The simple observation that the contextual bandit problem can be seen as an instance of our setup where $d = 1$, implies that some dependence on $\log(\abs{\Pi})$ is necessary based on standard contextual bandit lower bounds \citep{lattimore2020bandit}. However, getting a lower bound of the form $\Omega\prn{H^d \log(\abs{\Pi})}$ is an interesting question, which we leave open for future work. 
\begin{figure}
    \centering
    \includegraphics[width=0.8\linewidth]{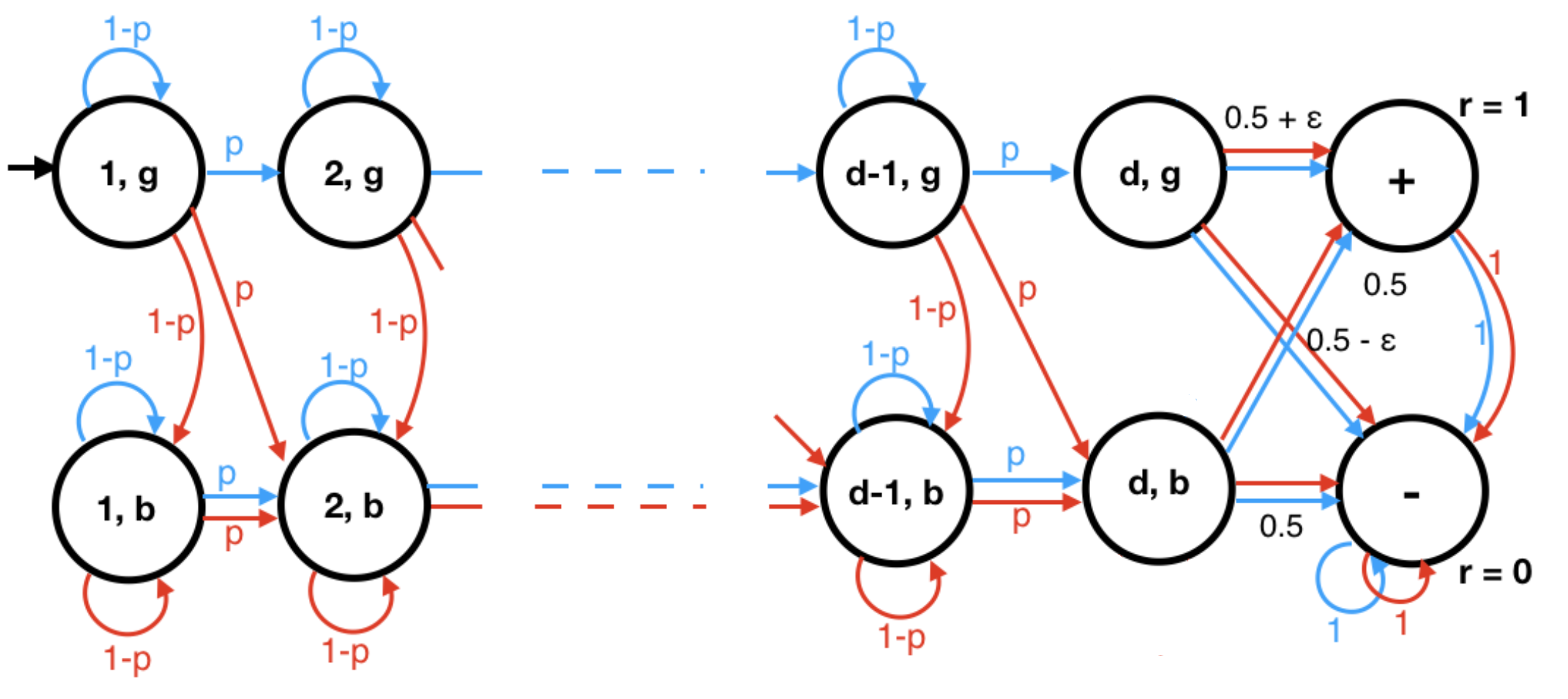}
    \caption{Latent state construction: contextual combination lock. As long as the agent follows actions of $\pi^*$(blue arrows), the agent remains in good states $(i, g)$ and receives a Bernoulli($1/2 + \epsilon$) reward but otherwise transits to bad states $(i,b)$ and receives a Bernoulli($1/2$) reward. 
    } 
    \label{fig:lower_bound_structure}
\end{figure} 
 
\section{Adaptive algorithms}  

In \pref{sec:basic_upper_bounds},
the algorithm introduced benefits from the guarantee provided by \pref{thm:basic_sample_complexity_bound},
which is near optimal in the worst case as the lower bound construction shows. 
However, in cases where the transition matrix induced by the policy class all have nicer eigenspectra, 
one could expect to have an improved sample complexity. Ideally, the
algorithm should automatically adapt to 
more favorable eigenspectra. 
This is precisely what we describe in this section. 
We give an adaptive algorithm whose sample complexity improves
when the eigenspectrum of transition matrices induced by the policy class 
admits a more favorable property. 

\subsection{Adaptivitity to the eigenspectrum}

Our adaptive algorithm, presented in \pref{alg:adaptive_policy_search} in \pref{app:adaptive_upper_bound_main}, is a policy search algorithm similar to  \pref{alg:policy_search} where, instead of invoking the procedure $\valestimate$,  we compute the value function for every policy $\pi$ by invoking the procedure  $\adavalestimate$ given in \pref{app:adaptive_upper_bound_main}.  

$\adavalestimate$ follows along the lines of $\valestimate$. When invoked for a policy $\pi$, it first estimates the expected rewards for the first $3d$ time steps. Then, $\adavalestimate$ computes the auto-regression coefficients $\alpha_k(\wh \lambda)$, and uses them to  predicts the expected rewards for all future time steps by extrapolating. The major difference between $\valestimate$ and  $\adavalestimate$ is the way the coefficients $\alpha_k(\wh \lambda)$ are computed. Specifically, using  $\Delta \ldef{} 2d 4^d \sqrt{(8 K^{3d} \log(6d \abs*{\Pi}) / \delta ) / n}$, the procedure $\adavalestimate$ computes the coefficients $\wh \lambda$ by solving the optimization problem 
\begin{align*} 
   \wh \lambda \leftarrow \argmin_{\lambda \in \mathbb C^d} &~  \prod_{k=2}^d \prn[\big]{\sum_{h = 0}^{H-1} \abs{\lambda_k}^h}  
   ~~~~ \text{s.t.} ~~ \lambda_{1} = 1, |\lambda_{k}| \leq 1 &\textrm{for } 2 \leq k \leq  d, \\  
   & \textrm{and }~ \abs[\Big]{\sum_{k=1}^d (-1)^{k+1} \alpha_{k}(\lambda) \wh R_{h - k} - \wh R_{h} } \leq \Delta  &\textrm{for } d + 1 \leq h \leq 3d. 
\end{align*} 

The above modification to the computation of $\wh \lambda$ allows our error propagation 
bound to adapt to $\lambda$, which defines the coefficients of autoregression for the expected rewards in policy $\pi$ (given in \pref{lem:basic_recurrence_relation}). The propagated error would be small if the coordinates of $\lambda$ are bounded away from $1$. The policy $\wt \pi$, returned by \pref{alg:adaptive_policy_search}, thus enjoys the following adaptive performance guarantee.
\begin{theorem}[Adaptive upper bound] 
\label{thm:adaptive_upper_bound_main} 
For a given $\delta \in (0,1)$, $d$-rank MDP, horizon $H$ and a finite policy class $\Pi$, after collecting $n$ episodes, \pref{alg:adaptive_policy_search} returns a policy $\wt \pi$ that with probability at least $1 - \delta$ admits the following guarantee: 
\begin{align*} 
V\ind{\wt \pi} &\geq \max_{\pi \in \Pi} V\ind{\pi} - O \prn[\bigg]{ d H^2 (16 e)^{2d} \cdot \max_{\pi' \in \Pi} \prod_{k=2}^{d} \prn[\Big]{\sum_{j=0}^{H-1} \abs{\lambda\ind{\pi'}_k}^{j}}^2  \sqrt{\frac{K^{3d} \log(6 \Pi d/ \delta)}{n}}}, 
\end{align*} 
\end{theorem} 
Proof of \pref{thm:adaptive_upper_bound_main} follows along the lines of the proof of \pref{thm:basic_sample_complexity_bound} where we replace the error propagation bound (in  \pref{lem:basic_ep_bound}) by a similar bound that adapts to the eigenspectrum of the transition matrix $T\ind{\pi}$. We defer the proof details to appendix \pref{app:adaptive_upper_bound_main}. Note that for any $\abs{\lambda_k} \leq 1$ and thus $\sum_{h = 0}^{H-1} \abs{\lambda_k}^h \leq H$.  Using this fact in \pref{thm:adaptive_upper_bound_main} recovers the result of \pref{thm:basic_sample_complexity_bound}, albeit upto a multiplicative factor of $2^{2d}$. In the following, we provide an example of a low rank MDP  problem in which the adaptive bound above could be much better than the worst case upper bound in \pref{thm:basic_sample_complexity_bound}. 
\begin{corollary}[Well mixing MDP] Given $\delta \in (0,1)$, horizon $H$ and a finite policy class $\Pi$. Let $M$ be a $d$-rank MDP such that the second largest eigenvalue of the transition matrix $T\ind{\pi}$ satisfies $\abs{\lambda\ind{\pi}_2} \leq 1 - \gamma~$ for every policy $\pi \in \Pi$. Then, after collecting $n$ episodes, our adaptive algorithm returns a policy $\wt \pi$ that with probability at least $1 - \delta$ admits the following guarantee: 
\begin{align*} 
V\ind{\wt \pi} \geq \max_{\pi \in \Pi} V \ind{\pi} - \widetilde{O}\prn[\Big]{\prn[\Big]{\frac{K}{\gamma }}^{2d} \frac{1}{\sqrt{n}}},  \end{align*} 
where the $\widetilde{O}$ hides polynomial factors of $d, H, \log(1/\delta)$ and multiplicative constants. 
\end{corollary} 

\ascomment{Add a sentence here on the proof of this!} 
We next show through a lower bound that the adaptive upper bound in \pref{thm:adaptive_upper_bound_main}  cannot be improved further. We defer the proof details to \pref{app:adaptive_lower_bound_proof}. 

\begin{theorem}[Adaptive lower bound]  \label{thm:adaptive_lower_bound}
Let $\epsilon \in (0, 1/16)$, $\delta \in (0, 1/2)$, $d \geq 4$ and $(\lambda_i)_{i \in [d]} \in [0, 1]^{d}$
satisfy
\begin{align*}
    d^{2d} \lesssim  \prod_{i=1}^{d} \frac{1}{1 - \lambda_i} &\lesssim \exp(H)
& \textrm{and} 
&&\sum_{i=1}^{d} \frac{1}{1 - \lambda_i} \leq \frac{H}{4 \ln(4d)}~.
\end{align*}

Then, there is a realizable policy class and a family of MDPs with rank at most $\Theta(d)$, finite observation space, horizon $H$ and two actions such that:  For each $i \in [d]$, policy $\pi$ and MDP $M$ in this class, there is an eigenvalue of the induced transition matrix $T^\pi_{M}$ in $[\lambda_i/2, \lambda_i]$. Furthermore, any algorithm that returns, with probability at least $1 - \delta$ an $\epsilon$-optimal policy for any MDP in this family, has to collect at least  
\begin{align*} 
         \Omega\prn[\Big]{\frac{1}{\epsilon^2 d^d} 
    \sqrt{\prod_{i=1}^{d}\frac{1}{1 - \lambda_i}}  \log\prn*{1/{2 \delta}}}
\end{align*} 
episodes in expectation in some MDP in this family. 
\end{theorem}

\paragraph{Adaptivity to rank.} In \pref{app:rank_adaptivity}, we also provide an adaptive algorithm that can find the best policy in the class $\Pi$ without knowing the value of the rank parameter $d^*$. 
Our adaptive algorithm, given in \pref{alg:adaptive_policy_search_rank}, follows from standard techniques in the model selection literature. For every $d \in [H]$, we compute an optimal policy $\widetilde{\pi}_{d}$ assuming that the rank $d^* = d$. Then, for each $d \in [H]$, we estimate the value function for the policy $\widetilde{\pi}_d$ by drawing $n/2H$ fresh trajectories using that policy. Finally, we return the policy $\wt \pi$ from the set $\crl*{\widetilde{\pi}_d}_{d  \in [H]}$ with the highest estimated value. The returned policy $\wt \pi$ satisfies, with probability at least $1- \delta$, 
\begin{align*}
V \ind{\wt \pi} \geq \max_{\pi \in \Pi} ~ V\ind{\pi}   - \wt{O}\prn[\Big]{\prn[\Big]{\frac{H}{d^*}}^{2d^*} \sqrt{ \frac{(8K)^{3d^*} \log(\abs*{\Pi} / \delta )}{n}}}. 
\end{align*}
We defer full details of the analysis to the Appendix. 
\section{Conclusion} \label{sec:conclusion} 
We presented a new analysis of reinforcement learning with rich
observations in the agnostic setting, under the low rank
MDP assumption. We gave both a non-adaptive and an adaptive algorithm for
learning a quasi-optimal policy in this scenario, which we showed to
benefit from guarantees that are only polynomial in the horizon and
the number of actions, and only logarithmic in the size of the policy
class considered. While our bound is exponential in the MDP rank, we
give nearly matching lower bounds proving that that dependency is
unavoidable.
The agnostic setting is a more realistic setting that has 
received less attention in the literature. We view this 
work as initiating the study of this general setting under
workable assumptions and believe that many other algorithmic
and theoretical aspects of such scenarios need to be   
studied further.  

\subsection*{Acknowledgements} 
Part of the work was performed when AS was an intern at Google Research, NYC. KS acknowledges support from NSF CAREER Award 1750575. YM  has received funding from the European Research Council (ERC) under the European Union’s Horizon 2020 research and innovation program (grant agreement No. 882396), by the Israel Science Foundation (grant number 993/17) and the Yandex Initiative for Machine Learning at Tel Aviv University.  

\newpage 
\setlength{\bibsep}{6pt} 
\bibliography{refs}
\newpage

\appendix
\renewcommand{\contentsname}{Contents of Appendix}
\tableofcontents
\addtocontents{toc}{\protect\setcounter{tocdepth}{3}} 
\clearpage

\setlength\parindent{0pt}

\section{Detailed comparison to prior work} \label{app:related_work} 

Provably sample-efficient learning algorithms have been well studied in the classical tabular RL literature \citep{kearns2002near, brafman2002r}. However, the number of samples required by these algorithms to find the optimal policy $\pi^*$ scales with the size of the state space $\abs{\cX}$ \citep{jaksch2010near, lattimore2012pac}, and thus these methods fail to scale to the rich observation settings where $\abs{\cX}$ could be astronomically large. There have been significant recent advances in developing efficient algorithms for such rich observation settings, albeit under additional assumptions. The two main styles of assumptions considered in the literature to make learning tractable are: (a) the learner has access to a value function class $\cF$ that realizes the optimal value function $f^*$ for the underlying MDP, and (b) the underlying transition dynamics admits additional structure such as low rank or linear decomposition, etc. We note that the goal of these works is to find the optimal policy for the underlying MDP. In comparison, in our work, we assume access to a policy class $\Pi$ and our goal is to find a policy $\wt \pi$ that could compete with the best policy in the class $\Pi$. In the following, we compare our setup with the assumptions made in the prior work.   

\paragraph{RL with general value function approximation.} Recently there has been great interest in designing RL algorithms with general function approximation \citep{jiang2017contextual, dann2018oracle,  sun2019model, du2019provably, wang2020reinforcement}. In particular, \cite{jiang2017contextual} introduced the notion of Bellman rank, a measure of complexity that depends on the underlying environment and the value function class $\cF$, and provide statistically efficient algorithms for learning problem for which Bellman rank is bounded. This was later extended to model-based algorithms by \cite{sun2019model}. While these algorithms work across a variety of problem settings, their sample complexity scales with $\log(\abs{\cF})$. Furthermore, these algorithms also require the optimal value function $f^*$ to be realized in $\cF$. In our work, we do not assume that the learner has access to a value function class $\cF$. In fact, given a value function class $\cF$, we can construct the policy class $\Pi_{\cF}$ that corresponds to greedy policies induced by the class $\cF$. However, given just a policy class $\Pi$, one can not construct a value function class, without additional knowledge of the underlying dynamics. 

\begin{example} Let $\cX = \crl*{0, 1, \ldots, N}$, $\cA = \crl{0, 1}$, $\Pi = \crl{\pi_0, \pi_1}$ and $H = 2$. For every action $a \in \cA$, we define the reward $r(x, a) = 1$ when $x$ is even, and $r(x, a) = 0$ when $x$ is odd. Further, we assume that the transition dynamics $T$ is parameterized by a vector $p \in \crl*{0, 1}^N$ such that for any state $x$, if $p(x) = 1$ and $a = 1$, then the next state $x'$ is sampled uniformly at random from the set of even numbers in $\cX$. Otherwise, we sample an odd number uniformly at random for $x'$. Clearly, in order to learn the optimal value function, the leaner needs to recover the value of the vector $p$ on at least $O(N)$ states. From standard packing arguments, we get that in $N$ dimensions there are at least $2^{O(N)}$ vectors that are $O(N)$ apart. Thus, any appropriate value function class $\cF$ that contains $p$ must have size at least $2^{O(N)}$. 
\end{example} 
\paragraph{Linear MDP assumption.} Our \pref{ass:low_rank} implies that for any policy $\pi \in \Pi$, the transition dynamics exhibits a low-rank decomposition with dimension $d$, that is  
$
    T\ind{\pi}(x' | x) = \langle \phi\ind{\pi}(x), \psi\ind{\pi}(x') \rangle,
$,  
for some d-dimensional feature maps $\phi\ind{\pi}, \psi\ind{\pi} \colon \cX \mapsto \bbR^d$. Low rank MDPs and linear transition models have recently gained a lot of  attention in the RL literature \citep{yang2020reinforcement, jin2020provably, modi2020sample, wang2019optimism}. The works most closely related to our setup are those of  \cite{jin2020provably} and \cite{yang2020reinforcement}, who give algorithms to find optimal policy in low rank MDPs with known feature maps $\phi$. Similarly, the other algorithms also assume that the learner either observes the feature $\phi(x)$, or the feature $\psi(x)$. However, in our setup, the learner neither observes the features $\phi \ind{\pi}$ nor the features $\psi \ind{\pi}$, thus restricting application of these algorithms to our setting. 

A new line of work, initiated by \cite{agarwal2020flambe}, focuses on the representation learning question in the above setting. They assume that the feature functions $\phi$ and $\psi$, although not known to the learner, are realized in the given classes $\Phi$ and $\Psi$ respectively. In order to find the optimal policy, their algorithm first identifies the underlying feature functions $\phi^*$ and $\psi^*$, and thus, their sample complexity guarantees scale with $\log(\abs{\Phi} \abs{\Psi})$. Later, \citet{modi2021model} show that a similar approach also works when the learner has only access to a $\Phi$ but not $\Psi$. In comparison, we do not assume knowledge of either classes $\Phi$ or $\Psi$, and instead work with a policy class $\Pi$. 
In fact, the following simple illustrative example shows that the feature function $\Phi$ could be arbitrarily complex even when $\abs*{\Pi}$ is small, and thus we can not hope to learn the feature function from samples. 

\begin{example}
\label{eg:eq_one}
Let $\cX = [N]$, $\cA = \crl{0, 1}$. We define the feature function $\psi(x) \in \bbR^2$ such that $(1/2N, 0)^\top$ if $x$ is even and $(0, 1/2N)^\top$ if $x$ is odd. Further, for $\lambda \geq 0$, define the feature function $\phi_\lambda(x) \in \bbR^2$ such that $\phi_\lambda(x) = (1, 0)^\top$ if $\sin(x/\lambda) \geq 0$, and $\phi_\lambda(x) = (0, 1)^\top$ otherwise. In this MDP, the next state $x'$ is either sampled uniformly at random from the set of even numbers in $\cX$ or sampled uniformly at random from the set of odd numbers in $\cX$, depending on the value of $\sin(x/\lambda)$.

Note that the mapping $x \mapsto \sin(x/\lambda)$ could be arbitrarily complex when $\lambda$ is small. In fact, the function class $\Phi = \crl{\lambda \mid x \mapsto \sin(x/\lambda)}$ has infinite VC dimension. Thus, one cannot hope to learn the feature function $\phi_\lambda$ from samples. 
\end{example} 

It is worth noting that in the above example, FLAMBE \citep{agarwal2020flambe}, MOFFLE \citep{modi2021model}, or  in fact any other approach that attempts to recover the feature function $\phi$, as mentioned above will not succeed. Furthermore, when $\abs{\Pi}$ is large and the length of the episode $H$ is large, the previously known agnostic upper bounds of $\tfrac{\abs{\Pi}}{\epsilon^2}$ or $\tfrac{2^H \log(\abs{\Pi})}{\epsilon^2}$ are also prohibitively large. However, in the above example, our algorithm enjoys a sample complexity bound of ${\tfrac{H^4 \log(\abs{\Pi})}{\epsilon^2}}$. 

Finally, note that in our setup, the decomposition of the induced transition kernel (into $\phi\ind{\pi}$ and $\psi\ind{\pi}$) may be different for each policy $\pi$ in the class $\Pi$. Furthermore, there may be policies outside of $\Pi$ that do not even exhibit such a low-rank decomposition. Thus, although our low rank assumption is similar to those in linear or low-rank MDPs \citep{agarwal2020flambe}, our model is more general. 
 
\paragraph{Comparison to Block MDP model.} \citet{krishnamurthy2016pac} introduced the block MDP model, where a small number of latent states $\cS$ govern the transition dynamics, and the observations $x \in \cX$ are generated depending on the current latent state $s$. In this model, there is a decoding function $g^*$ that maps observations $x$ back to the latent state $s$ that generates $x$. \cite{du2019provably, misra2020kinematic} assume that the learner is given a realizable class of decoding functions $\cG$ and show that the true mapping $g^* \in \cG$ can be learnt efficiently, both computationally and statistically, which can then be used to find the optimal policy. However, note that the transition matrix in a Block MDPs with $S$ latent states has rank at most $\abs{S}$, and thus their model is captured by our \pref{ass:low_rank}. However, in our setup, we do not assume that the leaner has access to the class $\cG$. In fact, Example 2 above shows that the latent state map $g^*$ (the mapping $\phi_{\lambda}(x)$ in that case) could be arbitrarily complex even when $\Pi$ is small, and thus we can not hope to  learn $g^*$ from samples. 

\paragraph{Policy gradient methods.} Model free direct policy search algorithms that directly maximize the value function have shown tremendous empirical success \citep{kakade2001natural, kakade2002approximately, levine2013guided, schulman2015trust, schulman2017proximal}, and recently, have been analysed from a theoretical perspective \citep{agarwal2019theory, abbasi2019politex, bhandari2019global, liu2020improved, agarwal2020pc}. While these methods operate directly on a policy class $\Pi$, as we do in our work, they require additional modelling assumptions in order to succeed; foremost being that the policy class $\Pi$ exhibits a differentiable paraeterization. Further assumptions include that the policy class $\Pi$ contains the optimal policy $\pi^*$, the policy class $\Pi$ has a good coverage over the state space \citep{agarwal2019theory}, and  that the underling MDP has a linear factorization with known feature maps \citep{agarwal2020pc}. We do not require these assumptions. 

\paragraph{DICE/DualDICE algorithms.} Recent works of \cite{liu2018breaking} and \cite{nachum2019dualdice} provide estimators that do not suffer the curse of horizon, i.e. the factor of $A^H$, in off-policy estimation of expected policy rewards by applying importance sampling on average visitation distributions of single steps of state-action pairs, instead of the much higher dimensional distribution of the whole trajectories. However, their estimator requires access to a function class $\cH$ that contains the importance weights of the average visitation distribution. We do not require access to such a class $\cH$ in our estimator of expected policy rewards.   

\paragraph{POMDP with reactive policies.} 
We will show in the following that our theory and algorithm applies to partially observable Markov decision processes (POMDPs), as long as policies are reactive, that is, only take the current observation into account. Although existing works such as \citep{jiang2017contextual} show polynomial sample-complexity bounds for POMDPs with reactive policy classes, they require the \emph{optimal} policy to be reactive, which is not true in POMDPs in general. In contrast, we can handle the important scenario where reactive policies can achieve good but not necessarily close to optimal performance and we are interested in finding the best such policy.

A POMDP consists of a MDP with finite state space $\cS$, action space $\cA$ and horizon $H$ where observed rewards at each step are drawn from a distribution with mean $r(s_h, a_h)$ that depends on the current state $s_h$ and action $a_h$. Similarly, the next state is drawn fron a transition kernel $P(s_{h+1} | s_h, a_h)$. However, in a POMDP, the current state is not observable and the agent instead receives an observation $x_h \in \cX$. We consider the formulation where the observation is drawn from a distribution $O(x_h | s_h)$ that depends on the current latent state $s_h$. Unlike in, e.g., Block MDP models, $x_h$ does not need to be sufficient to decode $s_h$ and this model does not need to be an MDP over the observation space $\cX$. As a consequence, the optimal actions do in general depend on \emph{all} previous observations.
Nonetheless, reactive policies which are of the form $\cX \rightarrow \cA$ and  only take the current observation into account, often achieve good performance and are of particular interest in practice due to their simplicity. 

Since a POMDP may not be a MDP over observations, such models are formally outside of our scope. However, as our technique never explicitly accesses observations except through the policy, we can cast a POMDP problem as follows in our framework.
For any policy $\pi \colon \cX \rightarrow \cA$ in our policy class $\Pi$ we define a stochastic policy $\pi'$ over latent states as $\pi'(a | s) = \sum_{x \in \cX} \indicator{\pi(x) = a} O(x | s)$ and denote the class of these policies by $\Pi' \subseteq \cS \rightarrow \Delta(\cA)$. Running our algorithms on a POMDP with policy class $\Pi$ is equivalent to running them on an MDP with direct access to latent states $\cS$ and policy class $\Pi'$.
Since an MDP with finite state space $\cS$ has rank at most $|\cS|$, our guarantees apply to POMDPs with a reactive policy class and we can set $d = |\cS|$. 

\paragraph{Exponential lower bounds for planning and offline RL.}

Several publications
\citep{wang2020statistical, zanette2020exponential, weisz2021exponential} recently provide exponential lower bounds for learning the optimal policy with access to a realizable linear Q-function class $\cF$ of dimension $d$ in several settings. Most related is \citet{wang2020statistical}, wo study offline RL where the agent has only access to a dataset of transition samples and show even if the dataset has good coverage of the features of $\cF$, a sample complexity that is exponential in $d$ or $H$ is unavoidable. In contrast, we allow the agent to collect samples arbitrarily by interacting with the MDP and although our algorithms first collect a dataset non-adaptively, the uniform action choices ensure good state coverage as opposed to just feature coverage which avoids the existing lower bounds. 

\clearpage  
\section{Cayley-Hamilton theorems} 

The following result holds for any matrix $A$ with rank $d$.

\begin{lemma}[Cayley-Hamilton Theorem for rank $d$ matrices  \citep{segercrantz1992improving}] 
\label{lem:cayley_hamilton_basic}  
Let $A \in \bbC^{N \times N}$ be a matrix with rank at most $d$, where $d \leq N$, and let $\lambda = \prn*{ \lambda_1, \ldots, \lambda_d} \in \bbC^d$ denote the set of eigenvalues of $A$. Then, $A$ satisfies the relation 
\begin{align*} 
   A^{d + 1} &= \sum_{k=1}^{d} (-1)^{k+1} \alpha_{k}(\lambda) A^{d + 1- k}, 
\end{align*} 
where the coefficient $\alpha_k(\lambda)$ are given by the sum of degree $k$ monomials:
\begin{align*} 
   \alpha_{k}(\lambda) = \sum_{x \in \crl{0, 1}^d \text{~s.t.~} \nrm{x}_1 = k}  \lambda_1^{x_1}\lambda_2^{x_2} \ldots \lambda_d^{x_d}.  
\end{align*} 
\end{lemma} 

The proof of the above follows from the characteristic polynomial for rank $d$ matrices, which allows us to express $d + 1$-th power for any matrix $A$ in terms of the lower powers. 

We will soon provide an extension of the above result which allows us to express the $n$-th power of the matrix $A$ in terms of the lower powers. Before doing so, we need to define some additional notation.

\subsection{Coefficients $\alpha_{m, k}$} 

For any $m \geq 0$ and $k \geq 0$, we first define the coefficients $\alpha_{m, k}$. 
 
\begin{definition}
\label{def:alpha_definition} 
For any $k \geq 0$ and $\lambda = (\lambda_1, \ldots, \lambda_d) \in \bbC^{d}$, define $\alpha_{m, k}(\lambda)$ to denote the quantity 
\begin{align*} 
   \alpha_{m, k}(\lambda) \ldef{} \sum_{y \in \crl*{0, \ldots, m}^d}  \mathbbm{1}\crl[\Big]{\sum_{j=1}^d \indicator{y_j > 0} = k ~ {\text{and}} ~ \sum_{j=1}^{d} y_j = m}  \prod_{j=1}^d \lambda_j^{y_j}.  \numberthis \label{eq:alpha_defn_appx} 
\end{align*} 
whenever $m \geq k$ and $\alpha_{m, k} = 0$ when $m \leq k$ or $k > d$. Further, for the ease of notation, for any $k \in [d]$, we define $\alpha_k(\lambda)$ to denote the quantity $\alpha_{k, k}(\lambda)$. 
\end{definition} 

The following lemma provides a useful technical relation between the coefficients $\alpha_{m, k}$ defined above. 
\begin{lemma} 
\label{lem:alpha_recursive_relation} 
For any $m \geq 0$, $k \in [d]$ and $\lambda \in \bbC^d$, the quantities $\prn{\alpha_{m, k}}_{k \in [d], m \geq 0}$ given in \pref{def:alpha_definition} satisfy 
\begin{align*} 
\sum_{j=1}^{m \wedge d} \alpha_{m, j}(\lambda) \cdot \alpha_{k, k}(\lambda) &= \sum_{j'=k + 1}^{(m + k) \wedge d} {j' \choose k} \alpha_{m + k, j'}(\lambda). 
\end{align*}
\end{lemma} 
\begin{proof} For the sake of the proof, we will be interpreting $\alpha_{m, j}(\lambda)$ and $\alpha_{k, k}(\lambda)$ as symmetric polynomials with $\lambda$ as the formal variables. The value of these quantities can be computed by plugging in the value of $\lambda_1, \ldots, \lambda_d$ for $\lambda$.  

Thus, $\alpha_{m, j}$ denotes a symmetric sum of monomials, where each monomial term has $j$ variables with sum of all the powers in that monomial being $m$. Similarly, $\alpha_{k, k}$ denotes a symmetric sum of monomials, where each monomial term has $k$ variables each with the power of $1$. Subsequently, when we take the product $\alpha_{m, j} \cdot \alpha_{k, k}$, we will get monomial terms, where in each term the sum of all the powers is $m + k$, but the total number of distinct variables can range from $k + 1$ to $\min\crl*{j + k, d}$. Since, the polynomials $\alpha_{m, j}$ and $\alpha_{k, k}$ are symmetric in $\lambda$, the resultant polynomial that we will get after taking their product will also be symmetric.  Furthermore, each of the monomial terms with $j'$ distinct variables can be generated through ${j' \choose k}$ different splits with $k$ variables that go into $\alpha_{k, k}(\lambda)$ and the rest $j' - k$ variables that go into $\alpha_{m, j'-k}(\lambda)$. Hence, the coefficient of $\alpha_{m + k, j'}$ would be exactly ${j' \choose k}$. We formalize this in the following: 

\begin{align*}
	\sum_{j=1}^{m \wedge d} \alpha_{m, j}(\lambda) \cdot \alpha_{k, k}(\lambda) &=  \sum_{j=1}^{m \wedge d} \Bigg( \sum_{y \geq 0} \mathbbm{1}\crl[\Big]{\sum_{i=1}^d \indicator{y_i > 0} = j \text{~and~} \sum_{i=1}^{d} y_i = m}  \prod_{i=1}^d \lambda_i^{y_i} ~ \times \\ & \qquad \qquad \qquad \qquad \sum_{y' \geq 0}  \mathbbm{1}\crl[\Big]{\sum_{i=1}^d \indicator{y'_i > 0} = k \text{~and~} \sum_{i=1}^{d} y'_i = k}  \prod_{i=1}^d \lambda_i^{y'_i} \Bigg)  \\
	&= \sum_{j=1}^{m \wedge d} \sum_{\substack{y, y' \geq 0}} \mathbbm{1}\crl[\Big]{\sum_{i=1}^d \indicator{y_i > 0} = j \wedge \sum_{i=1}^d \indicator{y'_i > 0} = k \wedge \sum_{i=1}^{d} y_i = m \wedge \sum_{i=1}^{d} y'_i = k}  \prod_{i=1}^d \lambda_i^{y_i + y'_i}  \\ 
	&= \sum_{j'=k+1}^{(m + k) \wedge d} {j' \choose k} \sum_{y, y' \geq 0} \mathbbm{1}\crl[\Big]{\sum_{i=1}^d \indicator{y_i + y'_i > 0} = j' \text{~and~} \sum_{i=1}^{d} y_i + y'_i = m + k} \prod_{i=1}^d \lambda_i^{y_i + y'_i} \\ 
	&= \sum_{j'=k+1}^{(m + k) \wedge d} {j' \choose k} \sum_{y'' \geq 0} \mathbbm{1}\crl[\Big]{\sum_{i=1}^d \indicator{y''_i > 0} = j' \text{~and~} \sum_{i=1}^{d} y''_i = m + k} \prod_{i=1}^d \lambda_i^{y''_i} \\
	&= \sum_{j'=k+1}^{(m + k) \wedge d} {j' \choose k} \cdot \alpha_{m + k, j'}, 
\end{align*} where $y''_i \ldef{} y_i + y'_i$ and the third equality in the above follows by rearranging the terms while satisfying the constraints inside the indicator. 
\end{proof} 

We next provide a bound on the value of $\alpha_{m, k}$ as a function of $m$ and $k$. 
\begin{lemma} 
\label{lem:alpha_independent_bound} 
For any $d \geq 1$, $m \geq 0$, $k \leq \min\crl*{d, m}$ and $\lambda \in \bbC^d$, that satisfies $\abs*{\lambda_{j}} \leq 1$ for all $j \in [d]$, the quantities $\alpha_{m, k}(\lambda)$ given in \pref{def:alpha_definition} satisfy the bound 
\begin{align*}
\abs*{\alpha_{m, k}(\lambda)} \leq \prn[\Big]{\frac{4 e \max\crl{m,  d}}{d}}^d. 
\end{align*} 
Furthermore, for $k= m \leq d$, we have that  $\alpha_{k}(\lambda) = \alpha_{k, k}(\lambda) \leq {4}^d$.     
\end{lemma} 
\begin{proof}
Starting from the definition of $\alpha_{m, k}(\lambda)$, we note that
\begin{align*}
\abs*{\alpha_{m, k}(\lambda)} &= \abs[\Big]{  \sum_{y \in \crl*{0, \ldots, m}^d}  \mathbbm{1}\crl[\Big]{\sum_{j=1}^d \indicator{y_j > 0} = k ~ {\text{and}} ~ \sum_{j=1}^{d} y_j = m}  \prod_{j=1}^d \lambda_j^{y_j}} \\
&\leq  \sum_{y \in \crl*{0, \ldots, m}^d}  \mathbbm{1}\crl[\Big]{\sum_{j=1}^d \indicator{y_j > 0} = k ~ {\text{and}} ~ \sum_{j=1}^{d} y_j = m} \abs[\Big]{ \prod_{j=1}^d \lambda_j^{y_j}} \\ 
&=  \sum_{y \in \crl*{0, \ldots, m}^d}  \mathbbm{1}\crl[\Big]{\sum_{j=1}^d \indicator{y_j > 0} = k ~ {\text{and}} ~ \sum_{j=1}^{d} y_j = m} \prod_{j=1}^d \abs*{ \lambda_j^{y_j}} \\ 
&\leq  \sum_{y \in \crl*{0, \ldots, m}^d}  \mathbbm{1}\crl[\Big]{\sum_{j=1}^d \indicator{y_j > 0} = k ~ {\text{and}} ~ \sum_{j=1}^{d} y_j = m}, 
\end{align*}
where the inequality in the second line follows from an application of Triangle inequality. The last line holds because $|\lambda_j|\leq1$, and thus $|\prod \lambda_j^{y_j}|\leq 1$ for any $y$. We note that the right hand side in the above expression denotes the number of ways of distributing $m$ balls into $d$ bins such that exactly $k$ of them are non-empty. If $m = k = 1$, we get that $\abs*{\alpha_{m, k}(\lambda)} \leq 1$. Otherwise, a simple counting argument implies that 
\begin{align*}
\abs*{\alpha_{m, k}(\lambda)} &\leq {d \choose k} {m - 1 \choose k - 1} \leq 2^d \cdot {m - 1 \choose k - 1}. 
\end{align*} 
When $m \leq d$ or $k = 1$, we can simply upper bound the above as $$\abs*{\alpha_{m, k}(\lambda)} \leq 2^d \cdot 2^m \leq 4^d.$$  Next, when $m > d$ and $k\geq 2$, using the fact that ${N \choose n} \leq \prn*{e N /  n}^n$ for any $0 < n \leq N$, we get that 
\begin{align*} 
\abs*{\alpha_{m, k}(\lambda)} &\leq 2^d \cdot \prn*{ \frac{e(m - 1)}{(k -1)}}^k \\
&\leq 2^d \cdot \prn*{\frac{2 e m}{k}}^k \\
&\leq 2^d \cdot \prn*{\frac{2 e m}{d}}^d,  
\end{align*}  where the inequality in the second line above holds because $(m - 1)/ (k -1) \leq 2 m / k$ for $k \geq 2$, and the inequality in the last line holds because the function $\prn*{x/y}^y$ is an increasing function of $y$ when $x \geq e y$. 

Considering the above two bounds together implies that: 
\begin{align*}
\abs*{\alpha_{m, k}(\lambda)} &\leq  \prn[\Big]{\frac{4 e \max\crl{m,  d}}{d}}^d. 
\end{align*} 
\end{proof}

\subsection{Coefficients $\beta_{m, k}$} 
We next define the coefficients $\beta_{m,k}$ which will be useful in our upper bound analysis. 
\begin{definition} 
\label{def:P_beta} 
For any $d \geq 1$, $\lambda  \in \bbC^d$ and $m \geq 0$, define the vector $\beta_m(\lambda) \in \bbC^d$ using the following recursion:  
\begin{enumerate}[label=(\alph*)]
    \item $\beta_0(\lambda) \ldef{}  (\alpha_1(\lambda), \ldots, \alpha_d(\lambda))^\top$, and, 
    \item For $m \geq 1$,  define $\beta_{m}(\lambda) \ldef{} \prn*{\beta_{m, 1}(\lambda), \ldots, \beta_{m, d}(\lambda)}^\top$ as
\begin{align*} 
\beta_{m, k}(\lambda) &= \begin{cases} \beta_{m-1, 1}(\lambda) \cdot \alpha_k(\lambda) - \beta_{m-1, k+1}(\lambda) &\text{for}~ 1 \leq k \leq d -1 \\ 
\beta_{m-1, d}(\lambda) \cdot \alpha_d(\lambda) &\text{for}~ {k = d}
\end{cases}, 
\end{align*} 
\end{enumerate}
where $\alpha_k(\lambda)$ is as defined in \pref{eq:alpha_defn_appx} , and $\beta_{0, k}$ denotes the $k$-th coordinate of the vector $\beta_0$. 
\end{definition} 

The next technical lemma provides a relation between the $\beta$ and $\alpha$ values defined above. 
\begin{lemma}
\label{lem:rel_beta_alpha} 
 For any $m \geq 0$ and $k \in [d]$, 
	\begin{align*}
	\beta_{m, k}(\lambda) = \sum_{j=k}^{(m + k) \wedge d} {j - 1 \choose k -1} \alpha_{m + k, j}(\lambda). \numberthis \label{eq:bounded_coefficients1}  
\end{align*} 
\end{lemma}
\begin{proof}
We prove the desired relation via induction over $m$. For the base case, when $m = 0$, from the definition of $\beta_{0, k }$, we note that 
\begin{align*}
\beta_{0, k } &= \alpha_{k, k}(\lambda) = \sum_{j=k}^{k} {k - 1 \choose k-1} \alpha_{k, j}(\lambda). 
\end{align*}
Now, we proceed to the induction step. Assume that the relation \pref{eq:bounded_coefficients1} holds for all $m' < m$. Thus, for any $k \in [d]$, from the definition of $\beta_{m, k}(\lambda)$, we have that 
\begin{align*} 
\beta_{m, k}(\lambda) &= \beta_{m-1, 1}(\lambda) \cdot \beta_{0, k}(\lambda) - \beta_{m-1, k+1}(\lambda) \\ 
&= \prn[\Big]{\sum_{j=1}^{m \wedge d} \alpha_{m, j}(\lambda)} \cdot \alpha_{k, k}(\lambda)  - \sum_{j=k+1}^{(m + k) \wedge d} {j - 1 \choose k} \cdot \alpha_{m + k, j}(\lambda), 
\end{align*} 
where the equality in the second line follows from using the relation \pref{eq:bounded_coefficients1} for time step $m-1$. Using the identity in \pref{lem:alpha_recursive_relation} in the above, we get that 
\begin{align*}
\beta_{m, k}(\lambda) &= \sum_{j=k}^{(m + k) \wedge d} {j \choose k} \cdot \alpha_{m + k, j}(\lambda)  - \sum_{j=k+1}^{(m + k) \wedge d} {j - 1 \choose k} \cdot \alpha_{m + k, j}(\lambda) \\ 
&= \sum_{j=k}^{(m + k) \wedge d} {j - 1\choose k - 1} \cdot \alpha_{m + k, j}(\lambda), 
\end{align*}
where the last line uses the relation ${j \choose k} = {j - 1 \choose k -1 } + {j - 1 \choose k}$. This completes the induction step. Thus, proving that the relation \pref{eq:bounded_coefficients1} holds for all $m \geq 0$ and $k \in [d]$. 	
\end{proof}

We next provide a bound on the value of the coefficients $\beta_{m, k}$ as a function of $m$ and $k$. 

\begin{lemma}  
\label{lem:bounded_coefficients} 
For any $d \geq 1$, $m \geq 0$, $k \leq d$ and $\lambda \in \bbC^d$, such that $\abs*{\lambda_{j}} \leq 1$ for all $j \in [d]$, the quantities $\beta_{m, k}(\lambda)$ defined in \pref{def:P_beta} satisfy the bound
\begin{align*}
\abs{\beta_{m, k}(\lambda)} \leq \prn[\Big]{\frac{8 e \max\crl{m + k,  d}}{d}}^d. 
\end{align*}

\end{lemma} 
\begin{proof} As a consequence of \pref{lem:rel_beta_alpha}, we have that for any $m \geq 0$ and $k \in [d]$, 
\begin{align*}
	\beta_{m, k}(\lambda) = \sum_{j=k}^{(m + k) \wedge d} {j - 1 \choose k -1} \cdot \alpha_{m + k, j}(\lambda). 
\end{align*} Thus, using Triangle inequality, we have that
\begin{align*} 
	\abs{\beta_{m, k}} &= \abs[\Big]{ \sum_{j=k}^{(m + k) \wedge d}  {j - 1 \choose k -1} \cdot  \alpha_{m + k, j}(\lambda)} \\ 
	&\leq \sum_{j=k}^{(m + k) \wedge d} {j - 1 \choose k -1} \cdot \abs{\alpha_{m + k, j}(\lambda)}.  
\end{align*}
Plugging in the bound on $\abs{\alpha_{m+k, j}(\lambda)}$ from  \pref{lem:alpha_independent_bound} in the above, we get that 
\begin{align*}
		\abs{\beta_{m, k}}  &\leq  \sum_{j=k}^{d} {j - 1 \choose k -1} \cdot \prn[\Big]{\frac{4 e \max\crl{m + k,  d}}{d}}^d \\ 
		&\overleq{\proman{1}} {d \choose k} \cdot \prn[\Big]{\frac{4 e \max\crl{m + k,  d}}{d}}^d \\ 
		&\overleq{\proman{2}} \prn[\Big]{\frac{8 e \max\crl{m + k,  d}}{d}}^d, 
\end{align*} where the inequality in $\proman{1}$ is given by the fact that any $N$ and $n$, we have $\sum_{j=n}^{N} {j \choose n} = {N + 1 \choose n + 1}$, and the inequality in $\proman{2}$ holds because for any $k \leq d$, ${d \choose k} \leq 2^d$. 
\end{proof}

\subsection{Extension of the Cayley-Hamilton theorem} 

The following result is an extension of the Cayley-Hamilton theorem (\pref{lem:cayley_hamilton_basic}) for rank $d$ matrices, and relies on the coefficients $\beta_{m, k}$ defined above.  
\begin{lemma}[Cayley-Hamilton Theorem extension]   
\label{lem:cayley_hamilton_extension} 
Let $A \in \bbC^{N \times N}$ be a matrix with rank at most $d$, where $d \leq N$, and let $\lambda = \prn*{ \lambda_1, \ldots, \lambda_d} \in \bbC^d$ denote the set of eigenvalues of $A$. Then, for any $m \geq 0$, 
\begin{align*}
	A^{d + m + 1} &= \sum_{k=1}^{d} (-1)^{k+1} \beta_{m, k}(\lambda) A^{d + 1- k} \numberthis \label{eq:CH_beta_relation} 
\end{align*}
where the coefficients vector $\beta_m(\lambda) \ldef{} \prn*{\beta_{m, 1}(\lambda), \ldots, \beta_{m, k}(\lambda)}$ are given in \pref{def:P_beta}. 
\end{lemma}
\begin{proof} 

We give a proof by induction over $m$. For the base case, when $m = 0$,  \pref{lem:cayley_hamilton_basic} implies that 
\begin{align*} 
   A^{d + 1} &= \sum_{k=1}^{d} (-1)^{k+1} \alpha_{k}(\lambda) A^{d + 1- k} = \sum_{k=1}^{d} (-1)^{k+1} \beta_{0, k}(\lambda) A^{d + 1- k}, \numberthis \label{eq:CH_beta_relation1} 
\end{align*} where the second equality follows form the definition of the vector $\beta_0(\lambda)$. We next prove the induction step. 

Assume that the relation \pref{eq:CH_beta_relation} holds for all $m' < m$. We note that 
\begin{align*}	
A^{d +1+m} &= A A^{d + 1 + (m - 1)} \\ 
				    &\overeq{\proman{1}} A \prn[\Big]{ \sum_{k=1}^{d} (-1)^{k+1} \beta_{m-1, k}(\lambda) \cdot A^{d + 1- k}}  \\ 
					  &= \beta_{m-1, 1}(\lambda) A^{d+1}  + \sum_{k=2}^{n} (-1)^{k+1} \beta_{m-1, k}(\lambda) \cdot A^{d + 2 - k}  \\ 
					  &= \beta_{m-1, 1}(\lambda) A^{d+1}  + \sum_{k=1}^{n - 1} (-1)^{k} \beta_{m-1, k + 1}(\lambda) \cdot A^{d + 1 - k}. 
\end{align*} 
where the equality in $\proman{1}$ following from using the relation \pref{eq:CH_beta_relation} for time step $m - 1$. Plugging in the expansion for $A^{d+1}$  from  \pref{eq:CH_beta_relation1} in the above, we get that 
\begin{align*}
A^{d+ 1 + m} &= \beta_{m-1, 1}(\lambda) \prn[\Big]{\sum_{k=1}^{d} (-1)^{k+1} \beta_{0, k} (\lambda)\cdot A^{d + 1  - k} }  + \sum_{k=1}^{d - 1} (-1)^{k} \beta_{m-1, k + 1}(\lambda) \cdot A^{d + 1- k} \\ 
   					&= \sum_{k=1}^{d} (-1)^{k+1} \prn*{\beta_{m-1, 1}(\lambda) \cdot \beta_{0, k}(\lambda) - \beta_{m-1, k+1}(\lambda) } A^{d +1 - k}.  \numberthis \label{eq:CH_beta_relation2} 
\end{align*} where in the second line, we defined $\beta_{m-1, d + 1} = 0$. We next note that for any $ k \in [d]$, 
\begin{align*}
	\beta_{m-1, 1}(\lambda) \cdot \beta_{0, k}(\lambda) - \beta_{m-1, k+1}(\lambda) &= \beta_{m-1, 1}(\lambda) \cdot \alpha_{k}(\lambda) - \beta_{m-1, k+1}(\lambda) \\ 
&= \beta_{m, k}(\lambda), 
\end{align*} where the second line above follows from the definition of $\beta_{m, k}$. Using this relation in \pref{eq:CH_beta_relation1}, we get that 
\begin{align*}
A^{d +1+m} &=  \sum_{k=1}^{d} (-1)^{k+1} \beta_{m, k}(\lambda) A^{d + 1- k},
\end{align*} hence completing the induction step. Thus, the relation \pref{eq:CH_beta_relation} holds for all $m \geq 0$. 
\end{proof} 

\section{Missing proofs from \pref{sec:basic_upper_bounds}}   
\subsection{Proof of \pref{lem:basic_recurrence_relation}} \label{app:basic_recurrence_relation} 
\autoregressionlemma*

\begin{proof} For any time step $h \geq 1$, let $\mu\ind{\pi}_{h}$ denote the distribution over the observation space $\cX$ at time step $h$ when starting from the initial distribution $\mu_0$ and taking actions according to the policy $\pi$. Using the definition of the transition matrix $T \ind{\pi}$, we note that 
\begin{align*}
	\mu\ind{\pi}_{h} = T\ind{\pi} \mu\ind{\pi}_{h-1},  \numberthis \label{eq:basic_recurrence}
\end{align*}
where $\mu\ind{\pi}_0$ is defined as $\mu_0$. Further, let $\nu\ind{\pi} \in \bbR^\cX$ denotes  the vector of expected rewards under policy $\pi$ on the observation space, i.e., for any observation $x \in \cX$, 
\begin{align*}
\nu\ind{\pi}(x) \ldef{}  \En_{a \sim \pi(x)} \brk*{r(s, a)}. 
\end{align*}
Thus, for any $h \leq H$, the expected reward $R\ind{\pi}_h$ is given by the expression
\begin{align*}
	R\ind{\pi}_h &= \tri{\nu\ind{\pi}, \mu_{h} \ind{\pi} } 
	                    = \tri{\nu\ind{\pi}, \prn{T \ind{\pi}}^{d+1} \mu_{h- d - 1} \ind{\pi} }, \numberthis \label{eq:basic_recurrence1}
\end{align*} 
where the second equality follows from recursively using the relation \pref{eq:basic_recurrence}. Using the Cayley-Hamilton theorem (\pref{lem:cayley_hamilton_basic}) for the matrix $T\ind{\pi}$, with rank at most $d$, we get that
\begin{align*}
	\prn{T\ind{\pi}}^{d +1} &= \sum_{k=1}^d (-1)^{k+1} \alpha_k(\lambda) \prn{T\ind{\pi}}^{d +1 - k}, 
\end{align*}
where $\lambda = \prn*{\lambda\ind{\pi}_1, \ldots, \lambda\ind{\pi}_d}$ denotes the set of eigenvalues of $T\ind{\pi}$. Plugging the above in relation \pref{eq:basic_recurrence1} for $h \geq d + 1$, we get that 
\begin{align*}
	R\ind{\pi}_h &= \tri{\nu\ind{\pi}, {\sum_{k=1}^d \alpha_k(\lambda) \prn{T\ind{\pi}}^{d + 1 - k} } \mu_{h- d - 1} \ind{\pi} } \\
	&=  \sum_{k=1}^d (-1)^{k+1} \alpha_k(\lambda)  \tri{\nu\ind{\pi},  \mu_{h- k} \ind{\pi} } = \sum_{k=1}^d (-1)^{k+1} \alpha_k(\lambda) R\ind{\pi}_{h - k }, 
\end{align*} 
where the last equality follows from plugging back the expression for $R\ind{\pi}_{h - k }$ from \pref{eq:basic_recurrence1}. 
\end{proof} 

\subsection{Proof of \pref{lem:ss_ub}} 
The following result provides an upper bound on the error in our estimates for the expected reward for any policy $\pi \in \Pi$.  

\ISbound*
\begin{proof} 
First fix any $h \in [3d]$ and $\pi \in \Pi$. The expected policy reward estimate is given by 

\begin{align*}
   \wh R\ind{\pi}_h &= \frac{1}{n} \sum_{i =1}^n r_h\ind{t} \prod_{h' \leq h} \prn*{K\indicator{\pi(x\ind{t}_{h'}) = a_{h'}\ind{t}}}
\end{align*}
Clearly, $\wh R\ind{\pi}_h$ is an unbiased estimate of $R\ind{\pi}_h$ as 
\begin{align*}
\bbE\ind{\bar \pi} 
    \brk{\wh R\ind{\pi}_h} &= \frac{1}{n} \sum_{t=1}^n \En\ind{\bar{\pi}} \brk[\Big]{ r_h\ind{t} \prod_{h' \leq h} \prn*{K\indicator{\pi(x\ind{t}_{h'}) = a_{h'}\ind{t}}}} \\ 
    &= \frac{1}{n} \sum_{t=1}^n \bbE\ind{\bar \pi} 
    \brk[\Big]{ r_h\ind{t} \prod_{h' \leq h} \frac{\pi(a\ind{t}_{h'} | x\ind{t}_{h'})}{\bar \pi(a\ind{t}_{h'} | x\ind{t}_{h'})}
    } \\
    &= \frac{1}{n} \sum_{t=1}^n \bbE\ind{\pi} 
    \brk{r_h^t
    } = \frac{1}{n} \sum_{t=1}^n R\ind{\pi}_h = R\ind{\pi}_h,
\end{align*}
where $\bar \pi$ denotes the stochastic policy that picks actions uniformly at random and is used to draw the trajectory $(x_h\ind{t}, a_h\ind{t}, r_h\ind{t})_{h=1}^H$ for $t \in [n]$. The equality in the second line above follows from the definition of $\bar \pi$ and the last line follows by a change of measure to the case where the trajectories are sampled using the policy $\pi$. We next consider the second moment of each individual term in the estimator 
\begin{align*}
    \bbE\ind{\bar \pi} 
    \brk[\Big]{
    (r_h^t)^2 \prod_{h' \leq h} \prn*{K\indicator{\pi(x\ind{t}_{h'}) = a_{h'}\ind{t}}}^2 
    }
    &\overset{\proman{1}}{\leq} K^{2h} \prod_{h' \leq h} \bbP\ind{\bar \pi} \prn*{\pi(x^t_{h'}) = a^t_{h'}} \\
    &\overset{\proman{2}}=K^{2h} \cdot \frac{1}{K^h} = K^h,
\end{align*}
where the inequality $\proman{1}$ uses that $0 \leq r^i_h \leq 1$, and the inequality $\proman{2}$ holds because $\bar \pi$ draws actions uniformly at random which implies that $
\bbP\ind{\bar \pi}\prn*{a^i_{h'} = \pi(x^i_{h'}) \mid x^i_{h'}} = 1/ K$. Therefore the variance for the $t$th sample, 
\begin{align*}
 \bbV\ind{\bar \pi}\brk[\Big]{ r_h\ind{t} \prod_{h' \leq h} \prn*{K\indicator{\pi(x\ind{t}_{h'}) = a_{h'}\ind{t}}}} \leq K^h.
\end{align*}

Since all episodes are i.i.d., an application of Bernstein's inequality implies that with probability at least $1 - \delta$
\begin{align*}
 \abs*{\wh R\ind{\pi}_h - R\ind{\pi}_h}
 &\leq \sqrt{\frac{2 \bbV\ind{\bar \pi}\brk*{K^h \indicator{a^i_{1:h} = \pi(a^i)_{1:h}} r_h^i} \log(2 / \delta)}{n}} + \frac{4K^h}{3}\frac{\log(2 / \delta)}{n}\\
 & \leq 
 \sqrt{\frac{2 K^h \log(2 / \delta)}{n}} + \frac{2K^h\log(2 / \delta)}{n}. 
\end{align*}

Taking a union bound, we get that with probability at least $1 - \delta$, for all $h \in [3d]$ and $\pi \in \Pi$, 
\begin{align*}
 \abs*{\wh R\ind{\pi}_h - R\ind{\pi}_h}
 &\leq  \sqrt{\frac{2 K^{3d} \log(6 d \abs*{\Pi} / \delta)}{n}} + \frac{2K^{3d}\log(6 d \abs*{\Pi} / \delta)}{n}. 
\end{align*}

\end{proof} 

\subsection{Proof of \pref{lem:basic_ep_bound}}  \label{app:appendix_basic_ep_bound} 
Before providing the proof of \pref{lem:basic_ep_bound}, we first introduce the matrix $P(\lambda)$ that depends on the eigenvalues $\lambda \in \bbC^d$, and establish a technical result about the eigenspectrum of $P(\lambda)$. 

\begin{definition} \label{def:P_defn_app} 
 For any $d \geq 1$, $\lambda = (\lambda_1, \ldots, \lambda_d) \in \bbC^d$, define the matrix $\transfer(\lambda) \in \bbC^{d \times d}$ such that 
\begin{align*}
	\brk{\transfer(\lambda)}_{i, k} &= \begin{cases}
		(-1)^{k+1} \alpha_{k}(\lambda) & \text{when} ~ i = 1 ~ \text{and} ~ 1 \leq k \leq d \\ 
		1 & \text{when} ~ 2 \leq i \leq d - 1 ~ \text{and~}  k = i - 1 \\ 
		0 & \text{otherwise}  
	\end{cases}, 
\end{align*} 
where the value of $\alpha_k(\lambda)$ is given in \pref{def:alpha_definition}. 
\end{definition}

The following technical result considers the eigenspectrum of the matrix $P(\lambda)$. 

\begin{lemma}
\label{lem:P_eigenvalues} 
For any $\lambda \in \bbC^d$,  the eigenvalues of the matrix $\transfer(\lambda)$ are given by $\prn*{\lambda_1, \ldots, \lambda_k}$. 
\end{lemma}
\begin{proof} 
For the ease of notation, define $\alpha_{k}$ to denote $\alpha_{k, k}(\lambda)$ for $k \in [d]$. We start by computing the characteristic polynomial of the matrix $\transfer(\lambda)$, which is given by 
\begin{align*} 
\det(z I- \transfer(\lambda)) &= \det{ \begin{bmatrix}  (z - \alpha_{1}) & \alpha_{2}  & - \alpha_{3} & \cdots & (-1)^{d} \alpha_{d} \\  
- 1  & z & 0 & \cdots & 0 \\ 
0 & -1 & z & \cdots & 0 \\ 
& \vdots & & \cdots & \vdots & \\ 
0 & 0 & 0 & \cdots & z \\ 
 \end{bmatrix}}. 
\end{align*}

Computing the determinant by expansing along the first row, we get that 
\begin{align*}
	\det(z I- \transfer(\lambda)) &= (z - \alpha_{1})z^{d-1} + \sum_{k=2}^{d} (-1)^{k+1} \cdot ( (-1)^{k} \alpha_{k}) \cdot (-1)^{k-1} \cdot z^{d-k} \\ 
	&= z^d - \alpha_{1} z^{d-1} + \alpha_{k} z^{d-2} + \ldots + (-1)^{d} \alpha_{d}.  
\end{align*}

Using the definition of $\alpha_k$ from \pref{def:alpha_definition}, we can factorize the above polynomial as
\begin{align*}
	\det(z I- \transfer(\lambda))  &= \prod_{k=1}^d (z - \lambda_k). 
\end{align*}

Since, the eigenvalues of any matrix are given by the roots of its characteristic polynomial, the above computation shows that the eigenvalues of the matrix $\transfer(\lambda)$ are given by $( \lambda_1, \ldots, \lambda_d)$. 
\end{proof}

The following structural lemma shows that for any autoregression with coefficients $(\alpha_1(\lambda), \ldots, \alpha_k(\lambda))$, the $(m + d)$-th term can be expressed using the $m$-th power of the matrix $P(\lambda)$. Recall that the expected rewards for any policy satisfy such an autoregression whenever the underlying MDP has low rank (see \pref{lem:basic_recurrence_relation}).

\begin{lemma}
\label{lem:recurrence_extension}
Let $\lambda \in \bbC^d$ and $R_1, \ldots, R_d \in \bbR$. For any $h \geq d + 1$, let $R_h$ be given by 
\begin{align*} 
R_h = \sum_{k=1}^{d} (-1)^{k+1} \alpha_k(\lambda) R_{h - k}, \numberthis \label{eq:recurrence_extension1}
\end{align*}
where the coefficient $\alpha_k(\lambda)$ are defined in \pref{def:alpha_definition}. Then, for any $m \geq 0$,  
\begin{align*}
R_{m + d} = \tri*{U, P(\lambda)^{m} V} \numberthis \label{eq:recurrence_extension}
\end{align*} 
where the vector $U \ldef{} (1, 0, \ldots, 0)^\top \in \bbR^d$, the vector $V \ldef{} \prn{R_d, R_{d-1}, \ldots, R_1}^\top \in \bbR^d$ and the matrix $P(\lambda) \in \bbC^d$ is defined in \pref{def:P_beta}. 
\end{lemma} 
\begin{proof} For any $h \geq d$, define the vector $U_h \in \bbR^d$ such that
\begin{align*}
	U_h \ldef{} \prn*{R_h, R_{h-1}, \ldots, R_{h - d + 1}}^\top.
\end{align*} 
We first note that for any $j \in [d]$ such that $j \geq 2$, 
\begin{align*}
U_h[j] = R_{h - (j) - 1} = U_{h-1}[j - 1]. 
\end{align*}
Further, using the recurrence relation \pref{eq:recurrence_extension1}, we get that for any $h \geq d + 1$, 
\begin{align*}
U_h[1] = R_{h} = \sum_{k=1}^{d} (-1)^{k+1} \alpha_k(\lambda) R_{h - k} = \sum_{k=1}^{d} (-1)^{k+1} \alpha_k(\lambda) U_{h - 1}[k].   
\end{align*}
The above two relation imply that for any $h \geq d + 1$, 
\begin{align*}
U_{h} = P(\lambda) U_{h-1}, \numberthis \label{eq:recurrence_extension3}
\end{align*}
where the matrix $P(\lambda)$ is defined such that
\begin{align*}
	\brk{\transfer(\lambda)}_{j, k} &= \begin{cases}
		(-1)^{k+1} \alpha_{k}  & \text{when} ~ j = 1 ~ \text{and} ~ 1 \leq k \leq d \\ 
		1 & \text{when} ~ 2 \leq j \leq d ~ \text{and~}  k = j + 1 \\ 
		0 & \text{otherwise} 
	\end{cases}. 
\end{align*}
Setting $h = d + m $ in  relation \pref{eq:recurrence_extension3}, we get that for any $m$
\begin{align*}
	U_{m + d} = P(\lambda) U_{m - 1 + d} = \cdots = P(\lambda)^{m} U_d. 
\end{align*} 
Finally, we note that for any $m \geq 0$, 
\begin{align*} 
R_{m + d} = \tri*{V, U_{m + d}} = \tri*{V, P(\lambda)^m U_d}, 
\end{align*} where the vector $V = \prn*{1, 0, \ldots, 0} \in \bbR^d$ and the vector $U_d = \prn*{R_d, \ldots, R_1}^\top \in \bbR^d$.   
\end{proof}

We are finally ready to prove  \pref{lem:basic_ep_bound}. The following proof is based on the  extension of the Cayley-Hamilton theorem for rank $d$ matrices (see \pref{lem:cayley_hamilton_extension}) and uses \pref{lem:recurrence_extension}.  

\firstEPbound* 

\begin{proof} Using \pref{lem:recurrence_extension} for the sequences $\crl{R_{h}}$ and $\crl{\wt R_{h}}$ respectively, we get that for any $m \geq 0$, 
\begin{align*}
R_{d + m} = \tri{U, \transfer(\lambda)^m V} \qquad \text{and} \qquad 
\wt R_{d + m} = \tri{\wt U, \transfer(\wh \lambda)^m \wt V}, 
\end{align*}
where the matrices $\transfer(\lambda), \transfer(\wh \lambda) \in \bbR^{d \times d}$ are defined according to \pref{def:P_beta} and the vectors $U, \wt U, V, \wt V \in \bbR^d$ are independent of $\lambda$ and $m$. Thus, for any $m \geq 0$, 
\begin{align*} 
\abs{R_{m + d} - \wt R_{m + d}} &=  \abs{\tri{U, \transfer(\lambda)^m V} - \tri{\wt U, \transfer(\wh \lambda)^m V}} \\ 
&= \abs{\tri{ \bar{U}, \bar{\transfer}^m \bar{V}}},  \numberthis \label{eq:basic_ep_bound3} 
\end{align*} where the vectors $\bar{U}, \bar{V} \in \bbR^{2d}$ and the block diagonal matrix $\bar{P} \in \bbR^{2d \times 2d}$ are defined as 
\begin{align*}
\bar{V} \ldef{} \begin{bmatrix} V \\ - \wt V \end{bmatrix}, \quad \bar{U} \ldef{} \begin{bmatrix} U  \\ - \wt U \end{bmatrix} \quad \text{and} \quad \bar{P} \ldef{} \begin{bmatrix} \transfer(\lambda) & 0 \\ 
0 & \transfer(\wh \lambda) 
\end{bmatrix}. 
\end{align*} 
An application of \pref{lem:P_eigenvalues} implies that the eigenvalues of the matrix $P(\lambda)$ and the matrix $P(\wh \lambda)$ are given by $\lambda$ and $\wh \lambda$ respectively. Since the matrix $\bar{P}$ is block-diagonal, we note that the set of eigenvalues of the matrix $\bar{P}$ is given by $\bar{\lambda} = \prn{\lambda_1, \wh \lambda_1, \ldots, \lambda_d, \wh \lambda_d}$. Note that the vector $\bar{\lambda}$ is not sorted except for the first two coordinates, however $\abs{\bar{\lambda}_{k}} \leq 1$ for all $k \in [2d]$. Using \pref{lem:cayley_hamilton_extension} for $2d \times 2d$ matrix $\bar{P}$, we get that for any   $m \geq 2d + 1$, 
\begin{align*}
	\bar{P}^{2d + m + 1} = \sum_{k=1}^{2d} \beta_{m, k} (\bar{\lambda}) \cdot  \bar{P}^{2d + 1 - k}. 
\end{align*} 
Using the above relation with \pref{eq:basic_ep_bound3} and setting $m = h - 3d - 1$, we get that for any $h \geq 3d + 1$,
\begin{align*} 
\abs{R_{h} - \wt R_{h}}  = \abs*{\tri{ \bar{U}, \bar{\transfer}^{h -d} \bar{V}}} 
&= \abs[\Big]{\tri[\big]{\bar{U}, \sum_{k=1}^{2d} \beta_{h - 3d - 1, k}(\bar{\lambda}) \cdot \bar{P}^{2d + 1 - k} \bar{V} }} \\ 
&= \abs[\Big]{\sum_{k=1}^{2d} \tri{\bar{U},  \beta_{h - 3d - 1, k}(\bar{\lambda}) \cdot \bar{P}^{2d + 1 - k} \bar{V}}}. 
\end{align*}  
Using the triangle inequality on the right-hand side in the above, we obtain:
\begin{align*}
\abs{R_{h} - \wt R_{h}}  
& \leq \sum_{k=1}^{2d} \abs*{\beta_{h - 3d - 1, k} (\bar{\lambda})} \cdot \abs{ \tri*{\bar{U}, \bar{P}^{2d + 1 - k} \bar{V}}} \\ 
&\overeq{\proman{1}} \sum_{k=1}^{2d} \abs{\beta_{h - 3d - 1, k} (\bar{\lambda})} \cdot \abs{ R_{3d + 1 - k} - \wt R_{3d + 1 - k}} \\ 
&\overleq{\proman{2}}  2d \cdot  \prn[\Big]{\frac{4 e \max\crl{h - 3d - 1 + k,  2d}}{d}}^{2d} \cdot \max_{h' \leq 3d} ~ \abs{ R_{h'} - \wt R_{h'}} \\
&\leq  2d \cdot  \prn[\Big]{\frac{16 e \max\crl{h,  d}}{d}}^{2d} \cdot \max_{h' \leq 3d} ~ \abs{ R_{h'} - \wt R_{h'}}, \\
&= 2d \cdot  \prn[\Big]{\frac{16 e h}{d}}^{2d} \cdot \max_{h' \leq 3d} ~ \abs{ R_{h'} - \wt R_{h'}},
\end{align*}
where the equality in $\proman{1}$ holds due to relation \pref{eq:basic_ep_bound3} and the inequality $\proman{2}$ is given by the bound on $\abs{\beta_{h - 3d - 1, k} (\bar{\lambda})}$ from \pref{lem:bounded_coefficients}. The last line is due to the fact that $h > 3d$. 
\end{proof} 

\subsection{Supporting technical results for the proof of \pref{thm:basic_sample_complexity_bound}}   
\begin{lemma}
\label{lem:initial_conditions}  
Let $\lambda \in \bbC^d$ be such that $\abs{\lambda_k} \leq 1$ for all $k \in [d]$. Using the initial values $R_1, \ldots, R_{d}$, let $R_{h}$ be defined as 
\begin{align*} 
R_{h} \ldef{} \sum_{k=1}^d (-1)^{k+1} \alpha_k(\lambda) R_{h-k}.  \numberthis \label{eq:initial_error_bound1} 
\end{align*} 
Further, let  $\wh R_1, \wh R_2, \ldots, \wh R_{3d}$  denote the estimates for $R_1, \ldots, R_{3d}$ respectively, such that 
\begin{align*}
\max_{h \leq 3d} ~ \abs{\wh R_{h} - R_{h}} \leq \eta. \numberthis \label{eq:initial_error_bound2}    
\end{align*}
Then, 
\begin{enumerate}[label=(\alph*), leftmargin=8mm]
\item The optimization problem \pref{eq:optimization_problem} in \pref{alg:value_estimation_extrapolation} has a solution $(\wh \lambda, \wh \Delta)$ such that $$\abs{\wh \Delta} \leq 2d \cdot 4^d \cdot \eta.$$    
\item Further, let $\wt R_{h}$ be predictions according to \pref{line:reward_prediction1} in  \pref{alg:value_estimation_extrapolation} using the solution $\wh \lambda$. Then, 
\begin{align*}
\max_{h \leq 3d} ~ \abs{\wt R_h - R_h} \leq 2d \cdot \prn{64 e}^{d} \cdot \eta. 
\end{align*}
\end{enumerate} 
\end{lemma} 
\begin{proof} We prove the two parts separately below. 

\begin{enumerate}[label=$(\alph*)$, leftmargin=8mm] 
\item We first show that the optimization problem in \pref{eq:optimization_problem} is feasible. Specifically, we show that there exists a tuple $(\lambda', \Delta')$ that satisfies all the constraints in \pref{eq:optimization_problem} such that $\abs{\Delta'} \leq 2d 4^d \eta$. Set $\lambda' = \lambda$. We note that $\abs{\lambda'_1} = 1$ and  $\abs{\lambda'_k} \leq 1$ for all $k \leq d$ and thus all the constraints in \pref{eq:optimization_problem} are satisfied. Furthermore, for any $h \leq 3d$, 
\begin{align*}
\abs{\sum_{k=1}^d (-1)^{k+1} \alpha_{k}(\lambda) \wh R_{h - k} - \wh R_{h}} &\overeq{\proman{1}}   \abs{\sum_{k=1}^d (-1)^{k+1} \alpha_{k}(\lambda) \prn{\wh R_{h - k} - R_{h-k}} - (\wh R_{h} - R_h)} \\ 
&\overleq{\proman{2}}  \sum_{k=1}^d  \abs{\alpha_{k}(\lambda)} \cdot \abs{\wh R_{h - k} - R_{h-k}} + \abs{\wh R_{h} - R_h}  \\
&\overleq{\proman{3}} d \cdot 4^d \cdot \eta + \eta \\
&\leq 2d \cdot 4^d \cdot \eta \numberthis \label{eq:initial_error_bound3}   
\end{align*} where the equality $\proman{1}$ follows from the relation \pref{eq:initial_error_bound1} and the inequality $\proman{2}$ follows from Triangle inequality. The inequality $\proman{3}$ follows by plugging in the bound from \pref{lem:alpha_independent_bound} for $\abs{\alpha_k(\lambda)}$ and using the bound in  $\pref{eq:initial_error_bound2}$. The above implies that $\abs{\Delta'} \leq 2d 4^d \eta$. 

Thus, any solution $(\wh \lambda, \wh \Delta)$ of the optimization problem  in \pref{eq:optimization_problem} must satisfy  
\begin{align*}
\abs{\wh \Delta} \leq 2d \cdot 4^d \cdot \eta. \numberthis \label{eq:initial_error_bound9}   
\end{align*}

\item Let us first define some additional notation. For any $m \leq 2d$, define $\Delta_m$ as the error for $m$th expected reward when plugging in the minimizer solution $\wh \lambda$, i.e.,
\begin{align*}
\wh \Delta_m \ldef{} \sum_{k=1}^d (-1)^{k+1} \alpha_{k}(\wh \lambda) \cdot \wh R_{d + m - k} - \wh R_{d + m}. \numberthis \label{eq:initial_error_bound0}
\end{align*} 
Further, define $Z_m$ as the error in our prediction for the expected reward at $(m + d)$th time step, i.e. 
\begin{align*} 
Z_{m} &\ldef{} \wt R_{m + d} - \wh R_{m + d}.    \numberthis \label{eq:initial_error_bound3}
\end{align*} 

In the following, we will show that for all $m \geq 1$, 
\begin{align*} 
Z_m = \wh \Delta_m + \sum_{i=1}^{m - 1} \beta_{i - 1, 1}(\wh \lambda) \cdot \wh \Delta_{m - i}  \numberthis \label{eq:initial_error_bound}, 
\end{align*} where the coefficients $\beta_{i-1, 1}$ are given in \pref{def:P_beta}. 

Our desired bound follows as a direct consequence of \pref{eq:initial_error_bound}. For any $1 \leq m \leq 2d$,  
\begin{align*}
\abs{\wt R_{m + d} - R_{m + d}}  &\leq \abs{\wt R_{m + d} - \wh R_{m + d}} + \abs{\wh R_{m + d} - R_{m + d}} \\ 
&\overleq{\proman{1}}  \abs{ \wh \Delta_m + \sum_{i=1}^{m - 1} \beta_{i - 1, 1}(\wh \lambda) \cdot \wh \Delta_{m - i} } + \eta \\ 
&\overleq{\proman{2}} \abs{\wh \Delta_m} + \sum_{i=1}^{m-1} \abs{\beta_{i-1, 1} (\wh \lambda)} \abs{\wh \Delta_{m-i}} + \eta \\
&\overleq{\proman{3}} \abs{\wh \Delta} + \sum_{i=1}^{m-1} \abs{\beta_{i-1, 1} (\wh \lambda)} \abs{\wh \Delta} + \eta \\ 
&\overleq{\proman{4}} 2d \cdot 4^d \cdot \prn[\Big]{\frac{8e \max \crl{m, d}}{d}}^d \cdot \eta \\
&\leq   2d \cdot \prn{64 e}^{d} \cdot \eta, 
\end{align*} 
where the inequality $\proman{1}$ follows from the definition of $Z_m$ in \pref{eq:initial_error_bound3} and by using the bound in \pref{eq:initial_error_bound2}, and the inequality $\proman{2}$ above is due to Triangle inequality. The inequality $\proman{3}$ above follows by using the fact that $\abs{\wh \Delta_m} \leq \abs{\wh \Delta}$ for all $m \leq 2d$. Finally, the inequality $\proman{4}$ follows by plugging in the bound in \pref{eq:initial_error_bound9} and by using \pref{lem:bounded_coefficients} to bound $\abs{\beta_{i-1, 1}(\wh \lambda)}$.  
\end{enumerate}

\paragraph{Proof of relation \pref{eq:initial_error_bound}.} We prove this by induction over $m$. For the base case, when $m = 1$, 
\begin{align*} 
Z_{d + 1} &= \wt R_{d + 1}  - \wh R_{d + 1} 
\overeq{\proman{1}}  \sum_{k=1}^{d} (-1)^{k+1} \alpha_{k}(\wh \lambda) \cdot \wh R_{k} - \wh R_{d + 1}  \overeq{\proman{2}} \wh \Delta_{1}, 
\end{align*} where the equality in $\proman{1}$ follows from the definition of $\wt R_{d + 1}$ holds due to \pref{eq:initial_error_bound1} and $\proman{2}$ follows from the definition of $\wh \Delta_1$. 

We next show the induction step. For any $m \geq 2$, suppose that the relation \pref{eq:initial_error_bound} holds for all times $m' < m$. We note that  
\begin{align*} 
	Z_{d + m} &= \wt R_{d + m} - \wh R_{d + m} \\ 
			   &\overeq{\proman{1}} \sum_{k=1}^{d} (-1)^{k + 1} \alpha_{k}(\wh \lambda) \cdot \wt R_{d + m-k} - \wh R_{d + m} \\ 
			   &= \sum_{k=1}^{d} (-1)^{k + 1} \alpha_{k}(\wh \lambda)  \cdot \wt R_{d + m-k} - \sum_{k=1}^d (-1)^{k + 1}\alpha_{k}(\wh \lambda) \cdot \wh R_{d + m - k}  + \sum_{k=1}^d (-1)^{k + 1} \alpha_{k}(\wh \lambda)  \cdot \wh R_{d + m - k}  - \wh R_{d + m} \\ 
			   &\overeq{\proman{2}} \sum_{k=1}^{d} (-1)^{k + 1} \alpha_{k}(\wh \lambda) \cdot \wt R_{d + m-k} - \sum_{k=1}^d (-1)^{k + 1}\alpha_{k}(\wh \lambda) \cdot \wh R_{d + m - k}  + \wh \Delta_m \\ 
			   &= \sum_{k=1}^{d} (-1)^{k + 1} \alpha_{k}(\wh \lambda) \cdot \prn*{\wt R_{d + m-k} - \wh R_{d + m - k}} + \wh \Delta_{m} \\   
			   &\overeq{\proman{3}} \sum_{k=1}^{d} (-1)^{k + 1} \beta_{0, k}(\wh \lambda) \cdot \prn*{\wt R_{d + m-k} - \wh R_{d + m - k}} + \wh \Delta_{m} \\   
			   &\overeq{\proman{4}} \sum_{k=1}^{d} (-1)^{k + 1} \beta_{0, k} \cdot Z_{m - k} + \wh \Delta_{m}.
\end{align*} where $\proman{1}$ follows from the definition of $\wt R_{d + m}$ (see \pref{eq:reward_prediction_autoregression}) and $\proman{2}$ follows by the definition of $\wh \Delta_m$ in \pref{eq:initial_error_bound0}. The equality $\proman{3}$ above is due to the fact that $\beta_{0, k}(\wh \lambda) = \alpha_{k}(\wh \lambda)$ (by definition) and finally, the equality $\proman{4}$ follows from the definition of $Z_{m-k}$ in \pref{eq:initial_error_bound3}. Plugging in the induction hypothesis for $Z_{m -k}$ in the above, we get that   
\begin{align*} 
	Z_{d + m} &= \sum_{k=1}^{d} (-1)^{k + 1} \beta_{0, k}(\wh \lambda) \cdot \prn[\Big]{\wh \Delta_{m - k} + \sum_{j=1}^{m - k - 1} \beta_{j - 1, 1}(\wh \lambda) \cdot \wh \Delta_{m - k - j}} + \wh \Delta_{m} \\ 
	&= \wh \Delta_m + \sum_{i=1}^{m-1} \wh \Delta_{m - i} \cdot \prn[\Big]{(-1)^{i+1} \beta_{0, i}(\wh \lambda) +  \sum_{j=1}^{i-1} (-1)^{i-j-1} \beta_{j - 1, 1}(\wh \lambda) \cdot  \beta_{0, i-j}} \\ 
	&= \wh \Delta_m + \sum_{i=1}^{m-1}  \beta_{i - 1, 1}(\wh \lambda) \cdot \wh \Delta_{m - i},   
\end{align*} where the second line above follows by rearranging the terms and using the fact that $\beta_{0, k}(\wh \lambda) = 0$ whenever $k > d$, and the equality in the last line holds by using the fact that  $\beta_{0, k}(\wh \lambda) \cdot \beta_{h-1,1}(\wh \lambda) = \beta_{h-1, k + 1}(\wh \lambda) + \beta_{h, k}(\wh \lambda)$ for all $h, k \geq 0$ (see \pref{def:P_beta}). This completes the induction step, hence proving \pref{eq:initial_error_bound} for all $m \geq 1$. 
\end{proof} 

\subsection{Proof of \pref{thm:basic_sample_complexity_bound}} \label{app:basic_sample_complexity_bound}

We finally provide the proof of \pref{thm:basic_sample_complexity_bound} that characterizes the performance guarantee for the policy $\wt \pi$ returned by \pref{alg:policy_search}. 

\noindent
\begin{proof}[{Proof of \pref{thm:basic_sample_complexity_bound}} ]
Starting from \pref{lem:ss_ub}, we get that with probability at least  $1 - \delta$, for every policy $\pi \in \Pi$, our estimate $\wh R\ind{\pi}_h$ computed in  \pref{line:reward_estimation1} of \pref{alg:value_estimation_extrapolation}  satisfies the error bound 
\begin{align*} 
\max_{h' \leq 3d} ~ \abs{\wh R\ind{\pi}_{h'} - R\ind{\pi}_{h'}} &\leq \min \crl[\bigg]{ \sqrt{\frac{8 K^{3d} \log(6 d  \abs*{\Pi}/ \delta)}{n}}, \frac{4 K^{3d} \log(6 d \abs*{\Pi}/ \delta)}{n} }.  \numberthis  \label{eq:basic_autoregression4} 
\end{align*} 

Now, consider any policy $\pi \in \Pi$, and let $\wh \lambda\ind{\pi}$, $\wh \Delta\ind{\pi}$, $\wt R\ind{\pi}_h$ and $\wt V \ind{\pi}$ denote the corresponding local variables in the procedure $\valestimate$ when invoked in \pref{alg:policy_search} for the policy $\pi$. Further, let $\lambda\ind{\pi}$ denote the eigenvalues of the transition matrix $T\ind{\pi}$. As a consequence of \pref{lem:basic_recurrence_relation}, the expected rewards $R_h\ind{\pi}$ satisfy an autoregression where the coefficients are determined by $\lambda\ind{\pi}$. Specifically, for any $h \geq d +1$, 
\begin{align*} 
R\ind{\pi}_{h} = \sum_{k=1}^d (-1)^{k+1} \alpha_k(\lambda \ind{\pi}) \cdot R\ind{\pi}_{h-k}.  
\end{align*} 
Furthermore, by definition (see \pref{line:reward_prediction1} of \pref{alg:value_estimation_extrapolation}), the predicted rewards $\wt R\ind{\pi}_h$ also satisfy a similar autoregression where the coefficients are determined by $\wh \lambda \ind{\pi}$, the solution of the optimization problem in \pref{eq:optimization_problem} for the policy $\pi$. We have, for any $h \geq d + 1$, 
\begin{align*}
 	\wt R_{h}\ind{\pi} =   \sum_{k=1}^d (-1)^{k+1} \alpha_{k} (\wh \lambda \ind{\pi}) \cdot \wt R_{h - k}\ind{\pi} 
\end{align*}
where $\wt R_{h'} \ldef{} \wh R_{h'}$ for $h' \leq d$. Additionally, also note that $T\ind{\pi}$  is a stochastic matrix and thus $\abs{\lambda_k\ind{\pi}} \leq 1$ for all $k \in [d]$. By definition, we also have that $\abs{\wh \lambda_k\ind{\pi}} \leq 1$. Thus, using the error propagation bound in \pref{lem:basic_ep_bound} for the sequences $\crl{R_{h}\ind{\pi}}$ and $\crl{\wt R\ind{\pi}_h}$ we get that for any  $h \geq 3d + 1$, 
\begin{align*} 
	\abs{\wt R\ind{\pi}_h - R\ind{\pi}_h} &\leq 2d \cdot  \prn[\Big]{\frac{16 e h}{d}}^{2d} \max_{h' \leq 3d} ~ \abs{ \wt R_{h'} - R_{h'}}.  
\end{align*} 
The above bound implies that for any $h \geq 1$,  
\begin{align*} 
	\abs{\wt R\ind{\pi}_h - R\ind{\pi}_h} &\leq 2d \cdot  \prn[\Big]{\frac{16 e (h \vee d)}{d}}^{2d} \max_{h' \leq 3d} ~ \abs{\wt R_{h'} - R_{h'}} \numberthis \label{eq:general_sc1} 
\end{align*}
We note that an application of  \pref{lem:initial_conditions} implies that the predicted rewards $\wt R_{h'}\ind{\pi}$ satisfy the error bound 
\begin{align*}
\qquad  \max_{h' \leq 3d} ~ \abs{\wt R_{h'}\ind{\pi} - R_{h'}\ind{\pi}} &\leq 2d \cdot (64 e)^d \cdot \max_{h' \leq 3d} ~ \abs{\wh R\ind{\pi}_{h'} - R\ind{\pi}_{h'}} \\
&\leq  2d \cdot (64 e)^d \cdot \eta,   
\end{align*}
where $\eta$ denotes the right hand side of \pref{eq:basic_autoregression4}. 
Plugging the above in \pref{eq:general_sc1}, we get that 
\begin{align*} 
	\abs{\wt R\ind{\pi}_h - R\ind{\pi}_h} &\leq 4d^2 \cdot  \prn[\Big]{\frac{128 e^2 (h \vee d)}{d}}^{2d} \cdot \eta. \numberthis \label{eq:general_sc5}
\end{align*} for any $h \geq 1$. Thus, the error in the estimated value $\wt V\ind{\pi}$ for the policy $\pi$ is bounded by 
\begin{align*}
\abs{\wt V\ind{\pi} - V \ind{\pi}} 
&= \abs{\sum_{h=1}^{H} \prn{\wt R_h\ind{\pi} - R_h \ind{\pi}} }   \\
&\leq \sum_{h = 1}^H \abs{{\wt R_h\ind{\pi} - R_h \ind{\pi}}} \\
&\leq \sum_{h=1}^H 4d^2 \cdot  \prn[\Big]{\frac{128 e^2 (h \vee d)}{d}}^{2d} \cdot \eta \\
&\leq 4 d^3 \cdot  \prn[\Big]{\frac{128 e^2 H}{d}}^{2d} \cdot \eta,   \numberthis \label{eq:general_sc4} 
\end{align*} where the inequality in the second last line follows by using the bound in \pref{eq:general_sc5}, and the inequality in the last line holds because $H \geq d$. 

Since $\pi$ is arbitrary in the above chain of arguments, the error bound in \pref{eq:general_sc4} holds for all policies $\pi \in \Pi$. Thus, for any $\pi \in \Pi$, the policy $\wt \pi$ returned in \pref{line:choose_optimal_policy} of \pref{alg:policy_search} satisfies 
\begin{align*}
V\ind{\wt \pi} - V\ind{\pi} &= \prn{\wt V\ind{\pi} - V\ind{\pi}}  +  \prn{\wt V\ind{\wt \pi} - \wt V \ind{\pi}} + \prn{V\ind{\wt \pi} - \wt V\ind{\wt \pi}} \\
&\geq \prn{\wt V\ind{\pi} - V\ind{\pi}} + \prn{V\ind{\wt \pi} - \wt V\ind{\wt \pi}} \\
&\geq - \abs{\wt V\ind{\pi} - V\ind{\pi}} - \abs{V\ind{\wt \pi} - \wt V\ind{\wt \pi}}, 
\end{align*} where the inequality in the second line follows from the fact that $\wt V\ind{\wt \pi} \geq \wt V\ind{\pi}$ for every $\pi \in \Pi$ by the definition of the policy $\wt \pi$. Using the bound from \pref{eq:general_sc4} for policies $\pi$ and $\wt \pi \in \Pi$ in the above, we get that 
\begin{align*}
V\ind{\wt \pi} &\geq V\ind{\pi} - 4 d^3 \cdot  \prn[\Big]{\frac{128 e^2 H}{d}}^{2d} \cdot \eta \\
&\geq V\ind{\pi} - 4 d^3 \cdot  \prn[\Big]{\frac{128 e^2 H}{d}}^{2d} \cdot \min \crl[\bigg]{ \sqrt{\frac{8 K^{3d} \log(6 d \abs*{\Pi}/ \delta)}{n}}, \frac{4 K^{3d} \log(6 d \abs*{\Pi}/ \delta)}{n} } \\ 
&\geq V\ind{\pi} - 4 d^3 \cdot  \prn[\Big]{\frac{128 e^2 H}{d}}^{2d} \sqrt{\frac{8 K^{3d} \log(6 d \abs*{\Pi}/ \delta)}{n}} 
\end{align*} 
where the inequality in the second line above follows by plugging in the value of $\eta$ as the right hand side of \pref{eq:basic_autoregression4}, and the inequality in the last line holds due to the fact that $-\min\crl{a, b} \geq -a$ for any $a, b \geq 0$. 

Since the above holds for any $\pi \in \Pi$, we have that 
\begin{align*}
V\ind{\wt \pi} &\geq \max_{\pi \in \Pi} V\ind{\pi} - 4 d^3 \cdot  \prn[\Big]{\frac{128 e^2 H}{d}}^{2d} \sqrt{\frac{8 K^{3d} \log(6 d \abs*{\Pi}/ \delta)}{n}}, 
\end{align*} 
hence proving the desired statement. 
\end{proof} 

\section{Adaptive upper bounds}
In this section, we present \pref{alg:adaptive_policy_search} whose performance guarantee adapts to the unknown eigenspectrum of the underlying transition matrix. We then proceed to the proof of our adaptive upper bound \pref{thm:adaptive_upper_bound_main}. 

\subsection{Adaptive policy search algorithm} 
\begin{algorithm}[H]
\caption{Adaptive policy search algorithm (Adaptivity to unknown eigenspectrum)}  
\begin{algorithmic}[1]
	\Require horizon $H$, rank $d$,  number of episodes $n$, finite policy class $\Pi$ 
	\State Collect the dataset $\cD = \crl{\prn{x\ind{t}_h, a\ind{t}_h, r\ind{t}_h}_{h=1}^H}_{t = 1}^n$ by sampling $n$ trajectories where actions are sampled from $\uniform(\cA)$.
	\For {\text{policy} $\pi \in \Pi$} 
\State \label{line:ada_call_val_estimate}Estimate $\wt V\ind{\pi}$ by calling \adavalestimate($H, d, \cD, \pi$).
\EndFor  
\State\label{line:ada_choose_optimal_policy}\textbf{Return: } policy $\wt \pi$ with best estimated value $\wt \pi \in \argmax_{\pi \in \Pi} \wt V\ind{\pi}$.  
\end{algorithmic} 
\label{alg:adaptive_policy_search}  
\end{algorithm} 

\begin{algorithm}[H] 
\caption{Adaptive value estimation by autoregressive extrapolation} 
\begin{algorithmic}[1] 
\Function{\adavalestimate}{$H, d, \cD, \pi$}: 
\State Set $\Delta = 2d 4^d \min\crl[\Big]{\sqrt{\frac{8 K^{3d} \log(6d \abs*{\Pi} / \delta )}{n}}, {\frac{4 K^{3d} \log(6d \abs*{\Pi} / \delta )}{n}}}$.  
\For {\text{time step} $h= 1, \ldots, 3d$}   
\State\label{line:adaptive_reward_estimation1} Estimate expected rewards by importance sampling $$\wh R_h = \frac{1}{n} \sum_{i =1}^n r_h\ind{i} \prod_{h' \leq h} \prn*{K\indicator{\pi(x\ind{i}_{h'}) = a_{h'}\ind{i}}}$$   
\EndFor 
\State\label{line:adaeig_optimization} Estimate eigenvalues of the autoregression by solving the optimization problem: 
\begin{align*} 
   \wh \lambda \leftarrow \argmin_{\lambda \in \mathbb C^d} &~ 
   \prod_{k=2}^d\prn[\big]{\sum_{h=0}^{H-1} \abs{\lambda_k}^h}
   \numberthis \label{eq:adaptive_optimization_problem} \\ 
   \text{s.t.} ~~ &|\lambda_{1}| = 1, |\lambda_{k}| \leq 1 &\textrm{for } 2 \leq k \leq  d, \\  
   & \abs[\Big]{\sum_{k=1}^d (-1)^{k+1} \alpha_{k}(\lambda) \wh R_{h - k} - \wh R_{h} } \leq \Delta &\textrm{for } d + 1 \leq h \leq 3d.    
\end{align*} 
\State\label{line:adaptive_reward_prediction} Predict $\wt R_{h}$ as: 
\begin{align*} 
 	\wt R_{h} = 
\begin{cases}
\wh R_{h} & \text{for ~} 1 \leq h \leq d \\ 
\sum_{k=1}^d (-1)^{k+1} \alpha_{k}(\wh \lambda) \wt R_{h - k} & \text{for ~} d + 1 \leq h \leq H 
\end{cases}. 
\end{align*} 
\State\label{line:adaptive_value_prediction}\textbf{return:} Estimate of the value function $\wt V = \sum_{h=1}^H \wt R_h$.
\EndFunction
\end{algorithmic} 
\label{alg:adaptive_extrapolation} 
\end{algorithm} 

\subsection{Adaptive error propagation bound} 
The main technical innovation that leads to the adaptive upper bound in \pref{thm:adaptive_upper_bound_main} is the following bound on the propagated error in the $h$th step prediction. The bound in \pref{eq:adaptive_upper_bound_ep_main} adapts to the eigenvalues $\lambda$ and $\wh{\lambda}$, which define the auto-regressions for $\crl{R_h}$ and $\crl{\wt R_h}$ respectively. 

\begin{lemma}[Adaptive error propagation bound] 
\label{lem:adaptive_ep_bound_appx} 
Let $\lambda, \wh \lambda \in \bbC^d$ be such that $\max\crl{\abs*{\lambda_1}, \abs{\wh \lambda_1}} \leq 1$. 
Further, with the initial values  $\crl{R_1, \ldots, R_d}$ and $\crl{\wt R_1, \ldots, \wt R_{d}}$, let the sequence $\crl{\wt R_h}$ and $\crl*{R_h}$ be given by 
\begin{align*} 
	R_h = \sum_{k=1}^d (-1)^{k+1} \alpha_{k}(\lambda) \cdot R_{h - k } \qquad \text{and} \qquad  \wt R_h = \sum_{k=1}^d (-1)^{k+1} \alpha_{k}(\wh \lambda) \cdot \wt R_{h - k},  
\end{align*} 
where the coefficients $\alpha_{k}(\wh \lambda)$ and $\alpha_{k}(\lambda)$ are given in \pref{def:alpha_definition}. Then, for all $h \geq 1$, 
\begin{align*} 
	\abs{\wt R_h - R_h} &\leq 2^{2d}  h  \cdot \prod_{k=2}^{d} \prn[\Big]{\sum_{j=0}^{h-1} \abs{\lambda_k}^{j}} \cdot \prod_{k=2}^{d} \prn[\Big]{\sum_{j=0}^{h-1} \abs{\wh \lambda_k}^{j}} \cdot \max_{h' \leq 3d} ~ \abs{R_{h'} - \wt R_{h'}}.  \numberthis \label{eq:adaptive_upper_bound_ep_main}
\end{align*} 
\end{lemma} 

We defer the proof to \pref{app:adaptive_ep_bound_proof}. 

\subsubsection{Supporting technical results for the proof of \pref{lem:adaptive_ep_bound_appx}}  

\begin{lemma} 
\label{lem:z_md_zero} 
Given any vectors $u, v \in \bbR^d$ and a diagonalizable matrix $A \in \bbR^{d\times d}$ with eigenvalues $\lambda_1, \ldots, \lambda_d$ such that $\abs{\lambda_1} = 1$, let $z_{m, k}$ be defined such that 
\begin{align*} 
z_{m, k} & = \begin{cases}
	u^\top A^m v & \text{when} \quad k = 0 \\ 
	z_{m, k-1} - \lambda_{d + 1 - k} \cdot z_{m-1, k-1}  & \text{when} \quad 1 \leq k \leq d, 
\end{cases}. 
\end{align*} 
where $k \leq \min\crl*{d, m}$. 
Then, $z_{m, d} = 0$ for all $m \geq d$. Furthermore, for any $k \leq d$, 
the following inequality holds:
\begin{align*} 
\abs{z_{d, k}} \leq 2^k\max_{1\leq i \leq d}~ \abs{u^\top A^i v}.  
\end{align*} 
\end{lemma} 
\begin{proof} We prove the two statements separately below. 

\begin{enumerate}[label=(\alph*)]
\item We first show that $z_{m, d} = 0$ for all $m \geq d$. Since, the matrix $A$ is diagonalizable, we have    
\begin{align*} 
A = Q \Lambda Q^{-1}, 
\end{align*}
where $\Lambda = \diag \prn*{\lambda_1, \ldots, \lambda_d}$ and 
where $Q \in \bbR^{d \times d}$ is the matrix whose $k$th column is an eigenvector 
$q_k$ corresponding to the eigenvalue $\lambda_k$.   
In order to prove this, we will show that for any $m \geq 0$ and $k \leq \min(d, m)$, 
\begin{align*}
	z_{m, k} &= u^T Q D_{m, k} Q^{-1} v, \numberthis \label{eq:supporting_result_diagA1} 
\end{align*} where the matrix $D_{m, k} \in \bbR^{d \times d}$ is diagonal with entries given by 
\begin{align*}
\brk{D_{m, k}}_{i, i} &= \begin{cases}
	0 & \text{if} \quad d + 1 - k \leq i \leq d \\ 
	\lambda_i^{m - k} \prod_{k'=1}^{k} (\lambda_i - \lambda_{d + 1 - k'}) & \text{otherwise}.
\end{cases} \numberthis \label{eq:supporting_result_diagA2} 
\end{align*}
Specifically, the diagonal entry $\brk*{D_{m, k}}_{i, i} = 0$ for $i \geq d + 1 - k$.  Observe that for $m, k \geq 1$ and $i < d + 1 - k$, the following relation holds:
\begin{align*}
\brk{D_{m, k}}_{i, i}
& = \lambda_i^{m - k} \prod_{k' = 1}^{k} (\lambda_i - \lambda_{d + 1 - k'})\\
& = (\lambda_i - \lambda_{d + 1 - k}) \lambda_i^{m - k} \prod_{k' = 1}^{k - 1} (\lambda_i - \lambda_{d + 1 - k'})\\
& = \lambda_i^{m - k + 1} \prod_{k' = 1}^{k - 1} (\lambda_i - \lambda_{d + 1 - k'})  
- \lambda_{d + 1 - k} \, \lambda_i^{m - k} \prod_{k' = 1}^{k - 1} (\lambda_i - \lambda_{d + 1 - k'}) \\
\numberthis \label{eq:DkRelation}
    & = [D_{m, k - 1}]_{i, i} - \lambda_{d + 1 - k} \, [D_{m - 1, k - 1}]_{i, i}.
\end{align*}
Also, for $m, k \geq 1$ and $i = d + 1 - k$, the following relation holds:
\begin{align*}
\brk{D_{m, k}}_{i, i}
& = 0 \\  
& = \lambda_{d + 1 - k}^{m - k + 1} \prod_{k' = 1}^{k - 1} (\lambda_i - \lambda_{d + 1 - k'})
- \lambda_{d + 1 - k} \, \lambda_{d + 1 - k}^{m - k} \prod_{k' = 1}^{k - 1} (\lambda_i - \lambda_{d + 1 - k'})\\
\numberthis \label{eq:DkRelation2}
    & = [D_{m, k - 1}]_{d + 1 - k, d + 1 - k} - \lambda_{d + 1 - k} \, [D_{m - 1, k - 1}]_{d + 1 - k, d + 1 - k}.
\end{align*}
For $i > d + 1 - k$, by definition, we have
\begin{align*}
[D_{m, k}]_{i, i} 
& = [D_{m - 1, k}]_{i, i} = [D_{m - 1, k - 1}]_{i, i} = 0. 
\numberthis \label{eq:DkRelation3}
\end{align*}

We prove \pref{eq:supporting_result_diagA1} by an induction over the set of tuples $(m, k)$. The induction proceeds in a row-first manner by first keeping $m$ fixed and increasing $k$ from $1$ to $d$; we then increase $m$ to $m+1$ and proceed with the next row in the set of tuples $(m, k)$. For the base case, for $k =0$ and any $m \geq 0$,   
\begin{align*}
z_{m, 0} &= u^T A^m v \overeq{\proman{1}} u^T Q \Lambda^m Q^{-1} v \overeq{\proman{2}} u^T Q D_{m, 0} Q^{-1} v, 
\end{align*} where the equality $\proman{1}$ follows by using the fact that $A = Q \Lambda Q^{-1}$ and the inequality in $\proman{2}$ is given by the definition of the matrix $D_{m, 0}$. 

We next prove the induction step. For any $m \geq 0$ and $k \leq d$, suppose that \pref{eq:supporting_result_diagA1} holds for every tuple $(m', k')$ where  $m' < m$ 
and $k' \leq \min(m', d)$, and for every tuple $(m, k')$ where $0 \leq k' < k$. In the following, we will show that the relation \pref{eq:supporting_result_diagA1} will hold for the tuple $(m, k)$ as well. Using the definition of $z_{m, k}$, we get that 
\begin{align*}
z_{m, k} &= z_{m, k-1} - \lambda_{d + 1 - k} \, z_{m-1, k-1}  \\ 
&= u^T Q D_{m, k - 1} Q^{-1} v  - \lambda_{d + 1 - k}  u^T Q D_{m - 1, k - 1} Q^{-1} v  
\tag{Equation~\ref{eq:supporting_result_diagA1}}\\
&= u^T Q \prn{ D_{m, k - 1} - \lambda_{d + 1 - k}  D_{m - 1, k - 1} } Q^{-1} v \\
&=  u^T Q \prn{ D_{m, k}} Q^{-1} v.
\tag{Equations~\ref{eq:DkRelation}, \ref{eq:DkRelation2}, and \ref{eq:DkRelation3}}
\end{align*} 
This completes the induction step, thereby proving that  
\pref{eq:supporting_result_diagA1} holds for all $m \geq 0$ and $k \leq \min(d, m)$. 
Setting $k = d$ in relation \pref{eq:supporting_result_diagA1} gives
$D_{m, d} = \diag(0, \ldots, 0)$ for any $m \geq d$ and thus the following:
\begin{align*} 
z_{m, d} 
& =  u^T Q D_{m, d} Q^{-1} v 
= 0.  
\end{align*} 

\item In the following, we will show that for any $m \geq d$ and $k \leq \min\crl*{d, m}$, 
\begin{align*} 
\abs*{z_{m, k}} \leq 2^{k} \Delta,    \numberthis \label{eq:adaptive_gp_bound2}   
\end{align*} 
where $\Delta \ldef{} \max\crl{ \abs{ u^T A v}, \ldots,  \abs{ u^T A^d v}}$.  

We prove \pref{eq:adaptive_gp_bound2} by an induction over the set of tuples $(m, k)$. The induction proceeds in a row-first manner by first keeping $m$ fixed and increasing $k$ from $1$ to $d$; we then increase $m$ to $m+1$ and proceed with the next row in the set of tuples $(m, k)$. For the base case, we note that for $k = 0$ and any $m \leq d$, 
\begin{align*} 
\abs*{z_{m, 0}} &\leq u^T A^m v \leq \max\crl{ \abs{ u^T A v}, \ldots,  \abs{ u^T A^d v}} = \Delta.  
\end{align*}

We next show the induction step. Given any $m$ and $k$ such that $k \leq \min\crl*{d, m}$, assume that \pref{eq:adaptive_gp_bound2} holds for every tuple $(m', k')$ where  $m' < m$ and $k' \leq \min(m', d)$, and for every tuple $(m, k')$ where $0 \leq k' < k$. In the following, we will show that the relation \pref{eq:adaptive_gp_bound2} holds  for the tuple $(m, k)$ as well. Using the definition of $z_{m, k}$, we get that 
\begin{align*} 
\abs*{z_{m, k}} 
& = \abs*{z_{m, k-1} - \lambda_{d + 1 - k} \cdot z_{m-1, k-1}} \\
  					     & \leq \abs*{z_{m, k-1}} + \abs*{\lambda_{d + 1 - k}} \abs*{z_{m-1, k-1}} \\ 
  					     & \leq \abs*{z_{m, k-1}} + \abs*{z_{m - 1, k-1}}, 
\end{align*} 
where the last line holds because $\abs*{\lambda_{d + 1 - k}} \leq 1$.  Using the bound of  \pref{eq:adaptive_gp_bound2} for the tuples $(m, k - 1)$ and $(m - 1, k - 1)$, we obtain:
\begin{align*}
\abs*{z_{m, k}} &\leq 2^{k-1} \Delta + 2^{k-1} \Delta \leq 2^k \Delta. 
\end{align*}
This completes the induction step, hence proving  \pref{eq:adaptive_gp_bound2}  for all $m \geq d$ and $k \leq \min\crl*{d, m}$.

Finally, setting $m = d$ in \pref{eq:adaptive_gp_bound2} gives us the desired result. 
\end{enumerate}
\end{proof}

\begin{lemma} 
\label{lem:adaptive_power_bound_general} 
Let $A \in \bbR^{d \times d}$ be a matrix with eigenvalues 
$(\lambda_1, \ldots, \lambda_d)$ such that $\abs{\lambda_1} = 1$ 
and $\abs{\lambda_k} \leq 1$, for all $k \in [d]$. Then, for any
two vectors $u \in \bbR^d$ and $v \in \bbR^d$ and any $m \geq d +1 $,
the following inequality holds:
\begin{align*}
\abs{ u^T A^m v} &\leq 2^d \cdot \prod_{k=2}^d \prn[\Big]{ \sum_{j=0}^{m - d} \abs{\lambda_{k}}^j} \cdot \max\crl*{ \abs{ u^T A v}, \ldots,  \abs{ u^T A^d v}}.
\end{align*}
\end{lemma} 
\begin{proof} We will first prove the result  when the matrix $A$ has distinct eigenvalues. We will later extend the proof for general matrices $A$. 

\paragraph{Simpler setting: When $A$ has distinct eigenvalues.} 
We first introduce some notation to be used in the proof. Fix the vectors $u, v \in \bbR^d$. 
For any $m \geq 1$, and $k \leq \min\crl*{d, m}$, define $z_{m, k} \in \bbR$ 
as follows:
\begin{align*} 
z_{m, k} &= \begin{cases} 
	u^T A^m v & \text{when} \quad k = 0 \\ 
	z_{m, k-1} - \lambda_{d + 1 - k} \cdot z_{m-1, k-1}  & \text{when} \quad 1 \leq k \leq d.
\end{cases}.  \numberthis \label{eq:adaptive_gp_bound1}  
\end{align*} 
Further, define $\Delta \ldef{} 
\max\crl*{ \abs{ u^T A v}, \ldots,  \abs{ u^T A^d v}}$. Since the matrix $A$ has distinct eigenvalues, $A$ is diagonalizable and thus by \pref{lem:z_md_zero}, the following inequality holds for any $k \leq d$:
\begin{align*}
\abs{z_{d, k}} & \leq 2^k \Delta \leq 2^d \Delta. \numberthis \label{eq:adaptive_gp_bound4}  
\end{align*}

\noindent 
In order to prove the desired result, we first show that for any $k \leq d-1$, and any $m \geq d$, 
\begin{align*}
\abs{z_{m, k}} 
&\leq 2^d \Delta \cdot \prod_{k' = 2}^{d - k} \prn[\Big]{ \sum_{j=0}^{m - d} \abs{\lambda_{k'}}^j}. \numberthis \label{eq:adaptive_gp_bound6}  
\end{align*}
\pref{eq:adaptive_gp_bound6} can be shown by induction over $k$.  
For the base case, when $k = d - 1$, we note that for any $m > d$, 
\begin{align*}
\abs{z_{m, d - 1}} 
& = \abs{z_{m, d} + \lambda_{1} \cdot z_{m - 1, d - 1}} \\ 
& = \abs{\lambda_1} \abs{z_{m - 1, d  -1}} \\
& \leq \abs{z_{m - 1, d - 1}},
\end{align*} 
where the first line follows from the definition of $z_{m, d}$, 
and the equality in the second line holds because $z_{m, d} = 0$ 
for all $m \geq d$ (see \pref{lem:z_md_zero}). 
The inequality in the last line above is given by the fact 
that $\abs{\lambda_d} \leq 1$. Repeating the above $m - d$ times, we get that
\begin{align*} 
\abs{z_{m, d-1}} \leq \abs{z_{d, d-1}} &\leq 2^{d-1} \Delta, 
\end{align*} 
where the second inequality above follows from the bound \pref{eq:adaptive_gp_bound4}.

We next show the induction step. For any $k \leq d - 2$, suppose \pref{eq:adaptive_gp_bound6} holds for all $k' > k$ and all $m \geq d$. In the following, we will show that \pref{eq:adaptive_gp_bound6} also holds for $k$. Using the definition of $z_{m, k + 1}$ 
from \pref{eq:adaptive_gp_bound1}, we obtain: 
\begin{align*} 
	\abs{z_{m, k}} 
	& = \abs{z_{m, k + 1} + \lambda_{d - k} \cdot z_{m-1, k}} 
	\leq \abs{z_{m, k + 1}} +  \abs{ \lambda_{d - k}} \abs{z_{m-1, k}}.
\end{align*} 
Reiterating the above $m - d$ times by upper-bounding $\abs{z_{m-1, k}}$ yields: 
\begin{align*}
\abs{z_{m, k}} &\leq \sum_{j=0}^{m - d - 1} \abs{\lambda_{d  - k}}^j \abs{z_{m - j, k+1}} +  \abs{\lambda_{d - k}}^{m-d} \abs{z_{d, k}}.   
\end{align*} 
Plugging in the bound \pref{eq:adaptive_gp_bound6} for $\abs{z_{m - j, k + 1}}$ and the bound \pref{eq:adaptive_gp_bound4} for $\abs{z_{d, k}}$ in the above, we get that 
\begin{align*} 
\abs{z_{m, k}} &\leq \sum_{j=0}^{m - d - 1} \abs{\lambda_{d - k}}^j  2^d \Delta \cdot \prod_{k' = 2}^{d - k-1} \prn[\big]{ \sum_{j'=0}^{m - j - d} \abs{\lambda_{k'}}^{j'}}
 +  \abs{\lambda_{d - k}} ^{m-d} 2^d \Delta \\ 
& \leq 2^d \Delta \prn[\big]{  \sum_{j=0}^{m - d - 1} \abs{\lambda_{d - k}}^j  + \abs{\lambda_{d - k}}^{m-d}}  \cdot \prod_{k' = 2}^{d - k-1} \prn[\big]{ \sum_{j'=0}^{m - d} \abs{\lambda_{k'}}^{j'}} \\  
& = 2^d \Delta \cdot \prod_{k' = 2}^{d - k} \prn[\Big]{\sum_{j=0}^{m - d} \abs{\lambda_{k'}}^{j}}, 
\end{align*}
where the inequality in the second line follows from the fact that $\sum_{j'=0}^{m - d} \abs{\lambda_{k'}}^{j'} \geq 1$. This completes the induction step, thereby proving that \pref{eq:adaptive_gp_bound6}  
holds for all $k \leq d - 1$. The final statement follows by setting $k = 0$ in \pref{eq:adaptive_gp_bound6}. 

\paragraph{Extension to general matrices $A$.} We now prove the result for a general matrix $A$ by using the fact that matrices with distinct eigenvalues are dense in the space of $d \times d$ matrices. From \pref{thm:distinct_reduction}, we note that for every $\epsilon > 0$, there exists a matrix $B\ind{\epsilon}$ with distinct eigenvalues, denoted by $\lambda\ind{\epsilon} \in \bbC^d$, such that: 
\begin{enumerate}[label=(\alph*)]
    \item $\nrm{A^m - (B\ind{\epsilon})^m} \leq \epsilon$ for all $m \geq 1$.
    \item $\abs{\lambda\ind{\epsilon}_1} = 1$ and $\abs{\lambda_{k}\ind{\epsilon}} \leq  1$ for all $k \in [d]$. 
    \item $\nrm{B\ind{\epsilon}}_\infty \leq \nrm*{A}_\infty$. 
\end{enumerate} 

Using the above proof for the matrix $B\ind{\epsilon}$ which has distinct eigenvalues, we get that for all $m \geq d + 1$, 
\begin{align*}
\abs{u^T (B\ind{\epsilon})^m v} &\leq 2^d \prod_{k=2}^d \prn[\big]{\sum_{j=0}^{m-d} \abs{\lambda_k^\epsilon}^j } \cdot \max\crl*{ \abs{ u^T (B\ind{\epsilon}) v}, \ldots,  \abs{ u^T (B \ind{\epsilon})^d v}}.  \numberthis \label{eq:large_rank_case-1} 
\end{align*}

Furthermore, an application of \pref{thm:B_eigenvalue_bound} implies that the eigenvalues of the matrix $A$ and $B\ind{\epsilon}$ are related as: 
\begin{align*}
	\max_j \min_i ~ \abs{\lambda_i - \lambda\ind{\epsilon}_j} &\leq \prn{\nrm{A} + \nrm{B \ind{\epsilon}}}^{1 - 1/d} \nrm{A - B}^{1/d} \\ 
	&\leq \prn{d^2 \nrm{A}_\infty + d^2 \nrm{B \ind{\epsilon}}_\infty}^{1 - 1/d} \nrm{A - B}^{1/d} \\ 
	&\leq (2 d^2\nrm*{A}_\infty)^{( 1 - 1/d)} \cdot \epsilon^{1/d}, 
\end{align*} where the inequality in the second line above follows from the fact that for any matrix $B$, $\nrm*{B} \leq \nrm*{B}_{F} \leq d^2 \nrm*{B}_\infty$. The inequality in the third line above is given by the fact that $\nrm{B \ind{\epsilon}}_\infty \leq \nrm{A}_\infty$. Thus, if $\epsilon \leq  \frac{1}{2d^2 \nrm{A}_\infty} \cdot \min_{\lambda_i \neq \lambda_j} \abs{\lambda_i - \lambda_j}$, the above bound implies that the eigenvalues of $B\ind{\epsilon}$ are such that 
\begin{align*}
\abs{\lambda_k - \lambda\ind{\epsilon}_k} \leq (2 d^2\nrm*{A}_\infty)^{( 1 - 1/d)} \cdot \epsilon^{1/d},  \numberthis \label{eq:large_rank_case-2} 
\end{align*} for all $k \in [d]$.

Finally, using the fact that $\nrm{A^m - (B\ind{\epsilon})^m} \leq \epsilon$ for all $m \geq 1$ and the bound on the deviation in eigenvalues from \pref{eq:large_rank_case-2} in the relation \pref{eq:large_rank_case-1}, and taking the limit as $\epsilon$ approaches $0$, we get that, 
\begin{align*}
\abs{ u^T A^m v} &\leq 2^d \prod_{k=2}^d \prn[\big]{\sum_{j=0}^{m-d} \abs{\lambda_k}^j} \cdot \max\crl*{ \abs{ u^T A v}, \ldots,  \abs{ u^T A^d v}}.
\end{align*} 
This completes the proof of the lemma for general $d \times d$ matrices $A$. 
\end{proof} 

\begin{theorem}[Modification of Corollary 1 in \cite{hartfiel1992tracking}; Theorem 1 in \citep{hartfiel1995dense}] 
\label{thm:distinct_reduction} 
Let $A$ be a $d \times d$ matrix with eigenvalues $\lambda \in \bbC^d$ such that $\abs{\lambda_1} = 1$ and $\abs{\lambda_k} \leq 1$ for all $k \in [d]$. Then, for every $\epsilon > 0$, there exists a matrix $B\ind{\epsilon}$ such that: 
\begin{enumerate}[label=(\alph*)]
    \item $B\ind{\epsilon}$ has distinct eigenvalues. 
    \item $\nrm*{A^m - (B\ind{\epsilon})^m} \leq \epsilon$ for all $m \geq 1$. 
    \item $\abs{\lambda_1(B\ind{\epsilon})} = 1$ and  $\abs{\lambda_k(B\ind{\epsilon})} \leq 1$ for all $k \in [d]$.
    \item $\nrm{B \ind{\epsilon}}_{\infty} \leq \nrm{A}_{\infty}$.  
\end{enumerate} 
\end{theorem}

\begin{theorem}[Theorem 8.1.1. in \cite{bhatia2013matrix}]
\label{thm:B_eigenvalue_bound} 
Let $A, B$ be $d \times d$ with eigenvalues $\lambda_1, \ldots, \lambda_d$ and $\lambda'_1, \ldots, \lambda'_d$ respectively. Then, 
\begin{align*}
	\max_j \min_i ~ \abs{\lambda_i - \lambda'_j} \leq \prn{\nrm{A} + \nrm{B}}^{1 - 1/d} \nrm{A - B}^{1/d}. 
\end{align*} 
\end{theorem} 

\subsubsection{Proof of \pref{lem:adaptive_ep_bound_appx} }  \label{app:adaptive_ep_bound_proof}

We are finally ready to prove the adaptive error propagation bound given in \pref{lem:adaptive_ep_bound_appx}. 
\begin{proof}[{Proof of \pref{lem:adaptive_ep_bound_appx} } ] Using \pref{lem:recurrence_extension} for the sequences $\crl{R_{h}}$ and $\crl{\wt R_{h}}$ respectively, we get that for any $m \geq 0$, 
\begin{align*}
R_{d + m} = \tri{U, \transfer(\lambda)^m V}  
\intertext{and}
\wt R_{d + m} = \tri{\wt U, \transfer(\wh \lambda)^m \wt V}, 
\end{align*}
where the matrices $\transfer(\lambda), \transfer(\wh \lambda) \in \bbR^{d \times d}$ are defined according to \pref{def:P_beta} and the vectors $U, \wt U, V, \wt V \in \bbR^d$ are independent of $\lambda$ and $m$. Thus, for any $m \geq 0$, 
\begin{align*} 
\abs{R_{m + d} - \wt R_{m + d}} &=  \abs{\tri{U, \transfer(\lambda)^m V} - \tri{\wt U, \transfer(\wh \lambda)^m V}} \\ 
&= \abs{\tri{ \bar{U}, \bar{\transfer}^m \bar{V}}},  \numberthis \label{eq:adaptive_ep_bound3} 
\end{align*} where the vectors $\bar{U}, \bar{\beta} \in \bbR^{2d}$ and the block diagonal matrix $\bar{P} \in \bbR^{2d \times 2d}$ are defined as 
\begin{align*}
\bar{V} \ldef{} \begin{bmatrix} V \\ - \wt V \end{bmatrix}, \quad \bar{U} \ldef{} \begin{bmatrix} U  \\ - \wt U \end{bmatrix} \quad \text{and} \quad \bar{P} \ldef{} \begin{bmatrix} \transfer(\lambda) & 0 \\ 
0 & \transfer(\wh \lambda) 
\end{bmatrix}. 
\end{align*} 

An application of \pref{lem:P_eigenvalues} implies that the eigenvalues of the matrix $P(\lambda)$ and the matrix $P(\wh \lambda)$ are given by $\lambda$ and $\wh \lambda$ respectively. Since the matrix $\bar{P}$ is block-diagonal, we note that the set of eigenvalues of the matrix $\bar{P}$ is given by $\bar{\lambda} = \prn{\lambda_1, \wh \lambda_1, \ldots, \lambda_d, \wh \lambda_d}$. Note that the vector $\bar{\lambda}$ is not sorted except for the first two coordinates, however $\abs{\bar{\lambda}_{k}} \leq 1$ for all $k \in [2d]$.  Using \pref{lem:adaptive_power_bound_general} for the $2d \times 2d$ matrix $\bar{P}$ and the vectors $\bar{U}$ and $\bar{\beta}$, we get that for any $m \geq 2d + 1$, 
\begin{align*}
\abs{\tri{ \bar{U}, \bar{P}^{m} \bar{V}}} &\leq 2^{2d} \cdot \prod_{k=2}^{2d} \prn[\big]{ \sum_{j=0}^{m - 2d} \abs{\bar{\lambda}_k}^j} \cdot \max\crl{ \abs{\tri{ \bar{U}, \bar{P} \bar{V}}}, \ldots,  \abs{\tri{ \bar{U}, \bar{P}^{2d} \bar{V}}}} \\
&\leq 2^{2d}  \cdot m \cdot \prod_{k=2}^{d} \prn[\big]{ \sum_{j=0}^{m-1} \abs{\lambda_k}^j } \cdot  \prod_{k=2}^{d} \prn[\big]{ \sum_{j=0}^{m-1} \abs{\wh \lambda_k}^j }  \cdot \max_{m' \leq 2d} ~ \abs{\tri{ \bar{U}, \bar{P}^{m'} \bar{V}}},  \numberthis \label{eq:adaptive_ep_bound4} 
\end{align*}
where the inequality in the last line uses the fact that $\abs{\lambda_k} \leq 1$ and $\abs{\wh \lambda_k} \leq 1$ for all $k \in [d]$, and from thus $\sum_{j=0}^{m-1} \abs{\wh \lambda_k}^j \leq m$. Using the bound \pref{eq:adaptive_ep_bound4} in the relation \pref{eq:adaptive_ep_bound3}, we get that for any $h \geq 3d + 1$, 
\begin{align*}
	\abs{R_{h} - \wt R_{h}} &\leq \abs{\tri{ \bar{U}, \bar{\transfer}^{h - d} V }} \\
	&\leq  2^{2d}  h  \prod_{k=2}^{d} \prn[\big]{ \sum_{j=0}^{h-1} \abs{\lambda_k}^j } \cdot  \prod_{k=2}^{d} \prn[\big]{ \sum_{j=0}^{h-1} \abs{\wh \lambda_k}^j }  \cdot \max_{m' \leq 2d} ~ \abs{\tri{ \bar{U}, \bar{P}^{m'} \bar{V}}} \\
	 	&\overeq{\proman{1}} 2^{2d}  h \cdot \prod_{k=2}^{d} \prn[\big]{ \sum_{j=0}^{h-1} \abs{\lambda_k}^j } \cdot  \prod_{k=2}^{d} \prn[\big]{ \sum_{j=0}^{h-1} \abs{\wh \lambda_k}^j }   \cdot \max_{m' \leq 2d} ~ \abs{R_{d + m'} - \wt R_{d + m'}} \\
		&\leq 2^{2d}  h \cdot  \prod_{k=2}^{d} \prn[\big]{ \sum_{j=0}^{h-1} \abs{\lambda_k}^j } \cdot  \prod_{k=2}^{d} \prn[\big]{ \sum_{j=0}^{h-1} \abs{\wh \lambda_k}^j }   \cdot \max_{h' \leq 3d} ~ \abs{R_{h'} - \wt R_{h'}}, 
\end{align*} where the equality $\proman{1}$ follows due to relation \pref{eq:adaptive_ep_bound3}. 
\end{proof}

\subsection{Proof of \pref{thm:adaptive_upper_bound_main}} \label{app:adaptive_upper_bound_main} 
Before delving into the proof of \pref{thm:adaptive_upper_bound_main}, we first note the following technical lemma which concerns with the feasability and properties of the solutions of optimization problem \pref{eq:adaptive_optimization_problem} in \pref{alg:adaptive_extrapolation}. 

\begin{lemma} 
\label{lem:adaptive_initial_conditions}  
Let $\lambda \in \bbC^d$ such that $\abs{\lambda_k} \leq 1$ for all $k \in [d]$. Using the initial values $R_1, \ldots, R_{d}$, let $R_{h}$ be defined as 
\begin{align*} 
R_{h} = \sum_{k=1}^d (-1)^{k+1} \alpha_k(\lambda \ind{\pi}) R_{h-k}.  \numberthis \label{eq:initial_error_boundad2}    
\end{align*} 
Further, let  $\wh R_1, \wh R_2, \ldots, \wh R_{3d}$  denote the estimates for $R_1, \ldots, R_{3d}$ respectively, such that 
\begin{align*}
\max_{h \leq 3d} ~ \abs{\wh R_{h} - R_{h}} \leq \eta,  \numberthis \label{eq:initial_error_boundad1}    
\end{align*}
where $\eta \ldef{} \min\crl[\Big]{\sqrt{\frac{8 K^{3d} \log(6d \abs*{\Pi} / \delta )}{n}}, {\frac{4 K^{3d} \log(6d \abs*{\Pi} / \delta )}{n}}}$. Then, 
\begin{enumerate}[label=(\alph*), leftmargin=8mm]
\item The optimization problem \pref{eq:adaptive_optimization_problem} in  \pref{alg:adaptive_extrapolation} has a solution $\wh \lambda \in \bbC^d$ such that $\abs{\wh \lambda_1} = 1$ and 
\begin{align*}
\prod_{k=2}^d  \prn[\big]{\sum_{j=0}^{H-1} \abs{\wh \lambda_k}^j}  &\leq \prod_{k=2}^d  \prn[\big]{\sum_{j=0}^{H-1} \abs{\lambda_k}^j}. 
\end{align*} 
\item Further, let $\wt R_{h}$ be predictions according to \pref{line:adaptive_value_prediction} in \pref{alg:adaptive_extrapolation} using the solution $\wh \lambda$. Then, 
\begin{align*}
\max_{h' \leq 3d} ~ \abs{\wt R_h - R_h} \leq 2d \cdot \prn{64 e}^{d} \cdot \eta.
\end{align*}
\end{enumerate} 
\end{lemma} 
\begin{proof} In the following, we provide the proof for part-(a) of the lemma. The proof of part-(b) follows exactly along the lines of a similar statement proven in \pref{lem:initial_conditions}.

\paragraph{Proof of part-(a).}  We prove this by showing that the vector $\lambda \in \bbR^d$ satisfies all the constraints of the optimization problem in \pref{eq:adaptive_optimization_problem}. First note that $\abs{\lambda_1} = 1$ and $\abs{\lambda_k} \leq  1$ for all $k \leq d$, by definition. Furthermore, for any $h \leq 3d$, 
 \begin{align*}
\abs{\sum_{k=1}^d (-1)^{k+1} \alpha_{k}(\lambda) \wh R_{h - k} - \wh R_{h}} &\overeq{\proman{1}}   \abs{\sum_{k=1}^d (-1)^{k+1} \alpha_{k}(\lambda) \prn{\wh R_{h - k} - R_{h-k}} - (\wh R_{h} - R_h)} \\ 
&\overleq{\proman{2}}  \sum_{k=1}^d  \abs{\alpha_{k}(\lambda)} \cdot \abs{\wh R_{h - k} - R_{h-k}} +  \abs{\wh R_{h} - R_h}  \\
&\overleq{\proman{3}} d \cdot 4^d \cdot \eta + \eta \\
&\leq 2d \cdot 4^d \cdot \eta. 
\end{align*}

where the equality $\proman{1}$ follows from the relation \pref{eq:initial_error_bound1} and the inequality $\proman{2}$ follows from Triangle inequality. The inequality $\proman{3}$ follows by plugging in the bound from \pref{lem:alpha_independent_bound} for $\abs{\alpha_k(\lambda)}$ and using the bound in \pref{eq:initial_error_boundad1}. Plugging in the value of $\eta = \min\crl[\Big]{\sqrt{\frac{8 K^{3d} \log(6d \abs*{\Pi} / \delta )}{n}}, {\frac{4 K^{3d} \log(6d \abs*{\Pi} / \delta )}{n}}}$ in the above bound, we get that 

 \begin{align*} 
\abs{\sum_{k=1}^d (-1)^{k+1} \alpha_{k}(\lambda) \wh R_{h - k} - \wh R_{h}} &\leq 2d \cdot 4^d \cdot \min\crl[\Big]{\sqrt{\frac{8 K^{3d} \log(6d \abs*{\Pi} / \delta )}{n}}, {\frac{4 K^{3d} \log(6d \abs*{\Pi} / \delta )}{n}}}. \numberthis \label{eq:initial_error_boundad3}   
\end{align*}

Thus, the vector $\lambda \in \bbC$ is a feasible solution to the optimization problem in \pref{eq:adaptive_optimization_problem}. Next, noting the fact that \pref{eq:adaptive_optimization_problem} is a minimization problem, we get that for the returned solution $\wh \lambda$ must satisfy  
\begin{align*}
\prod_{k=2}^d  \prn[\big]{\sum_{j=0}^{H-1} \abs{\wh \lambda_k}^j}  &\leq \prod_{k=2}^d  \prn[\big]{\sum_{j=0}^{H-1} \abs{\lambda_k}^j}. 
\end{align*} 
\end{proof}

We are now ready to prove our adaptive upper bound in \pref{thm:adaptive_upper_bound_main}. The proof is very similar to the proof of \pref{thm:basic_sample_complexity_bound} given in \pref{app:basic_sample_complexity_bound}. The main technical difference is that we use an adaptive error propagation bound, given in \pref{lem:adaptive_ep_bound_appx}, instead of the error propagation bound from \pref{lem:basic_ep_bound} to control the  error in the predicted rewards. 

\begin{proof}[{Proof of \pref{thm:adaptive_upper_bound_main}} ]
Starting from \pref{lem:ss_ub}, we get that with probability at least  $1 - \delta$, for every policy $\pi \in \Pi$, our estimate $\wh R\ind{\pi}_h$ computed in  \pref{line:adaptive_reward_estimation1} of \pref{alg:adaptive_extrapolation} satisfies the error bound 
\begin{align*} 
\max_{h' \leq 3d} ~ \abs{\wh R\ind{\pi}_{h'} - R\ind{\pi}_{h'}} &\leq \min\crl[\Bigg]{\sqrt{\frac{8 K^{3d} \log(6d \abs*{\Pi} / \delta )}{n}}, {\frac{4 K^{3d} \log(6d \abs*{\Pi} / \delta )}{n}}}.  \numberthis  \label{eq:basic_autoregression4_adaptive} 
\end{align*}

Now, consider any policy $\pi \in \Pi$, and let $\wh \lambda\ind{\pi}$, $\wh \Delta\ind{\pi}$, $\wt R\ind{\pi}_h$ and $\wt V \ind{\pi}$ denote the corresponding local variables in the $\adavalestimate$ when invoked in \pref{alg:adaptive_extrapolation} for the policy $\pi$.  Further, let $\lambda\ind{\pi}$ denote the eigenvalues of the transition matrix $T\ind{\pi}$. As a consequence of \pref{lem:basic_recurrence_relation}, the expected rewards $R_h\ind{\pi}$ satisfy an autoregression where the coefficients are determined by $\lambda\ind{\pi}$. Specifically, for any $h \geq d +1$, 
\begin{align*} 
R\ind{\pi}_{h} = \sum_{k=1}^d (-1)^{k+1} \alpha_k(\lambda \ind{\pi}) \cdot R\ind{\pi}_{h-k}.  
\end{align*} 
Furthermore, by definition (see \pref{line:adaptive_reward_prediction} of \pref{alg:adaptive_extrapolation}), the predicted rewards $\wt R\ind{\pi}_h$ also satisfy a similar autoregression where the coefficients are determined by $\wh \lambda \ind{\pi}$, the solution of the optimization problem in \pref{eq:adaptive_optimization_problem} for the policy $\pi$. We have, for any $h \geq d + 1$, 
\begin{align*}
 	\wt R_{h}\ind{\pi} =   \sum_{k=1}^d (-1)^{k+1} \alpha_{k} (\wh \lambda \ind{\pi}) \cdot \wt R_{h - k}\ind{\pi} 
\end{align*}
where $\wt R_{h'} \ldef{} \wh R_{h'}$ for $h' \leq d$. Additionally, also note that $T\ind{\pi}$  is a stochastic matrix and thus $\abs{\lambda_k\ind{\pi}} \leq 1$ for all $k \in [d]$. By definition, we also have that $\abs{\wh \lambda_k\ind{\pi}} \leq 1$. Thus, using the error propagation bound in \pref{lem:adaptive_ep_bound_appx} for the sequences $\crl{R_{h}\ind{\pi}}$ and $\crl{\wt R\ind{\pi}_h}$, we get that for any  $h \geq 1$, 
\begin{align*} 
	\abs{\wt R\ind{\pi}_h - R\ind{\pi}_h} &\leq 4^{d}  h  \cdot \prod_{k=2}^{d} \prn[\Big]{\sum_{j=0}^{h-1} \abs{\lambda\ind{\pi}_k}^{j}} \cdot \prod_{k=2}^{d} \prn[\Big]{\sum_{j=0}^{h-1} \abs{\wh \lambda\ind{\pi}_k}^{j}} \cdot \max_{h' \leq 3d} ~ \abs{\wt R_{h'} - R_{h'}} \\ 
	&\leq 4^{d}  h  \cdot \prod_{k=2}^{d} \prn[\Big]{\sum_{j=0}^{h-1} \abs{\lambda\ind{\pi}_k}^{j}}^2 \cdot \max_{h' \leq 3d} ~ \abs{\wt R_{h'} - R_{h'}},   \numberthis \label{eq:general_sc1_adaptive}
\end{align*} where the inequality in the second line above follows from the fact that 
\begin{align*}
\prod_{k=2}^d  \prn[\big]{\sum_{j=0}^{H-1} \abs{\wh \lambda\ind{\pi}_k}^j}  &\leq \prod_{k=2}^d  \prn[\big]{\sum_{j=0}^{H-1} \abs{\lambda\ind{\pi}_k}^j}  
\end{align*} as a consequence of \pref{lem:adaptive_initial_conditions}-(a) for the policy $\pi$. Next, \pref{lem:adaptive_initial_conditions}-(b) for the policy $\pi$ implies that the predicted rewards $\wt R_{h'}\ind{\pi}$ satisfy the error bound 
\begin{align*}
\qquad  \max_{h' \leq 3d} ~ \abs{\wt R_{h'}\ind{\pi} - R_{h'}\ind{\pi}} &\leq 2d \cdot (64 e)^d \cdot \max_{h' \leq 3d} ~ \abs{\wh R\ind{\pi}_{h'} - R\ind{\pi}_{h'}} \\
&\leq  2d \cdot (64 e)^d \cdot \eta,   
\end{align*}
where $\eta$ denotes the right hand side of \pref{eq:basic_autoregression4_adaptive}. 
Plugging the above in \pref{eq:general_sc1_adaptive}, we get that 
\begin{align*} 
	\abs{\wt R\ind{\pi}_h - R\ind{\pi}_h} &\leq 2d h (256 e)^d \cdot \prod_{k=2}^{d} \prn[\Big]{\sum_{j=0}^{h-1} \abs{\lambda\ind{\pi}_k}^{j}}^2 
	\cdot \eta. \numberthis \label{eq:general_sc5_adaptive}
\end{align*} for any $h \geq 1$. Thus, the error in the estimated value $\wt V\ind{\pi}$ for the policy $\pi$ is bounded by 
\begin{align*}
\abs{\wt V\ind{\pi} - V \ind{\pi}} 
&= \abs{\sum_{h=1}^{H} \prn{\wt R_h\ind{\pi} - R_h \ind{\pi}} }   \\
&\leq \sum_{h = 1}^H \abs{{\wt R_h\ind{\pi} - R_h \ind{\pi}}} \\
&\leq \sum_{h=1}^H 2d h (256 e)^d \cdot \prod_{k=2}^{d} \prn[\Big]{\sum_{j=0}^{h-1} \abs{\lambda\ind{\pi}_k}^{j}}^2 
	\cdot \eta \\ 
&\leq 2d H^2 (256 e)^d \cdot \prod_{k=2}^{d} \prn[\Big]{\sum_{j=0}^{H-1} \abs{\lambda\ind{\pi}_k}^{j}}^2 
	\cdot \eta \\ 
	&\leq 2d H^2 (256 e)^d \cdot \max_{\pi' \in \Pi} \prod_{k=2}^{d}  \prn[\Big]{\sum_{j=0}^{H-1} \abs{\lambda\ind{\pi'}_k}^{j}}^2 
	\cdot \eta,  \numberthis \label{eq:general_sc4_adaptive} 
\end{align*} where the inequality in the second last line follows by using the bound in \pref{eq:general_sc5_adaptive}. 

Since $\pi$ is arbitrary in the above chain of arguments, the error bound in \pref{eq:general_sc4_adaptive} holds for all policies $\pi \in \Pi$. Thus, for any $\pi \in \Pi$, the policy $\wt \pi$ returned in \pref{line:ada_choose_optimal_policy} of \pref{alg:adaptive_policy_search} satisfies 
\begin{align*}
V\ind{\wt \pi} - V\ind{\pi} &= \prn{\wt V\ind{\pi} - V\ind{\pi}}  +  \prn{\wt V\ind{\wt \pi} - \wt V \ind{\pi}} + \prn{V\ind{\wt \pi} - \wt V\ind{\wt \pi}} \\
&\geq \prn{\wt V\ind{\pi} - V\ind{\pi}} + \prn{V\ind{\wt \pi} - \wt V\ind{\wt \pi}} \\
&\geq - \abs{\wt V\ind{\pi} - V\ind{\pi}} - \abs{V\ind{\wt \pi} - \wt V\ind{\wt \pi}}, 
\end{align*} where the inequality in the second line follows from the fact that $\wt V\ind{\wt \pi} \geq \wt V\ind{\pi}$ for every $\pi \in \Pi$ by the definition of the policy $\wt \pi$. Using the bound from \pref{eq:general_sc4_adaptive} for policies $\pi$ and $\wt \pi \in \Pi$ in the above, we get that  
\begin{align*} 
V\ind{\wt \pi} &\geq V\ind{\pi} - 4 d H^2 (256 e)^d \cdot \max_{\pi' \in \Pi} \prod_{k=2}^{d}  \prn[\Big]{\sum_{j=0}^{H-1} \abs{\lambda\ind{\pi'}_k}^{j}}^2
	\cdot \eta \\ 
&\geq V\ind{\pi} -  4 d H^2 (256 e)^d \cdot \max_{\pi' \in \Pi} \prod_{k=2}^{d}  \prn[\Big]{\sum_{j=0}^{H-1} \abs{\lambda\ind{\pi'}_k}^{j}}^2 \cdot  \min\crl[\Bigg]{\sqrt{\frac{8 K^{3d} \log(6d \abs*{\Pi} / \delta )}{n}}, {\frac{4 K^{3d} \log(6d \abs*{\Pi} / \delta )}{n}}} \\ 
&\geq V\ind{\pi} -  4 d H^2 (256 e)^d \cdot \max_{\pi' \in \Pi} \prod_{k=2}^{d}  \prn[\Big]{\sum_{j=0}^{H-1} \abs{\lambda\ind{\pi'}_k}^{j}}^2  \sqrt{\frac{8 K^{3d} \log(6 d \abs*{\Pi}/ \delta)}{n}} 
\end{align*} 
where the inequality in the second line above follows by plugging in the value of $\eta$ as the right hand side of \pref{eq:basic_autoregression4_adaptive}, and the inequality in the last line holds due to the fact that $-\min\crl{a, b} \geq -a$ for any $a, b \geq 0$. 

Since the above holds for any $\pi \in \Pi$, we have that 
\begin{align*} 
V\ind{\wt \pi} &\geq \max_{\pi \in \Pi} V\ind{\pi} - 4 d H^2 (256 e)^d \cdot \max_{\pi' \in \Pi} \prod_{k=2}^{d}  \prn[\Big]{\sum_{j=0}^{H-1} \abs{\lambda\ind{\pi'}_k}^{j}}^2  \sqrt{\frac{8 K^{3d} \log(6 d \abs*{\Pi}/ \delta)}{n}}, 
\end{align*} 
hence proving the desired statement. 
\end{proof}

\subsection{Adaptivity to rank}  \label{app:rank_adaptivity}
We now describe how the learner can find the best policy in the class $\Pi$, that satisfies \pref{ass:low_rank}, without knowing the value of the rank parameter. Let us denote the unknown rank parameter by $d^*$. Our adaptive algorithm, given in \pref{alg:adaptive_policy_search_rank}, follows from standard techniques in the model selection literature. For every $d \in [H]$, we compute an optimal policy $\widetilde{\pi}_{d}$ assuming that the rank $d^* = d$. Then, for each $d \in [H]$, we estimate the value function for the policy $\widetilde{\pi}_d$ by drawing $n/2H$ fresh trajectories using that policy. Finally, we return the policy $\wt \pi$ from the set $\crl*{\widetilde{\pi}_d}_{d  \in [H]}$ with the highest estimated value. The returned policy $\wt \pi$ satisfies, with probability at least $1- \delta$, 
\begin{align*} 
V \ind{\wt \pi} \geq \max_{\pi \in \Pi} ~ V\ind{\pi}   - O\prn[\Big]{\prn[\Big]{\frac{H}{d^*}}^{2d^*} \sqrt{ \frac{(8K)^{3d^*} \log(6d \abs{\Pi} / \delta )}{n}} - 2 \sqrt{\frac{\log(H) \log(1/\delta)}{n}}}.  \numberthis \label{eq:rank_adaptivity} 
\end{align*} 

\begin{algorithm}[h] 
\caption{Adaptive policy search algorithm (adaptivity to rank)}   
\begin{algorithmic}[1] 
	\Require horizon $H$, rank $d$,  number of episodes $n$, finite policy class $\Pi$ 
	\State Collect a dataset $\cD = \crl{\prn{x\ind{t}_h, a\ind{t}_h, r\ind{t}_h}_{h=1}^H}_{t = 1}^{n/2}$ by sampling $n/2$ trajectories where actions are sampled from $\uniform(\cA)$. 
	\For {$d \in \crl*{1, 2, \ldots, H}$}
	\For {\text{policy} $\pi \in \Pi$} 
\State Estimate $\wt V_{d}\ind{\pi}$ by calling \valestimate($H, d, \cD, \pi$).
\EndFor  
    \State Compute the policy $\widetilde{\pi}_{d} \in \argmax_{\pi \in \Pi} \wt V_{d}\ind{\pi}$. 
    \State Collect $n/2H$ more episodes using the policy $\widetilde{\pi}_{d}$ and estimate the value  $\bar{V}\ind{\widetilde{\pi}_d}$ using the empirical average of the returned rewards. 
\EndFor 
\State\label{line:ada_choose_optimal_policy} \textbf{Return: } policy $\wt \pi$ with best estimated value $\wt \pi \in \argmax_{d \in [H]} \bar{V} \ind{\wt \pi_d}$. 
\end{algorithmic} 
\label{alg:adaptive_policy_search_rank}   
\end{algorithm}

Note that, in \pref{alg:adaptive_policy_search_rank}, we cap the value of $d^*$ by $H$. In the case, when $d^* > H$, we can directly estimate the expected reward for each policy by importance sampling upto $H$ steps, and thus compute the optimal policy in $\Pi$. 

Finally, we can get an algorithm that adapts to both the unknown rank $d^*$ and the eigenspectrum simultaneously by using the procedure \adavalestimate~(given in \pref{alg:adaptive_extrapolation}) instead of the procedure \valestimate~in \pref{alg:adaptive_policy_search_rank}. This implies the following adaptive bound for well mixing MDPs.  
\begin{corollary}[Well mixing MDP] Given $\delta \in (0, 1)$, horizon $H$, a policy class $\Pi$ and a MDP M. 
\begin{enumerate}[label=(\alph*)]
    \item If for every policy $\pi \in \Pi$, the transition matrix $T\ind{\pi}$ has at most $d^*$ non-zero eigenvalues such that the second largest eigenvalue $\abs{\lambda\ind{\pi}_2} \leq 1 - \gamma~$, where $K$ and $\gamma$ are not known to the learner. Then, \pref{alg:adaptive_policy_search_rank} (run win $\adavalestimate$ instead of $\valestimate$) returns a policy $\wt \pi$ such that, with probability at least $ 1- \delta$, 
\begin{align*} 
V\ind{\wt \pi} \geq \max_{\pi \in \Pi} V \ind{\pi} - \widetilde{O}\prn[\Big]{\prn[\Big]{\frac{K}{\gamma }}^{2d^*} \frac{1}{\sqrt{n}}}.  \end{align*} 
    \item If for every policy $\pi \in \Pi$, the mixing time of the transition matrix $T\ind{\pi}$ is bounded by $\tau$, where $\tau$ is not known to the learner. Then, \pref{alg:adaptive_policy_search_rank} (run win $\adavalestimate$ instead of $\valestimate$) returns a policy $\wt \pi$ such that, with probability at least $ 1- \delta$,
\begin{align*} 
V\ind{\wt \pi} \geq \max_{\pi \in \Pi} V \ind{\pi} - \widetilde{O}\prn[\Big]{\frac{K^{2 \tau}}{\sqrt{n}}}. \end{align*} 
\end{enumerate}  
\end{corollary}

The exponential dependence in the mixing time in the above performance guaranee is unavoidable without further assumptions as illustrated by our lower bounds construction in \pref{sec:lower_bounds}. 
\clearpage 
\section{Lower bounds} \label{app:lower_bounds} 

\subsection{Lower bound construction}
\label{app:lower_bound_construction}
We start by describing the lower bound construction, consisting of the policy class $\Pi$ and the family $\mdpfamily$ of Markov decision processes with rank $2d+2$. All MDPs in the family $\mdpfamily$ have the observation space $\cX$ of finite (but very large) size $N = |\cX|$ and action space $\cA = \{0, 1\}$, but have different transition dynamics. 

\paragraph{Policy class $\Pi$.} The policy class $\Pi \subset \crl*{\cX \mapsto \cA}$ consists of $K = \prn*{H/d}^d$ deterministic policies that are sufficiently distinct from each other. Specifically, for any two distinct policies $\prn*{\pi, \pi'} \in \Pi^2$, 
\begin{align*}
\sum_{x \in \cX} \indicator{ \pi(x) \neq \pi'(x)} \geq \frac{N}{4}. 
\end{align*}
Existence of such a policy class follows from the Gilbert-Varshamov bound  (\pref{lem:gilbert_varshamov}) when $8 d \log \prn*{H/d} \leq~N$. 

\paragraph{Family of MDPs $\mdpfamily$.} 
Each MDP $M_{\pi, \obsmap} \in \mdpfamily$ is indexed by a policy $\pi \in \Pi$ (which will be optimal in that MDP) and a function $\obsmap \colon \cX \mapsto \cS$ that maps each observation in $x \in \cX$ to one of the $2d+2$ hidden states given by $S = \crl*{(1, g), (1, b), \ldots, (d, g), (d, b), +, -}$.  In the following, we describe the transition dynamics and the reward function for the MDP $M_{\pi^*, \phi}$. 

\paragraph{Transition dynamics of $M_{\pi, \phi}$.} The transition dynamics of $M_{\pi, \phi}$ is governed by the mapping $\phi$ and the dynamics in the $2d + 2$ latent states. The dynamics in the latent states $S$ is given by two parallel chains, depicted in \pref{fig:lower_bound_structure}.
Each latent state, except for the final states $+$ and $-$, have the form $(i, g)$ or $(i, b)$ where $i \in [d]$ denotes the index in the chain, and the notation $g$ and $b$ denotes good states and bad states respectively. The initial observation $x_0$ always corresponds to the hidden state $(1, g)$. At each time step, independent of the action taken, the chain index $i$ increases by $1$ with probability $p_i$ (defined later) or stays the same with probability $(1 - p_i)$. As long as the agent follows actions according to $\pi$, the next latent state remains a good state (with the second component $g$). However, as soon as the agent takes an action that $\pi$ would not have taken, the second component is set to $b$ and then stays $b$ forever. If the agent reaches latent state $(d, g)$ it transitions to the latent state $+$ with probability $\frac{1}{2} + \epsilon$ and to the latent state $-$ with probability $\frac{1}{2} - \epsilon$. From $(d, b)$, the agent transitions to both the latent states $+$ or $-$ with equal probability. Finally, from the hidden state $+$, the agent transitions to $-$ in the next step with probability $1$. The state $-$ always transitions back to itself independent of the action taken. 

We next describe, how the above dynamics in the latent state space defines the transition dynamics for the MDP $M_{\pi, \obsmap}$ in the observation space. Define $\cX_s \ldef{} \{ x \in \cX  ~|~ \phi(x) = s\}$ as the set of observations from $\cX$ that are mapped to latent state $s \in \cS$ by the feature mapping $\phi$, and define $D_s \ldef{} \uniform\prn*{\cX_s}$ to denote the uniform distribution over the set $\cX_s$. The initial observation $x_0$ is sampled independently from $\mu_0 = D_{(1, g)}$. The two parameters $\pi$ and $\obsmap$ of  MDP $M_{\pi, \obsmap}$ define the transition dynamics $T_{\pi, \obsmap}$ as follows:

\begin{enumerate}[label=\alph*)]
\item For any observations $x \in \cX_{(i, g)}$ of \textbf{good latent states}, where $i \in [d-1]$,
\begin{align*}
    T_{\pi, \obsmap}(x, a) &= \begin{cases}
    p_i D_{(i+1, g)} + (1 - p_i)D_{(i,g)}  &\textrm{if }  a = \pi(x)\\
    p_i D_{(i+1, b)} + (1 - p_i)D_{(i,b)} &\textrm{else } \\
    \end{cases}, 
\end{align*}
where the value of $p_i \in (0, 1)$ is set later and $T_{\pi, \obsmap}(x, a)$ denotes the probability distribution over the next observation $x'$ when taking action $a$ at observation $x$.
\item For any observations $x \in \cX_{(i, b)}$ of \textbf{bad latent states}, where $i \in [d-1]$ and all $a \in \cA$, 
\begin{align*}
    T_{\pi, \obsmap}(x, a) &= p_i D_{(i+1, b)} + (1 - p_i)D_{(i,b)}. 
\end{align*}
\item For any observations $x \in \cX_{(d, g)}$ of the \textbf{good goal state} and all $a \in \cA$, 
\begin{align*}
    T_{\pi, \obsmap}(x, a) &=  \prn*{\frac{1}{2} + \epsilon} D_{+} +
    \prn*{\frac{1}{2} - \epsilon} D_{-}, 
\end{align*}
where the bias $\epsilon \in (0, 1/2)$ is set later. 
\item For any observation $x \in \cX_{(d, b)}$ of the \textbf{bad goal state}, and all $a \in \cA$,
\begin{align*}
    T_{\pi, \obsmap}(x, a) &=  \frac{1}{2} D_{+} +
    \frac{1}{2} D_{-}. 
\end{align*}

\item For any observation $x \in \cX_- \cup \cX_+$ of  \textbf{latent states $-$ and $+$}, and all $a \in \cA$, 
\begin{align*}
    T_{\pi, \obsmap}(x, a) &= D_- .
\end{align*}
\end{enumerate}

\paragraph{Reward function for $M_{\phi, \pi}$.} For any observation $x$, the reward is $0$ unless the latent state correspond to $x$ is $+$, in which case, the reward is $1$. Specifically, 
\begin{align*}
    r(x, a) = \indicator{x \in \cX_{+}}~.
\end{align*}

\begin{figure}
    \centering
\includegraphics[width=0.8\linewidth]{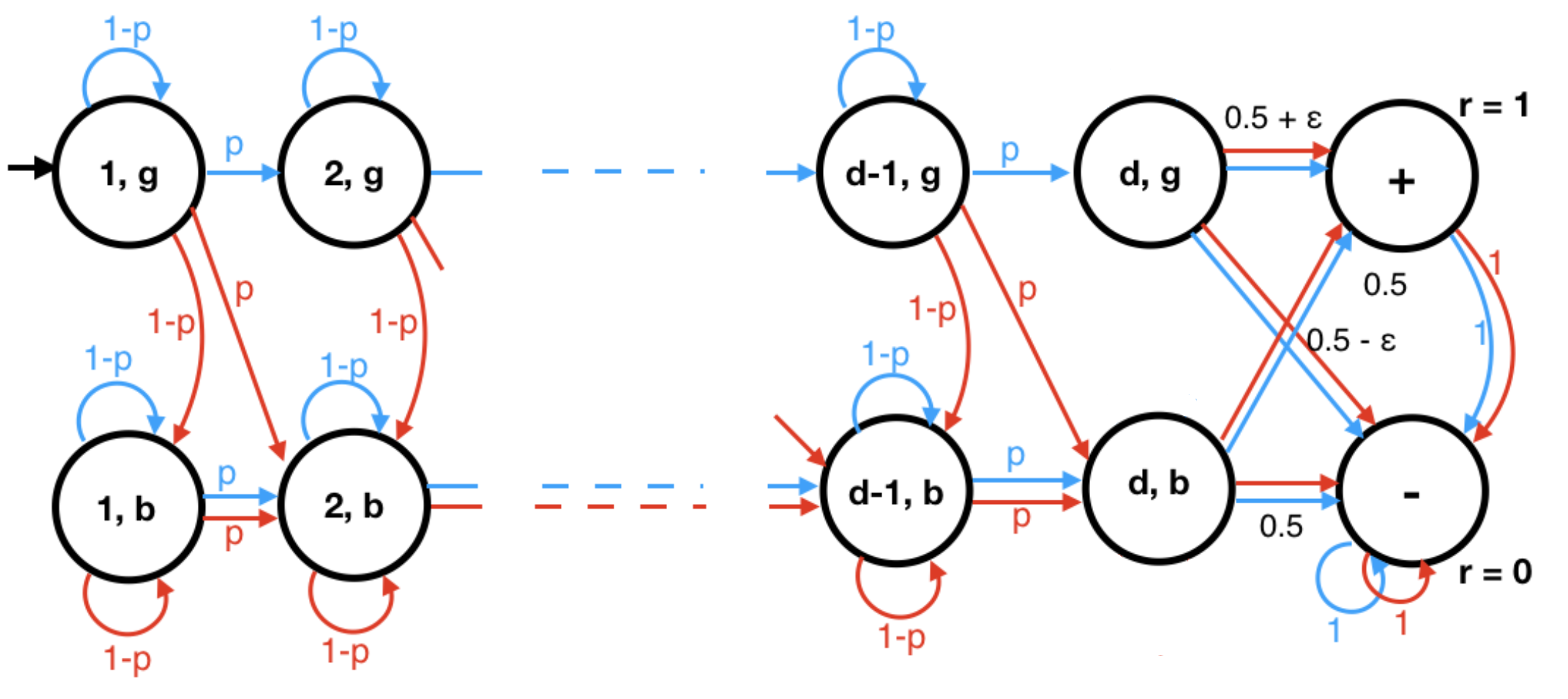}
\caption{Latent state construction: contextual combination lock. As long as the agent follows actions of the policy $\pi$ that characterizes the MDP $M_{\phi, \pi}$ (blue arrows), the agent remains in good states $(i, g)$ and receives a Bernoulli($1/2 + \epsilon$) reward but otherwise transits to bad states $(i,b)$ and receives a Bernoulli($1/2$) reward. 
    } 
    \label{fig:lower_bound_structure}
\end{figure}

\paragraph{Initial observation in $M_{\phi, \pi}$.} The initial observation $x_0$ is sampled uniformly at random from $\cX_{(1, g)}$. 

\paragraph{Additional MDP $M_{0, \obsmap}$.} In addition to the above defined MDPs $M_{\pi, \obsmap}$, we define the MDP $M_{0, \obsmap}$ for every $\obsmap$ in the MDP where latent states with $b$ and $g$ behave exactly the same. Specifically, the transition dynamics is given by 
\begin{align*}
    T_{0, \obsmap}(x, a) = \begin{cases}
    \frac{1}{2} D_{+} +
    \frac{1}{2} D_{-} & \textrm{if } x \in X_{(d, b)} \cup X_{(d, g)} \\
    \frac{p}{2} D_{(i+1, b)} + \frac{1 - p}{2} D_{(i,b)} +
    \frac{p}{2} D_{(i+1, g)} + \frac{1 - p}{2} D_{(i,g)}
    &\textrm{if } x \in X_{(i, b)} \cup X_{(i, g)}  \textrm{ for } i \in [d-1]\\
    D_{-}  & \textrm{if } x \in X_{+} \cup X_{-}
    \end{cases}. 
\end{align*}
Note that the actions taken do not affect the rewards or observations received in MDPs $M_{0, \phi}$ and thus every policy is an optimal policy. 

The family of MDPs $\mdpfamily$ is finally defined as 
\begin{align*}
    \mdpfamily \ldef{} \{M_{\pi, \obsmap} \mid \pi \in \Pi, ~ \obsmap \in \cX \mapsto \cS\} \cup
    \{M_{0, \obsmap} \mid{} \obsmap \in \cX \mapsto \cS\}.
\end{align*}
We note that the rank of each MDP in the class $\mdpfamily$ is $O(d)$ as show in the following lemma. 

\begin{lemma}[Rank bound for MDPs in $\mdpfamily$]  \label{lem:mdp_lb_rank}
Let $M_{\pi^*, \phi}$ be an MDP in $\mdpfamily$. Let $\pi \in \Pi$ be any policy and let $T^{\pi}_{\pi^*, \phi}$ denote the induced transition matrix of the policy $\pi$ in the MDP $M_{\pi^*, \phi}$. Then, the rank of the matrix $T^{\pi}_{\pi^*, \phi}$ is bounded as 
\begin{align*}
    \rank(T^\pi_{\pi^*, \obsmap}) \leq 2d - 1 
    \qquad 
    \textrm{and}
    \qquad \rank(T^\pi_{0, \obsmap}) \leq 2d - 1~.
\end{align*}
Further, the non-zero eigenvalues of $T^\pi_{\pi^*, \obsmap}$ and $T^\pi_{0, \obsmap}$
are given by \begin{align*}
    \begin{cases}
    1 - p_i & \textrm{ for } s = (i, b) \textrm{ where } i \in [d-1]\\ 
    (1 - p_i) \Pr_{x \sim \textrm{Unif}(\cX_s)}(\pi^*(x) = \pi(x) )
    & \textrm{ for } s = (i, g) \textrm{ where } i \in [d-1]\\
    1 & \textrm{ for } s = -.
    \end{cases}
\end{align*}
    and

\begin{align*}
    \begin{cases}
    \frac{1 - p_i}{2} & \textrm{ for } s =  (i, g) \textrm{ where } i \in [d-1]\\ 
        \frac{1 - p_i}{2} & \textrm{ for } s = (i, b)  \textrm{ where } i \in [d-1]\\ 
    0 & \textrm{ for } s \in \{+, (d, g), (d, b))\}\\
    1 & \textrm{ for } s = -.
    \end{cases}
\end{align*}
respectively.
\end{lemma}
\begin{proof} 
We can write the transition probability from observation $x$ to $x'$ as
\begin{align*}
    T^{\pi}_{\pi^*, \obsmap}(x' | x) = &\indicator{\pi(x) = \pi^*(x)} P_{\textrm{good}}(\obsmap(x') | s = \obsmap(x)) \frac{1}{|\cX_{\obsmap(x')}|}\\
    &+ \indicator{\pi(x) \neq \pi^*(x)} P_{\textrm{bad}}(\obsmap(x') | s = \obsmap(x)) \frac{1}{|\cX_{\obsmap(x')}|}
\end{align*}
where $P_{\textrm{good}} \in \bbR^{\cS \times \cS}$ and $P_{\textrm{bad}} \in \bbR^{\cS \times \cS}$  are the latent state transition kernels when the agent follows a good action and bad action respectively.
Without loss of generality, we can assume that latent states are ordered as
\begin{align*}
    (1, g), (2, g), \dots, (d, g), (1, b), (2, b), \dots, (d, b), +, -
\end{align*}
in which case the agent can only move forward (or stay in the same state) in this order. Thus, when writing $P_{\textrm{good}}$ and $P_{\textrm{bad}}$ as matrices over $\cS \times \cS$ in this order, they are upper-triangular matrices. Their eigenvalues correspond to the entries on the diagonal and hence, the probability of staying in each latent state is an eigenvalue, including $1 - p_i$ for states $(i, b)$ all $i \in [d-1]$ for both $P_{\textrm{good}}$ and $P_{\textrm{bad}}$.
In matrix form, the transition matrix over observations can be written as
\begin{align*}
    T^{\pi}_{\pi^*, \obsmap} = \frac{|\cS|}{|\cX|}I_{\textrm{good}} \Phi^T P_{\textrm{good}} \Phi + \frac{|\cS|}{|\cX|}I_{\textrm{bad}} \Phi^T P_{\textrm{bad}} \Phi,
\end{align*}
where $\Phi \in \bbR^{\cS \times \cX}$ with $\Phi_{s, x} = \indicator{\obsmap(x) = s}$ is a matrix form of $\obsmap$ and $I_{\textrm{good}} \in \bbR^{\cX \times \cX}$ and $I_{\textrm{bad}} \in \bbR^{\cX \times \cX}$ are diagonal matrices with entries $[I_{\textrm{good}}]_{x, x} = \indicator{\pi(x) = \pi^*(x)}$ and $[I_{\textrm{bad}}]_{x, x} = \indicator{\pi(x) \neq \pi^*(x)}$, respectively.
By the Weinstein–Aronszajn identity, the eigenvalues of $T^{\pi}_{\pi^*, \obsmap}$ are identical to the eigenvalues of 
\begin{align*}
    \frac{|\cS|}{|\cX|} \Phi I_{\textrm{good}} \Phi^T P_{\textrm{good}}  + \frac{|\cS|}{|\cX|} \Phi I_{\textrm{bad}} \Phi^T P_{\textrm{bad}} = I_{\cS, \textrm{good}}P_{\textrm{good}} + I_{\cS, \textrm{bad}}P_{\textrm{bad}},
\end{align*}
where $I_{\cS, \textrm{good}} \in \bbR^{\cS \times \cS}$ and $I_{\cS, \textrm{bad}} \in \bbR^{\cS \times \cS}$ are diagonal matrices that contain for each $s \in \cS$ the probability that policy $\pi$ matches $\pi^*$ or does not match $\pi^*$ on observations of $s$, respectively. Finally, $I_{\cS, \textrm{good}}P_{\textrm{good}} + I_{\cS, \textrm{bad}}P_{\textrm{bad}}$ is also an upper triangular matrix whose eigenvalues are the entries on the diagonal.
Therefore, the eigenvalues of this matrix and $T^{\pi}_{\pi^*, \obsmap}$ are
\begin{align*}
    \begin{cases}
    1 - p_i & \textrm{ for } s = (i, b) \textrm{ where } i \in [d-1]\\ 
    (1 - p_i) \Pr_{x \sim \textrm{Unif}(\cX_s)}(\pi^*(x) = \pi(x) )
    & \textrm{ for } s = (i, g) \textrm{ where } i \in [d-1]\\
    0 & \textrm{ for } s \in \{+, (d, g), (d, b))\}\\
    1 & \textrm{ for } s = -.
    \end{cases}
\end{align*}
Thus, the rank of $T^{\pi}_{\pi^*, \obsmap}$ is at most $2d - 1$.
Analogously, we can show that the eigenvalues of $T^{\pi}_{0, \obsmap}$ are
\begin{align*}
    \begin{cases}
    \frac{1 - p_i}{2} & \textrm{ for } s = (i, b) \textrm{ and } (i, g) \textrm{ where } i \in [d-1]\\ 
    0 & \textrm{ for } s \in \{+, (d, g), (d, b))\}\\
    1 & \textrm{ for } s = -.
    \end{cases}
\end{align*}
Thus, the rank of $T^{\pi}_{\pi^*, \obsmap}$ is at most $2d - 1$.
\end{proof}

\begin{lemma}[Gilbert-Varshamov bound \citep{massart2007concentration}]\label{lem:gilbert_varshamov}
Let $N > 1$. There exists a subset $\cV$ of $\{0, 1\}^N$ of size $|\cV| \geq \exp(N / 8)$ such that 
\begin{align}
    \sum_{i = 1}^N\indicator{v_i \neq v_i'} \geq \frac{N}{4}
\end{align}
for all $v, v' \in \cV$. 
\end{lemma} 

\subsection{Proof of \pref{thm:basic_lower_bound}} \label{app:basic_lower_bound_proof} 
In the following, we provide the lower bound which states that the factor of $\Omega(H^d)$ in unavoidable without making further assumptions. We restate \pref{thm:basic_lower_bound} here with explicit constants:  
\begin{theorem} 
Let $\wt \epsilon \in (0, 1/26)$, $\delta \in (0, 1/2)$, $d \geq 4$ and $H \geq 219 d$. There exists a  realizable policy class of size $(H/d)^d$ and a family of MDPs with rank at most $2d$, finite observation space, horizon $H$ and two actions such that: Any algorithm that returns an $\wt \epsilon$-optimal policy in any MDP in this family with probability at least $1 - \delta$ has to collect at least
\begin{align*} 
    \frac{1}{12168 \cdot H\wt \epsilon^2}  \prn[\Big]{\frac{H}{41 d }}^{d/2} \log\prn[\Big]{\frac{1}{2 \delta}} 
\end{align*}
episodes in expectation in some MDP in this family. 
\end{theorem}

\newcommand{\alg}{\mathscr{A}}

\begin{proof} Consider any $(\wt \epsilon, \delta)$-PAC RL algorithm $\alg$. Our lower bound is based on the policy class $\Pi$ and the family of MDPs $\mdpfamily$ constructed in \pref{app:lower_bound_construction}. We set $\abs{\Pi} = (H/d)^d$ and  $p_i = d/H$ for all $i \in [d]$. ~\\ 

We first define additional notation. Let the random variable $G$ denote the first time-step in an episode when an observation from the latent state $(d, g)$ or $(d, g)$ is observed. In order to reach these latent states, the agent is required to do $d-1$ latent state progressions, each happening with probability $p_1, \ldots, p_{d-1}$ respectively. Furthermore, for our constructions in \pref{app:lower_bound_construction} of the class $\cM$, we note that the distribution of $G$ only depends on $\{p_i\}_{i < d}$, $d$ and $H$, but is otherwise independent of the MDP instance, the parameter $\eps$ and the played policy. In fact, when $p_i = d/H$, an application of \pref{lem:G_ub} implies that 
\begin{align*}
\Pr(G \leq H-1) \geq 1 - \exp(- 2/5). \numberthis \label{eq:lb_basic_G_bound} 
\end{align*}
Now, consider any MDP $M_{\pi^*, \phi} \in \mdpfamily$ and let $V_{\pi^*, \obsmap}(\pi)$ denote the expected return of the policy $\pi$ in the MDP $M_{\pi^*, \phi}$. From our MDP construction, we note that for the optimal policy $\pi^*$, 
\begin{align*} 
V_{\pi^*, \obsmap}(\pi^*) 
= \prn[\Big]{\frac{1}{2} + \epsilon} \Pr\prn*{G \leq H-1}. 
\end{align*} 
Similarly, for any other policy $\pi \in \Pi$,  we have\footnote{Throughout the proof, for any random variable $Y$, we define the notation $\En^{\pi}_{\pi^*, \phi} \brk*{Y}$ to denote the expectation of $Y$ where the trajectory is drawn using the policy $\pi$ in the MDP $M_{\pi^*, \phi}$.} 
\begin{align*}
V_{\pi^*, \obsmap}(\pi) 
=\frac{1}{2} \Pr\prn*{G \leq H-1} 
+ \epsilon \bbE_{\pi^*, \obsmap}^\pi \brk[\big]{ \indicator{\pi(X_{1:G-1}) = \pi^*(X_{1:G-1})}} \Pr\prn*{G \leq H-1}, 
\end{align*} where $\{\pi(X_{1:G-1}) = \pi^*(X_{1:G-1})\}$ denotes the event that the action chosen by $\pi$ agrees with that chosen by $\pi^*$ on the observations $X_{1:G-1}$ up to time step $G - 1$. Hence, we have that the suboptimality gap for the policy $\pi$ is 
\begin{align*}
V_{\pi^*, \obsmap}(\pi^*) - V_{\pi^*, \obsmap}(\pi) 
    &= \epsilon \Pr\prn*{G \leq H-1} \Pr_{\pi^*, \obsmap}^\pi \prn*{ \exists h \leq G-1  \text{~ s.t.~} \pi(X_{h}) \neq \pi^*(X_{h})}~. \numberthis \label{eq:lb_basic_suboptimality_relation} 
\end{align*}

Next, define the random variable $\tau$ to denote the number of episodes after which the algorithm $\alg$ terminates and let $\wh \pi$ denote the policy returned on termination. Both, $\tau$ and $\wh \pi$, depend on the algorithm $\alg$ and the underlying MDP on which $\alg$ collects data from. Since the algorithm $\alg$ is $(\wt \epsilon, \delta)$-PAC, we have that for any MDP $M_{\pi^*, \phi}$, with probability at least $1 - \delta$, 
\begin{align*}
V_{\pi^*, \obsmap}(\pi^*) - V_{\pi^*, \obsmap}(\wh \pi)  \leq \wt \epsilon. 
\end{align*} 
Using the relation in \pref{eq:lb_basic_suboptimality_relation}, and plugging in the bound  in \pref{eq:lb_basic_G_bound}, in the above, we get that 
\begin{align*}
\epsilon \Pr\prn*{G \leq H-1} \Pr_{\pi^*, \obsmap}^{\wh \pi} \prn*{ \exists h \leq G-1  \text{~ s.t.~} {\wh \pi}(X_{h}) \neq \pi^*(X_{h})} \leq \wt \epsilon, \numberthis \label{eq:suboptimality_bound_1}  
\end{align*} must hold with probability at least $1 - \delta$ for any MDP $M_{\pi^*, \phi}$. For our lower bound constructions, we set 
\begin{align*}
\eps = \frac{4 \wt \epsilon}{\Pr \prn*{G \leq H-1}} \leq 13 \wt \eps,  \numberthis \label{eq:eps_value_lower_bound}  
\end{align*} and thus \pref{eq:suboptimality_bound_1} implies that $\Pr^{\hat \pi}_{\pi^*, \phi} \prn{ {\wh \pi}(X_{1:G-1}) = \pi^*(X_{1:G-1})} \geq 3/4$ must hold with probability at least $1 - \delta$. Define the event 
\begin{align*} 
\mathrm{Opt}_{\pi^*, \phi}^{\alg} \ldef{} \crl[\big]{ \Pr^{\hat \pi}_{\pi^*, \phi} \prn{ A_{1:G-1} = \pi^*(X_{1:G-1})} \geq 3/4 }.   
\end{align*} 
The above analysis suggests that for any $M_{\pi^*, \phi}$ 
\begin{align*}
\Pr_{\pi^*, \phi} \prn*{\mathrm{Opt}_{\pi^*, \phi}^{\alg}} \geq 1 - \delta. \numberthis \label{eq:suboptimality_bound_2} 
\end{align*}

Next, for any $\pi^* \in \Pi$, define the measure $\Pr_{\pi^*}(Y) = \frac{1}{|\Phi|} \sum_{\obsmap \in \Phi} \Pr_{\pi^*, \obsmap}(Y)$, i.e.  the probability measure induced by first picking $\obsmap$ uniformly at random from the set of all mappings $\Phi$ and then considering the distribution induced by $M_{\pi^*, \obsmap}$. The measure $\Pr_{0}(Y) = \frac{1}{|\Phi|} \sum_{\obsmap \in \Phi} \Pr_{0, \obsmap}(Y)$ is defined analogously for the MDP $M_{0, \phi}$. Thus, from \pref{eq:suboptimality_bound_2}, we have that for any $\pi^* \in \Pi$,  
\begin{align*}
\Pr_{\pi^*}(\mathrm{Opt}_{\pi^*, \phi}^{\alg}) \geq 1 - \delta. \numberthis \label{eq:suboptimality_bound_3} 
\end{align*} 

We are now ready to prove the desired lower bound. Let 
\begin{align*} 
    T_{\max} \ldef{} \frac{1}{\delta} \cdot \frac{1}{12168 H \wt \epsilon^2} \prn[\Big]{\frac{H}{41 d }}^{d/2} \cdot \log\prn*{1/2\delta}. 
\end{align*} 
There are two natural scenarios: either (a)  $\Pr_{\pi^*}(\tau > T_{\max} ) > \delta$ for some $\pi^* \in \Pi$, or (b)  $\Pr_{\pi^*}(\tau > T_{\max} ) \leq \delta$ for all $\pi^* \in \Pi$. We analyse the two cases separately below. 

\paragraph{Case-(a): $\mb{\Pr_{\pi^*}(\tau > T_{\max} ) > \delta}$ for some $\pi^* \in \Pi$.} The lower bound follows immediately in this case. Note that,  
\begin{align*}
    \max_{\phi \in \Phi} \bbE_{\pi^*, \phi} \brk*{\tau}
    &>  \bbE_{\pi^*} \brk*{\tau} \geq \Pr_{\pi^*}(\tau > T_{\max} ) \cdot T_{\max} \geq \delta T_{\max}~.
\end{align*}
Hence, there exists an MDP in $M_{\pi^*, \phi} \in \mdpfamily$ for which the expected number of episodes collected by the algorithm $\alg$ is at least $\delta T_{\max}$, which is the desired lower bound.

\paragraph{Case-(b): $\mb{\Pr_{\pi^*}(\tau > T_{\max} ) \leq \delta}$ for all $\pi^* \in \Pi$.} Due to \pref{eq:suboptimality_bound_3}, for any policy $\pi^* \in \Pi$, we have: 
\begin{align*}
    \Pr_{\pi^*}(\tau \leq T_{\max} \wedge \mathrm{Opt}_{\pi^*, \phi}^{\alg}) &= \Pr_{\pi^*}(\mathrm{Opt}_{\pi^*, \phi}^{\alg}) - \Pr_{\pi^*}(\tau > T_{\max} \wedge \mathrm{Opt}_{\pi^*, \phi}^{\alg})  \\
    &\geq 1 - 2 \delta. 
\end{align*}
The above condition intuitively states that the policy returned by the algorithm will, with high probability, match the actions of the optimal policy for $G-1$ time steps for any policy $\pi^* \in \Pi$. On the other hand, we show in \pref{lem:packing_lb} through a packing argument that the expected number of policies that can be matched for $G-1$ steps when observations are drawn uniformly is bounded, i.e. 
\begin{align*}
    \bbE_{unif} \brk[\big]{ 
    \sum_{\pi^*} \indicator{\pi^*(X_{1:G}) = \pi(X_{1:G})}
    } \leq  \prn{41 \log \prn{H/d}}^d H + 2 ~, 
\end{align*} where the notation $ \bbE_{unif}[\cdot]$ denotes that $X_{1:G}$ are drawn independently from $\text{uniform}(\cX)$.  We denote this bound by $C =  \prn*{41 \log \prn*{H/d}}^d H + 2$. We show in \pref{lem:stopping_time_bound} through a careful information-theoretic argument that the expected stopping time of the algorithm $\cA$ on instances $M_{0, \phi}$ is bounded from below as
\begin{align*} 
    \bbE_{0}\brk*{ \tau
     } 
      &\geq 
     \frac{1}{8 \epsilon^2} \prn[\Big]{\frac{|\Pi|}{C}  - 
     \frac{8}{3} 
     }\log\prn*{1/2 \delta}
      -  \prn[\Big]{2T_{\max} + \frac{7}{12 \epsilon^2}\log\prn*{1/2\delta}} \cdot \frac{|\Pi|}{C} \cdot \obserr{T_{\max}}~,
\end{align*} 
where $\obserr{T_{\max}} \ldef{} {4T^2_{\max} H^2 |\cS|}/{N}$ accounts for the differences in observation distributions in different instances of $\mdpfamily$.

Plugging in the value of $|\Pi|$ and $C$, we note that
\begin{align*}
    \frac{|\Pi|}{C}
    \geq \frac{\prn*{H/d}^d}{2H \prn*{41 \log \prn*{H/d}}^d}
    = \frac{1}{2H} \prn[\Big]{\frac{H}{41 d \log \prn*{H/d}}}^d 
    \geq \frac{1}{2H} \prn[\Big]{\frac{H}{41 d }}^{d/2}~.
\end{align*}

Additionally, for $d \geq 4$ and $H/d \geq 219$, we have $8/3 \leq  \prn*{{H}/{41 d }}^{d/2} / 4H$ . 
Combining these bounds yields
\begin{align*}
        \bbE_{0}\brk*{ \tau
     } &\geq 
     \frac{1}{6084 H\wt \epsilon^2} \prn[\Big]{\frac{H}{41 d }}^{d/2} \log\prn*{1/{2 \delta} }
      -  \prn[\Big]{2T_{\max} + \frac{7}{12 \epsilon^2}\log\prn*{1/2\delta}} \cdot  \frac{\abs{\Pi}}{C} \cdot \obserr{T_{\max}},
\end{align*} 
Finally this bound only depends on the number of observations $N$ through $\obserr{T_{\max}}$ which goes to zero as $N \rightarrow \infty$. Therefore, we can pick $N$ large enough such that the second term becomes small enough and thus
\begin{align*}
        \bbE_{0}\brk*{ \tau
     } &\geq 
     \frac{1}{12168 H\wt \epsilon^2} \prn[\Big]{\frac{H}{41 d }}^{d/2} \log\prn*{1/2\delta}~.
\end{align*}
Since this bound holds on average over all MDPs $M_{0, \obsmap} \in \cM$, this lower bound must also hold in at least one specific $M_{0, \obsmap} \in \cM$. This gives us the desired statement. 
 \end{proof}

For the rest of the section, we will build on the notation introduced in the above proof. The following technical lemma gives a lower bound on $\En_0 \brk*{\tau}$ for the case-(b) above. 
\begin{lemma}\label{lem:stopping_time_bound} Let $\alg$ be any $(\wt \epsilon, \delta)$-PAC RL algorithm. Let $T_{\max} \in \bbN$ and assume that $\Pr_{\pi^*}(\tau \leq T_{\max} \wedge \mathrm{Opt}_{\pi^*, \phi}^{\alg}) \geq 1 - 2 \delta$ holds for all $\pi^* \in \Pi$. 
Further, let $C > 0$ denote an upper-bound on the number of policy matches per episode, i.e., for all $\pi \in \Pi$, 
\begin{align*} 
    \bbE_{\textrm{unif}}\brk[\Big]{
    \sum_{\pi^*} \indicator{\pi^*(X_{1:G}) = \pi(X_{1:G})}
    } \leq C. 
\end{align*}
Then the expected stopping time $\tau$ for the algorithm $\alg$ over MDP instances $M_{0, \phi}$ where $\phi$ is drawn randomly from $\Phi$ is bounded from below as
\begin{align*} 
    \bbE_{0}\brk*{ \tau
     } 
      &\geq 
     \frac{1}{8 \epsilon^2} \prn[\Big]{\frac{|\Pi|}{C}  - 
     \frac{8}{3} 
     }\log\prn*{1/2 \delta}
      -  \prn[\Big]{2T_{\max} + \frac{7}{12 \epsilon^2}\log\prn*{1/2\delta}} \cdot \frac{|\Pi|}{C} \cdot \obserr{T_{\max}}~,
\end{align*}
where $\obserr{T_{\max}} = {4T_{\max}^2 H^2 |\cS|}/{N}$.
\end{lemma}
\begin{proof} 

Let $G_i$ denote the first timestep when the agent reaches the latent state $(d, g)$ or $(d, b)$ in the $i$th episode collected by the algorithm $\alg$. We denote by
\begin{align*} 
    N_{\pi^*}^{\tau \wedge T_{\max}}
    = \sum_{i=1}^{\tau \wedge T_{\max}}
    \indicator{A_{i, 1:G_i-1} = \pi^*(X_{i, 1:G_i-1})}
\end{align*} 
 the number of episodes among the first $\tau \wedge T_{\max} = \min\{\tau, T_{\max}\}$ episodes where the actions $A_{i, 1:G_i - 1}$ played by $\alg$ in the $i$th episode matches those of $\pi^*$ on the corresponding observations, until the latent state $(d, g)$ or $(d, b)$ was reached. 
 We first lower-bound the expected value of $N_{\pi^*}^{\tau \wedge T_{\max}}$ under the measure induced by $\Pr_{0}$. To that end, we introduce auxiliary MDPs $M_{0, \pi^*, \obsmap}$ that are identical to $M_{\pi^*, \obsmap}$ on all latent states except for $(d, g)$. In $M_{0, \pi^*, \obsmap}$, we transition to both $+$ and $-$ with equal probability from the latent state $(d, g)$. \footnote{Note that the MDPs $M_{0, \pi^*, \obsmap}$ are only an analytical tool and do not belong to the class $\mdpfamily$.}
Analogous to $\Pr_{\pi^*}$, we define $\Pr_{0, \pi^*}$ to denote the law when $\obsmap$ is drawn uniformly from $\Phi$ beforehand and the underlying MDPS is $M_{0, \pi^*, \phi}$. We also define $\Pr_0$ as the law when $\pi^*$ is additionally drawn uniformly at random from $\Pi$ beforehand. Finally, $\En_{0, \pi^*} \brk*{\cdot}$ and $\En_0 \brk*{\cdot}$ are defined as the expectations under $\Pi_{0, \pi^*}$ and $\Pr_{0}$ respectively.  Following the standard machinery for lower-bounds \citep{garivier2019explore, domingues2021episodic}, we get that 
\begin{align*}
   \bbE_{0}\brk{N_{\pi^*}^{\tau \wedge T_{\max}}} 
   &\overgeq{\proman{1}} \bbE_{0, \pi^*}\brk{N_{\pi^*}^{\tau \wedge T_{\max}}} -  T_{\max} \obserr{T_{\max}}\\
   &\geq 
   \bbE_{0, \pi^*}\brk{N_{\pi^*}^{\tau \wedge T_{\max}}} \cdot 
   \frac{\operatorname{kl}(1/2, 1/2 + \epsilon)}{4 \epsilon^2}
   -  T_{\max} \obserr{T_{\max}}\\
   &= \frac{1}{4 \epsilon^2}
   \operatorname{KL}\prn[\Big]{\Pr_{0, \pi^*}^{\cF_{\tau \wedge T_{\max}}}, \Pr_{\pi^*}^{\cF_{\tau \wedge T_{\max}}}}
   -  T_{\max} \obserr{T_{\max}}, 
 \intertext{where the inequality $\proman{1}$ follow from an application of \pref{lem:unif_obs_1}. In the above, for any distributions $P$ and $Q$, the notation $\KL{P}{Q}$ denotes the KL-divergence between $P$ and $Q$, and the superscript $\cF_{\tau \wedge T_{\max}}$ denotes the conditioning w.r.t. the natural filtration generated by the first $\tau \wedge T_{\max}$ episodes. Further, define $\text{kl}(p, q)$ to denote the KL-divergence of two Bernoulli random variables with means $p$ and $q$ respectively. We now apply Lemma~1 of \citet{garivier2019explore} which gives that for any  $\cF_{\tau \wedge T_{\max}}$-measurable variable random variable $Z$ with values in $[0, 1]$, we have that } 
     \bbE_{0}\brk{N_{\pi^*}^{\tau \wedge T_{\max}}}  &\geq
      \frac{1}{4 \epsilon^2} 
   \operatorname{kl}\prn*{\bbE_{0, \pi^*}[Z], \bbE_{\pi^*}[Z]}
   -  T_{\max} \obserr{T_{\max}}\\
      &\overgeq{\proman{2}} 
      \frac{1}{4 \epsilon^2} \prn*{1 - \bbE_{0, \pi^*}[Z]}\log\prn[\Big]{\frac{1}{1 - \bbE_{\pi^*}[Z]}} 
      -\frac{1}{4 \epsilon^2} \log(2)
   - T_{\max} \obserr{T_{\max}}\\
         &\overgeq{\proman{3}}
      \frac{1}{4 \epsilon^2} \prn*{1  - \bbE_{0}[Z]}\log\prn[\Big]{\frac{1}{1 - \bbE_{\pi^*}[Z]}} 
      -\frac{\log(2)}{4 \epsilon^2}  -  \prn[\Big]{T_{\max} + \frac{1}{4 \epsilon^2}\log\prn[\Big]{\frac{1}{1 - \bbE_{\pi^*}[Z]}}} \obserr{T_{\max}} \numberthis \label{eq:lower_bound_2-1}
\end{align*} where the inequality $\proman{2}$ follows due to the fact that  $\text{kl}(p, q) \geq (1 - p) \log(1/(1-q)) - \log(2)$, and the inequality $\proman{3}$ holds from an application of \pref{lem:unif_obs_1}. 
Next, define the random variable $Z_{\pi*}$ as 
\begin{align*} 
    Z_{\pi^*} = \Pr_{\pi^*}\prn*{\tau \leq T_{\max} \wedge \mathrm{Opt}_{\pi^*, \phi}^{\cA}~|~ \cF_{\tau \wedge T_{\max}}}
\end{align*}
and note that $Z_{\pi^*}$ is $\cF_{\tau \wedge T_{\max}}$-measurable by construction.  Thus, plugging $Z = Z_{\pi^*}$ in \pref{eq:lower_bound_2-1} and using the fact that $\bbE_{\pi^*}\brk{Z_{\pi^*}} = \Pr_{\pi^*}\prn{\tau \leq T_{\max} \wedge \mathrm{Opt}_{\pi^*, \phi}^{\cA}} \geq 1 - 2 \delta$ (by assumption), we get that 
\begin{align*}
    \bbE_{0}\brk{N_{\pi^*}^{\tau \wedge T_{\max}}}  &\geq 
    \frac{1}{4 \epsilon^2} \prn*{1  - \bbE_{0}\brk*{Z_{\pi^*}}}\log\prn*{1/2\delta}
      -\frac{\log(2)}{4 \epsilon^2} 
   -  \prn[\Big]{T_{\max} + \frac{1}{4 \epsilon^2}\log\prn*{1/ 2\delta}} \obserr{T_{\max}},
\end{align*}
Summing the above for all policies $\pi^* \in \Pi$ yields that  
\begin{align*}
  \sum_{\pi^* \in \Pi} \bbE_{0}\brk{N_{\pi^*}^{\tau \wedge T_{\max}}} 
    &\geq 
    \frac{1}{4 \epsilon^2} \prn[\Big]{|\Pi|  - \sum_{\pi^* \in  \Pi}\bbE_{0}\brk*{Z_{\pi^*}}}\log\prn*{ 1/ 2 \delta }
      -\frac{|\Pi| \log(2)}{4 \epsilon^2} -  \prn[\Big]{T_{\max} + \frac{\log\prn{1 / 2 \delta }}{4 \epsilon^2}} |\Pi| \obserr{T_{\max}}. \numberthis \label{eq:tau_bound_raw} 
\end{align*} 

We further lower bound the above by deriving an upper bound on $\sum_{\pi^* \in \Pi} \En_0 \brk*{Z_{\pi^*}}$. Note that for any $\pi^* \in \Pi$, 
\begin{align*} 
    Z_{\pi^*} 
    &= \Pr_{\pi^*}\prn*{\tau \leq T_{\max} \wedge \mathrm{Opt}_{\pi^*, \phi}^{\alg} \mid \cF_{\tau \wedge T_{\max}}}\\
    &\overeq{\proman{1}} \bbE_{\pi^*}\brk*{
    \indicator{\tau \leq T_{\max}}
    \indicator{
    \Pr^{\hat \pi}_{\pi^*, \phi} \prn*{ \wh \pi(X_{1:G-1}) = \pi^*(X_{1:G-1})} \geq 3/4}
    ~|~ \cF_{\tau \wedge T_{\max}}}\\
    &\overleq{\proman{2}} \frac{4}{3} \bbE_{\pi^*}\brk*{
    \indicator{\tau \leq T_{\max}}
    \Pr^{\hat \pi}_{\pi^*, \phi} \prn*{ \wh \pi(X_{1:G-1}) = \pi^*(X_{1:G-1})} 
    ~|~ \cF_{\tau \wedge T_{\max}}}
\end{align*} 

where the equality $\proman{1}$ above follows from the definition of $\mathrm{Opt}_{\pi^*, \phi}^{\alg}$, and the inequality in $\proman{2}$ holds from an application of Markov's inequality and using the fact that  $\indicator{\tau \leq T_{\max}}$ is $\cF_{\tau \wedge T_{\max}}$-measurable (by construction). Note that when $\tau \leq T_{\max}$ (the only outcomes where the random variable inside the expectation can be non-zero), we also have that  $\hat \pi$ is $\cF_{\tau \wedge T_{\max}}$-measurable.\footnote{This assumes a deterministic algorithm but we can handle stochastic algorithms by simply conditioning on $\hat \pi$ (and therefore the internal randomness of the algorithm) as well.} Thus, the only randomness in $\Pr^{\hat \pi}_{\pi^*, \phi}$ above is due to $\obsmap$ which affects the distribution of the observations $X_{1:G-1}$ inside $\Pr^{\hat \pi}_{\pi^*, \phi}$. However, note that this distribution is exactly the distribution of observations in the $(\tau+1)$th episode if we assume (without loss of generality) that the algorithm plays $\hat \pi$ in that episode. Thus, we can write the right hand side in the above as 
\begin{align*}
    Z_{\pi^*} & \leq  \frac{4}{3} 
    \indicator{\tau \leq T_{\max}}
    \Pr_{\pi^*} \prn*{ \pi^*(X_{\tau+1,1:G_{\tau+1}-1}) = \hat \pi(X_{\tau+1,1:G_{\tau+1}-1}) 
    \mid \cF_{\tau \wedge T_{\max}}}~.
\end{align*}
An application of \pref{lem:uniform_observations_cond} in the above implies that \begin{align} \label{eqn:zrv_bound}
    Z_{\pi^*} & \leq  \frac{4}{3} 
    \indicator{\tau \leq T_{\max}} 
    \prn[\Big]{\Pr_{\textrm{unif}} \prn*{ \pi^*(X_{1:G-1}) = \hat \pi(X_{1:G-1}) } + \frac{2 |\cS| H^2 (T_{\max} + 1)}{N}} 
    ~, 
\end{align} which further implies that 
\begin{align*}
    \sum_{\pi^* \in \Pi}\bbE_{0}\brk*{Z_{\pi^*}} &\leq 
    \frac{4}{3}\bbE_0 \brk[\Big]{
    \sum_{\pi^* \in \Pi} 
    \Pr_{\textrm{unif}} \prn*{ \pi^*(X_{1:G-1}) = \hat \pi(X_{1:G-1}) }}
    + \frac{4}{3} |\Pi| \obserr{T_{\max}} \\
    &=     \frac{4}{3}C
    + \frac{4}{3} |\Pi| \obserr{T_{\max}},  \numberthis \label{eqn:zrv_bound2}
\end{align*} where the value of $C$ and $\obserr{T_{\max}}$ are given in the lemma statement. Plugging the above bound in \pref{eq:tau_bound_raw}, we get that  
\begin{align*}
   \sum_{\pi^* \in \Pi} \bbE_{0}\brk{N_{\pi^*}^{\tau \wedge T_{\max}}} 
    &\geq 
    \frac{1}{4 \epsilon^2} \prn[\Big]{|\Pi|  - \frac{4}{3} C}\log\prn*{ 1/ 2 \delta }
      -\frac{|\Pi| \log(2)}{4 \epsilon^2}  -  \prn[\Big]{T_{\max} + \frac{7\log\prn*{1 / 2 \delta }}{12 \epsilon^2}} |\Pi| \obserr{T_{\max}}. \numberthis \label{eq:tau_bound_raw2} 
\end{align*} 

\paragraph{Relating policy matches to stopping time: } In the following, we show an upper bound on $ \sum_{\pi^* \in \Pi} \bbE_{0}\brk{N_{\pi^*}^{\tau \wedge T_{\max}}} $ that, when taken together with the above lower bound, gives us the desired lower bound on $\En_0 \brk*{\tau}$. We note that 
\begin{align*} 
    \sum_{\pi^* \in \Pi} \bbE_{0}\brk{N_{\pi^*}^{\tau \wedge T_{\max}}} 
    &= \sum_{t=1}^{T_{\max}} \bbE_{0}\brk[\Big]{ \indicator{\tau > t-1}
    \sum_{\pi^* \in \Pi}  \indicator{A_{t, 1:G_t-1} = \pi^*(A_{t, 1:G_t-1})}}\\
    &= \sum_{t=1}^{T_{\max}} \bbE_{0}\brk[\Big]{ \indicator{\tau > t-1}
    \bbE_{0} \brk[\Big]{\sum_{\pi^* \in \Pi}  \indicator{\pi_t(X_{t, 1:G_t-1}) = \pi^*(X_{t, 1:G_t-1})} \mid \pi_{t}, \cF_{t-1}}}\\
    & \overset{\proman{1}}{\leq} 
    \sum_{t=1}^{T_{\max}} \bbE_{0}\brk[\Big]{ \indicator{\tau > t-1}
    \bbE_{\textrm{unif}} \brk[\Big]{\sum_{\pi^* \in \Pi}  \indicator{\pi_t(X_{1:G-1}) = \pi^*(X_{1:G-1})} }} + T_{\max}
    |\Pi| \obserr{T_{\max}}\\
    & \overset{\proman{2}}{\leq} 
    \sum_{t=1}^{T_{\max}} \bbE_{0}\brk*{ \indicator{\tau > t-1}
     C } 
    + T_{\max}
    |\Pi| \obserr{T_{\max}}
    \\& 
    \leq C  \bbE_{0}\brk*{ \tau \wedge T_{\max}
     } 
    + T_{\max}
    |\Pi| \obserr{T_{\max}} \\ 
    &\leq C  \bbE_{0}\brk*{ \tau
     } 
    + T_{\max} 
    |\Pi| \obserr{T_{\max}} \numberthis \label{eq:tau_bound_raw3}
\end{align*} 
where the inequality $\proman{1}$ follows from an application of \pref{lem:uniform_observations_cond} and the inequality $\proman{2}$ follows from the definition of $C$ given in the lemma statement.

Combining the lower bound in \pref{eq:tau_bound_raw2} with the upper bound in \pref{eq:tau_bound_raw3} and rearranging the terms yields that  
\begin{align*}
    \bbE_{0}\brk*{ \tau
     } &\geq 
     \frac{1}{4 \epsilon^2} \prn[\Big]{\frac{|\Pi|}{C}  - 
     \frac{4}{3}
     }\log\prn*{1/2\delta}
      -\frac{|\Pi|}{C}\frac{1}{4 \epsilon^2} \log(2)
      -  \prn[\Big]{2T_{\max} + \frac{7}{12 \epsilon^2}\log\prn*{1/2\delta}} \frac{|\Pi|}{C} \obserr{T_{\max}}
      \\
      &\geq 
     \frac{1}{8 \epsilon^2} \prn[\Big]{\frac{|\Pi|}{C}  - 
     \frac{8}{3}
     }\log\prn*{1/2\delta}
      -  \prn[\Big]{2T_{\max} + \frac{7}{12 \epsilon^2}\log\prn*{1/2\delta}} \frac{|\Pi|}{C} \obserr{T_{\max}}
\end{align*}
where the last inequality is due to the fact that $\delta \leq 1 \leq \exp(2) / 4$ and thus $\log (1/2\delta) \geq 2 \log(2)$. This concludes the desired statement. 
\end{proof}

\subsection{Proof of \pref{thm:adaptive_lower_bound} (eigenspectrum dependent lower bounds)}  \label{app:adaptive_lower_bound_proof} 
 We here restate \pref{thm:adaptive_lower_bound} with explicit constants:
\begin{theorem}[Adaptive lower bound]  \label{thm:adaptive_lower_bound_full}
Let $\wt \epsilon \in (0, \frac{1}{16})$, $\delta \in (0, \frac{1}{2})$, $d \geq 4$ and $(\lambda_i)_{i \in [d-1]} \in [0, 1]^{d-1}$
satisfy
\begin{align*}
    \prn[\Big]{\frac{16}{3}\prn{15(d-1)}^{d-1}}^2  \leq  \prod_{i=1}^{d-1} \frac{1}{1 - \lambda_i} &\leq \frac{8}{7}\exp\prn*{H/2} 
& \textrm{and}
&&\sum_{i=1}^{d-1} \frac{1}{1 - \lambda_i} \leq \frac{H}{4 \ln(4d)}~.
\end{align*}

Then, there is a realizable policy class and family of MDPs with rank at most $\Theta(d)$, finite observation space, horizon $H$ and two actions such that:  For each $i \in [d]$, policy $\pi$ and MDP $M$ in this class, there is an eigenvalue of the induced transition matrix $T^\pi_{M}$ in $[\lambda_i/2, \lambda_i]$. Furthermore, any algorithm that returns, with probability at least $1 - \delta$ an $\epsilon$-optimal policy for any MDP in this family, has to collect at least  
\begin{align*} 
         \frac{1}{1100 \wt \epsilon^2}     \prn[\Big]{\frac{1}{15(d-1)}}^{d-1}
    \sqrt{\prod_{i=1}^{d-1}\frac{1}{1 - \lambda_i}}  \log\prn*{1/2\delta}
\end{align*} 
episodes in expectation in some MDP in this family. 
\end{theorem} 
\begin{proof}
This theorem follows immediately from setting $p_i = 1 - \lambda_i$ with \pref{lem:P_eigenvalues} and \pref{lem:p_based_lb} below.
\end{proof}
\begin{lemma}\label{lem:p_based_lb} 
Let $\wt \epsilon \in (0,  1/16)$ and
let $\cM$ be the family of MDPs defined in the proof of \pref{thm:basic_lower_bound} but where the probability $p$ for progression in latent states $(i, g)$ and $(i,b)$ is set to $p_i \in (0, 1)$. If
\begin{align*}
     \prn[\Big]{\frac{16}{3}\prn{15(d-1)}^{d-1}}^2 \leq \prod_{i=1}^{d-1} \frac{1}{p_i} &\leq  \frac{8}{7}\exp\prn*{H/2} 
& \textrm{and} 
&&\sum_{i=1}^{d-1} \frac{1}{p_i} \leq \frac{H}{4 \ln(4d)}~.
\end{align*}
then any learner that returns an $\wt \epsilon$-optimal policy in every MDP in this class with probability at least $1- \delta$ has to collect at least
\begin{align*}
         \frac{1}{1100\wt \epsilon^2}   \prn[\Big]{\frac{1}{15(d-1)}}^{d-1}
    \sqrt{\prod_{i=1}^{d-1}\frac{1}{p_i}}  \log\prn*{1/2\delta}
\end{align*}
episodes in expectation in at least one MDP in the family. 
\end{lemma}

\begin{proof}
We follow the proof of \pref{thm:basic_lower_bound} but set
\begin{align*}
T_{\max} \ldef{}
    \frac{1}{\delta} \cdot \frac{1}{1100 \wt \epsilon^2}     \prn[\Big]{\frac{1}{15(d-1)}}^{d-1}
    \sqrt{\prod_{i=1}^{d-1}\frac{1}{p_i}}  \log\prn*{1/2\delta}~.
\end{align*}
Then one of two cases can happen: Either there is an MDP $M \in \cM$ in the class where the algorithm samples in expectation at least $\bbE_{M}[\tau] \geq \delta T_{\max}$ episodes,  or \begin{align*}
    \Pr_{\pi^*}(\tau \leq T_{\max} \wedge \mathrm{Opt}_{\pi^*, \phi}^{\alg}) \geq 1 - 2 \delta
\end{align*} holds for all $\pi^* \in \Pi$ where $\mathrm{Opt}_{\pi^*, \phi}^{\alg} = \{ \Pr^{\hat \pi}_{\pi^*, \phi} \prn*{ A_{1:G-1} = \pi^*(X_{1:G-1})} \geq 3/4 \}$ denote this event, where the policy returned by the algorithm $\hat \pi$ is $\wt \epsilon$-optimal in  $M_{\pi^*, \obsmap}$. Since the first case immediately gives us the desired lower bound, in the following, we consider the case that 
$\Pr_{\pi^*}(\tau \leq T_{\max} \wedge \mathrm{Opt}_{\pi^*, \phi}^{\alg}) \geq 1 - 2 \delta$ holds for every policy $\pi^* \in \Pi$. 

An application of \pref{lem:packing_lb_pi} gives us 
\begin{align*}
    \bbE_{\textrm{unif}}\brk[\Big]{
    \sum_{\pi^*} \indicator{\pi^*(X_{1:G} = \pi(X_{1:G})}
    } \leq  C  ~,
\end{align*}
for any policy $\pi$ with $C =  2 + |\Pi|  \cdot \prod_{i=1}^{d-1}\prn*{ p_i \frac{\log|\Pi|}{\log(8/7)} }$ as long as
\begin{align*}
  |\Pi| \leq \frac{8}{7}\exp\prn*{H/2}.
\end{align*}
Applying \pref{lem:stopping_time_bound}, the expected stopping time $\tau$ of  algorithm $\alg$ on instances $M_{0, \phi}$ is bounded from below as
\begin{align*}
    \bbE_{0}\brk*{ \tau
     } 
      &\geq 
     \frac{1}{8 \epsilon^2} \prn[\Big]{\frac{|\Pi|}{C}  - 
     \frac{8}{3}
     }\log\prn*{1/2\delta}
      -  \prn[\Big]{2T_{\max} + \frac{7}{12 \epsilon^2}\log\prn*{1/2\delta}} \frac{|\Pi|}{C} \obserr{T_{\max}}~.
\end{align*}
We now set $|\Pi| = \prod_{i=1}^{d-1} \frac{1}{p_i}$ and bound the ratio 
\begin{align*}
    \frac{|\Pi|}{C} &\geq \min \crl[\bigg]{
    \frac{1}{4}\prod_{i=1}^{d-1}\frac{1}{p_i},  
    \frac{\prod_{i=1}^{d-1}\frac{1}{p_i}}{\prn*{\ln \prod_{i=1}^{d-1}\frac{1}{p_i}}^{d-1}} \prn*{\ln \frac{8}{7}}^{d-1} 
    }\\
    &\geq \min \crl[\bigg]{
    \frac{1}{4}\prod_{i=1}^{d-1}\frac{1}{p_i},  
    \sqrt{\prod_{i=1}^{d-1}\frac{1}{p_i}} \prn*{\frac{\ln \frac{8}{7}}{2(d-1)}}^{d-1} 
    } \\
    &\geq   
    \prn[\Big]{\frac{1}{15(d-1)}}^{d-1}
    \sqrt{\prod_{i=1}^{d-1}\frac{1}{p_i}}  
\end{align*}
where the inequality in the second line follows from $\ln(y) \leq 2(d-1) y^{\frac{1}{2(d-1)}}$ and the last line above is due to the fact that $|\Pi| \geq 1$ and $d \geq 2$.
Now, 
\begin{align*}
   |\Pi| = \prod_{i=1}^{d-1} \frac{1}{p_i} \geq \prn[\Big]{\frac{16}{3}\brk{15(d-1)}^{d-1}}^2
\end{align*}
is sufficient for $\frac{|\Pi|}{C} \geq \frac{16}{3}$ which yields
\begin{align*}
    \bbE_{0}\brk*{ \tau
     } 
      &\geq 
     \frac{1}{16 \epsilon^2} \frac{|\Pi|}{C} \log\prn*{1/2\delta}
      -  \prn[\Big]{2T_{\max} + \frac{7}{12 \epsilon^2}\log\prn*{1/2\delta}} \frac{|\Pi|}{C} \obserr{T_{\max}}~.
\end{align*}
Since $\obserr{T_{\max}}$ is the only term that depends on $N$, we can pick $N$ large enough so that the first term dominates and
\begin{align*}
    \bbE_{0}\brk*{ \tau
     } 
      &\geq 
     \frac{1}{17 \epsilon^2} \frac{|\Pi|}{C} \log\prn*{1/2\delta}
     \geq
     \frac{1}{17 \epsilon^2}     \prn[\Big]{\frac{1}{15(d-1)}}^{d-1}
    \sqrt{\prod_{i=1}^{d-1}\frac{1}{p_i}}  \log\prn*{1/2\delta}
     .
\end{align*}
It only remains to resolve the $1/\epsilon^2$ to $1/\wt \epsilon^2$. To that end, 
we now bound the probability of reaching the goal state by the end of the episode by
\begin{align*}
    \Pr(G \leq H-1) \geq \Pr\prn[\Big]{G \leq \sum_{i=1}^{d-1} \frac{2}{p_i }\ln\frac{2d}{1/2}} \geq \frac{1}{2}~,
\end{align*}
because by \pref{lem:G_ub}, the probability that the agent spends more than $\frac{2}{p_i}\ln\frac{2d}{1/2}$ time steps in states $(i, b)$ or $(i, g)$ is bounded by $\frac{1}{2d}$. Thus,
\begin{align*}
    \epsilon = \frac{4 \wt \epsilon}{\Pr(G \leq H-1)} \leq 8 \wt \epsilon.
\end{align*}
which yields the final bound
\begin{align*}
    \bbE_{0}\brk*{ \tau
     } 
      &\geq 
     \frac{1}{1100 \wt \epsilon^2} \prn[\Big]{\frac{1}{15(d-1)}}^{d-1}
    \sqrt{\prod_{i=1}^{d-1}\frac{1}{p_i}}  \log\prn*{1/2\delta}
     .
\end{align*}
Since bound on the stopping time holds on average over instances $M_{0, \obsmap}$, there must be at least one MDP instance for which the expected stopping time adheres to this lower-bound. This proves the desired adaptive lower bound. 
\end{proof}

\begin{lemma}\label{lem:G_ub}
Let the progression probabilities $p_i = p$ for all $i \in [d-1]$ where $p \in (0, 1)$. Further, let $G$ denote the time step within the episode at which a goal step is reached. For any $\delta \in (0, 1)$
\begin{align*}
    \Pr\prn[\Big]{G \leq \frac{2d}{p} \ln  \frac{1}{\delta}} \geq 1 - \delta~.
\end{align*}
\end{lemma}
\begin{proof}
The event that $G$ is at least $n+1$ is equivalent to at most $d-2$ state progressions within $n$ trials which each happen with probability $p$. Let $X_i \in \{0, 1\}$ be the indicator for a state progression at time $i$. Then 
\begin{align*}
    \Pr(G \geq n + 1) 
    &= \Pr\prn[\Big]{\sum_{i=1}^{n} X_i \leq d-2} =
    \Pr\prn[\Big]{\sum_{i=1}^{n} X_i \leq np \prn[\big]{1 - \prn[\big]{\frac{d-2}{np} - 1}}}\\
    &\leq \exp\prn[\Big]{ - \frac{np}{2} \prn[\Big]{\frac{d-2}{np} - 1}^2}
\end{align*}
by a multiplicative Chernoff bound. This yields for all $\delta \in (0, 1)$
\begin{align*}
    \Pr\prn[\Big]{G \geq \frac{2d}{ p }\ln  1 / \delta} \geq 1 - \delta~. 
\end{align*} 
\end{proof}

\subsection{Change of observation distributions}

\begin{lemma}\label{lem:uniform_observation_conditional}
Let $\cF_{i, h-1} = \sigma\prn*{\cF_{i-1}, \{X_{i, h'}, A_{i, h'}, R_{i, h'}\}_{h' \in [h-1]}}$ be the sigma-field of everything observable up to before the $h$'th observation in episode $i$.
Then
\begin{align*}
    \nrm{\Pr_{0, \pi^*}\prn*{X_{i, h} | \cF_{i, h-1}} - \Pr_{\textrm{unif}}(X_{i,h} )}_1 
    &\leq 2 \frac{|\cS| H i}{N}\\
        \nrm{\Pr_{\pi^*}\prn*{X_{i, h} | \cF_{i, h-1}} - \Pr_{\textrm{unif}}(X_{i,h} )}_1 
    &\leq 2 \frac{|\cS| H i}{N}\\
      \nrm{\Pr_{0}\prn*{X_{i, h} | \cF_{i, h-1}} - \Pr_{\textrm{unif}}(X_{i,h} )}_1 
    &\leq 2 \frac{|\cS| H i}{N}~,
\end{align*}
where $\Pr_{\textrm{unif}}(X_{i,h} )$ is the uniform distribution over all possible observations $\cX$.
\end{lemma}

\begin{proof}
We prove the statement for $\Pr_{0, \pi^*}$ but the others can be proven analogously.

Let $\cF_{i, h-1} = \sigma\prn*{\cF{i-1}, \{X_{i, h'}, A_{i, h'}, R_{i, h'}\}_{h' \in [h-1]}}$ be the sigma-field of everything observable up to before the $h$'th observation in episode $i$. Further,  $\cF_{i, h-1}' = \sigma\prn*{\cF_{i,h-1}, \{S_{k, l}\}_{k \in [i-1], l \in [H]},
 \{S_{i, l}\}_{l \in [h]}}$ is the sigma-field that in addition includes all latent state labels up to $S_{i,h}$.
 
 Since $\cF'_{i,h}$ determines the latent state mapping $\obsmap$ for the observations encountered so far but all assignment of the remaining observations remains equally likely, we can write the conditional distribution of observation $X_{i,h}$ in closed form as
\begin{align*}
    \Pr_{0, \pi^*}\prn*{X_{i, h} = x | \cF'_{i, h-1}} 
    = \begin{cases}
    \frac{1}{N/ |\cS|} & \textrm{ if } x \in \cX_{\textrm{obs}}^{s}\\
    0 & \textrm{ if } x \in \cX_{\textrm{obs}} \setminus \cX_{\textrm{obs}}^{s}\\
    \prn[\big]{1 - \frac{|\cX_{\textrm{obs}}^{s}|}{N/ |\cS|}} \frac{1}{N - |\cX_{\textrm{obs}}|} & \textrm{ if } x \in \cX \setminus \cX_{\textrm{obs}}\\
    \end{cases}
\end{align*}
where $\cX_{\textrm{obs}}$ are all observations encountered so far and 
$\cX_{\textrm{obs}}^{s}$ are all observations encountered in $S_{i,h}$ so far.
Now
\begin{align*}
    \abs*{\Pr_{0, \pi^*}\prn*{X_{i, h} = x | \cF'_{i, h-1}} - \Pr_{\textrm{unif}}(X_{i,h} = x)} = 
    \begin{cases}\frac{1}{N/ |\cS|} - \frac{1}{N} 
     & \textrm{ if } x \in \cX_{\textrm{obs}}^{s}\\
    \frac{1}{N} & \textrm{ if } x \in \cX_{\textrm{obs}} \setminus \cX_{\textrm{obs}}^{s}\\
    \abs[\Big]{\frac{1}{N} - \prn[\big]{1 - \frac{|\cX_{\textrm{obs}}^{s}|}{N/ |\cS|}} \frac{1}{N - |\cX_{\textrm{obs}}|}} & \textrm{ if } x \in \cX \setminus \cX_{\textrm{obs}}\\
    \end{cases}
\end{align*}
and thus
\begin{align*}
    \nrm*{\Pr_{0, \pi^*}\prn*{X_{i, h} | \cF'_{i, h-1}} - \Pr_{\textrm{unif}}(X_{i,h} )}_1 
    \leq  \frac{2\max\{|\cS||\cX_{\textrm{obs}}^{s}| , |\cX_{\textrm{obs}}|\}}{N}
    \leq \frac{2|\cS|T_{\max} H}{N}~.
\end{align*}
Since $\Pr_{0, \pi^*}(X_{i,h} | \cF_{i, h-1}) = \bbE_{0, \pi^*}(\Pr_{0, \pi^*} (X_{i,h}  | \cF_{i, h-1}') | \cF_{i, h-1})$ by marginalization, we also have 
\begin{align*}
    \nrm*{\Pr_{0, \pi^*}\prn*{X_{i, h} | \cF_{i, h-1}} - \Pr_{\textrm{unif}}(X_{i,h} )}_1 
    \leq \frac{2 |\cS|T_{\max} H}{N}
\end{align*}
which means that as long as $T_{\max} \ll N$, the conditional distribution of the current observations remains close to $\textrm{Uniform}(\cX)$.
\end{proof}

\begin{lemma}\label{lem:uniform_observations_cond}
Let $\cF_{i} = \sigma\prn[\big]{\{X_{k, h'}, A_{k, h'}, R_{k, h'}\}_{h' \in [H], k \in [i]}}$ denote the natural filtration at the end of episode $i$.
Then
\begin{align*}
        \nrm*{\Pr_{\pi^*}\prn*{X_{i, 1:H} | \cF_{i-1}} - \Pr_{\textrm{unif}}(X_{i,1:H} )}_1 
    &\leq  \frac{2|\cS| H^2 i}{N}\\
    \nrm*{\Pr_{0}\prn*{X_{i, 1:H} | \cF_{i-1}} - \Pr_{\textrm{unif}}(X_{i,1:H} )}_1 
    &\leq  \frac{2|\cS| H^2 i}{N}
    ~,
\end{align*}
where $\Pr_{\textrm{unif}}(X_{i,1:H})$ is the product distribution of uniform distributions over all possible observations $\cX$.
\end{lemma}
\begin{proof}
The random variables $X_{i,1:H}$ are $\cF_{i}$-measurable. We can therefore consider any event $A \in \cF_{i}$ and show that
\begin{align*}
    \abs*{\Pr_{\pi^*}\prn*{A | \cF_{i-1}} - \Pr_{\textrm{unif}}(A)} &\leq  \frac{|\cS| H^2 i}{N}\\
     \abs*{\Pr_{0}\prn*{A | \cF_{i-1}} - \Pr_{\textrm{unif}}(A)} &\leq  \frac{|\cS| H^2 i}{N}
\end{align*}
analogously to \pref{lem:unif_obs_1} below. The result then follows immediately from the identity of $\ell_1$ norm and total variation. 
\end{proof}

\begin{lemma}\label{lem:unif_obs_1}
Let $A \in \cF_{T_{\max}}$ be any event that is $\cF_{T_{\max}}$-measurable, where $\cF_{T_{\max}}$ is the sigma-field induced by everything up to $T_{\max}$ episodes.
Then
\begin{align*}
    \abs*{\Pr_{0}\prn*{A} - \Pr_{0, \pi^*}\prn*{A}} &\leq \obserr{T_{\max}}
    = \frac{4T_{\max}^2 H^2 |\cS|}{N}
    ~.
\end{align*}
\end{lemma}
\begin{proof}
Denote by $\Pr_{0, \pi^*}^{t, h, \textrm{unif}}$ the distribution that matches $\Pr_{0, \pi^*}$ but where all observations after the $h$'th observation in episode $t$ are drawn uniformly random from $\cX$.
First, since $A \in \cF_{T_{\max}}$ and $\Pr_{0, \pi^*}(B) = \Pr_{0, \pi^*}^{T_{\max}, \textrm{unif}}(B)$ for all events $B \in \cF_{T_{\max}}$, we have
\begin{align*}
    \Pr_{0, \pi^*}(A) &= \Pr_{0, \pi^*}^{T_{\max}, H, \textrm{unif}}(A)~.
\end{align*}
We now peel off one time step at a time by showing that $\abs*{\Pr_{0, \pi^*}^{t, h, \textrm{unif}}(A) - \Pr_{0, \pi^*}^{t, h+1, \textrm{unif}}(A)} \leq \frac{2|\cS|H t}{N}$. By the definition of these probabilities, the following chain of equations holds:
\begin{align*}
     &\Pr_{0, \pi^*}^{t, h, \textrm{unif}}(A)\\
     &= \bbE_{0, \pi^*}^{t, h, \textrm{unif}}\brk*{\Pr_{0, \pi^*}^{t, h, \textrm{unif}}(A | X_{t, h}, \cF_{t, h-1})}
     = \bbE_{0, \pi^*}^{t, h, \textrm{unif}}\brk[\Big]{\Pr_{0, \pi^*}^{t, h-1, \textrm{unif}}(A | X_{t, h}, \cF_{t, h-1})}\\
     &
     = \bbE_{0, \pi^*}^{t, h, \textrm{unif}}\brk[\Big]{\sum_{x \in \cX} \Pr_{0, \pi^*}^{t, h-1, \textrm{unif}}(A | X_{t, h} = x, \cF_{t, h-1}) 
     \Pr_{0, \pi^*}^{t, h, \textrm{unif}}(X_{t, h} = x | \cF_{t, h-1})
     }\\
          &
     = \bbE_{0, \pi^*}^{t, h, \textrm{unif}}\brk[\Big]{\sum_{x \in \cX} \Pr_{0, \pi^*}^{t, h-1, \textrm{unif}}(A | X_{t, h} = x, \cF_{t, h-1}) 
     \Pr_{0, \pi^*}^{t, h-1, \textrm{unif}}(X_{t, h} = x | \cF_{t, h-1})
     }\\
     &\qquad + \bbE_{0, \pi^*}^{t, h, \textrm{unif}}\brk[\Big]{\sum_{x \in \cX} \Pr_{0, \pi^*}^{t, h-1, \textrm{unif}}(A | X_{t, h} = x, \cF_{t, h-1}) 
     \prn*{\Pr_{0, \pi^*}^{t, h, \textrm{unif}}(X_{t, h} = x | \cF_{t, h-1}) - \Pr_{0, \pi^*}^{t, h-1, \textrm{unif}}(X_{t, h} = x | \cF_{t, h-1})}
     }\\
               &
     = \Pr_{0, \pi^*}^{t, h-1, \textrm{unif}}(A) + \bbE_{0, \pi^*}^{t, h, \textrm{unif}}\brk[\Big]{\sum_{x \in \cX} \Pr_{0, \pi^*}^{t, h-1, \textrm{unif}}(A | X_{t, h} = x, \cF_{t, h-1}) 
     \prn*{\Pr_{0, \pi^*}(X_{t, h} = x | \cF_{t, h-1}) - \Pr_{\textrm{unif}}(X_{t, h} = x)}
     }~.
\end{align*}
Thus, by rearranging terms, we have
\begin{align*} 
    \abs{\Pr_{0, \pi^*}^{t, h, \textrm{unif}}(A) - \Pr_{0, \pi^*}^{t, h-1, \textrm{unif}}(A)} \leq 
    \bbE_{0, \pi^*}^{t, h, \textrm{unif}}\brk*{
     \nrm*{\Pr_{0, \pi^*}(X_{t, h}  | \cF_{t, h-1}) - \Pr_{\textrm{unif}}(X_{t, h})}_1
     }
     \leq 
    \frac{2|\cS|H t}{N}~,
\end{align*}
where the last inequality follows from \pref{lem:uniform_observation_conditional}.
We now consider
\begin{align*}
    \Pr_{0, \pi^*}\prn*{A} - \Pr_{0, \pi^*}^{1, 0, \textrm{unif}}(A)
    &= \Pr_{0, \pi^*}^{T_{\max}, H, \textrm{unif}}(A) - \Pr_{0, \pi^*}^{1, 0, \textrm{unif}}(A)\\
    &= \sum_{h=1}^H \sum_{t =1}^{T_{\max}}
    \Pr_{0, \pi^*}^{t, h, \textrm{unif}}(A) - \Pr_{0, \pi^*}^{t, h-1, \textrm{unif}}(A)
\end{align*}
where $\Pr_{0, \pi^*}^{t, 0, \textrm{unif}} = \Pr_{0, \pi^*}^{t-1, H, \textrm{unif}}$ and
apply the bound to each term to arrive at
\begin{align*}
    \abs{\Pr_{0, \pi^*}\prn*{A} - \Pr_{0, \pi^*}^{1, 0, \textrm{unif}}(A)} 
    \leq \frac{2T_{\max}^2 H^2 |\cS|}{N}~.
\end{align*}
Note that for the distribution $\Pr_{0, \pi^*}^{1, 0, \textrm{unif}}$ all observations are drawn uniformly at random and the rewards do not depend on the actions. Thus,  $\Pr_{0, \pi^*}^{1, 0, \textrm{unif}} = \Pr_{0}^{1, 0, \textrm{unif}}$ and we can derive analogously to above that
\begin{align*}
\abs{\Pr_{0}\prn*{A} - \Pr_{0}^{1, 0, \textrm{unif}}(A)} 
   = \abs{\Pr_{0}\prn*{A} - \Pr_{0, \pi^*}^{1, 0, \textrm{unif}}(A)} 
    \leq \frac{2T_{\max}^2 H^2 |\cS|}{N}~.
\end{align*}
Combining both bounds using the triangle inequality yields
the desired statement 
\begin{align*}
    \abs*{\Pr_{0}\prn*{A} - \Pr_{0, \pi^*}\prn*{A}} &\leq \frac{4T_{\max}^2 H^2 |\cS|}{N}~.
\end{align*}
\end{proof} 

\subsection{Bounds on expected policy matches per episode}
\begin{lemma}[Bound on expected policy matches with equal $p_i$]
\label{lem:packing_lb}
Let $\pi \colon \cX \mapsto \cA$ any policy (that does not need to be in the given policy class $\Pi$) and $H \geq 62 d$. Further, set
$    |\Pi| = \prn*{H/d}^d,$ and $p = \frac{d}{H}~.$ 
Then
\begin{align*}
    \bbE_{\textrm{unif}}\brk[\Big]{
    \sum_{\pi^*} \indicator{\pi^*(X_{1:G} = \pi(X_{1:G})}
    } \leq  \prn[\Big]{41 \log \prn*{H/d}}^d H + 2 ~,
\end{align*}
where $\Pr_{\textrm{unif}}$ draws all $H$ observations $X_{1:H}$ i.i.d. from $\operatorname{Uniform}(\cX)$ and $G$ as usual.
\end{lemma}
\begin{proof}For any $h \in \bbN$ with $d \leq h \leq H$, the following holds:
\begin{align*}
    &\bbE_{\textrm{unif}}\brk[\Big]{
    \sum_{\pi^*} \indicator{\pi^*(X_{1:G} = \pi(X_{1:G})}
    }\\ 
    &= 
    \Pr(G \leq  h) \bbE_{\textrm{unif}}\brk[\Big]{
    \sum_{\pi^*} \indicator{\pi^*(X_{1:G} = \pi(X_{1:G})} ~ \mid G \leq h
    }\\
    &\quad+\Pr(G > h) \bbE_{\textrm{unif}}\brk[\Big]{
    \sum_{\pi^*} \indicator{\pi^*(X_{1:G} = \pi(X_{1:G})} ~ \mid G > h
    } \\
    &\leq 
    |\Pi| \Pr(G \leq h)  + \bbE_{\textrm{unif}}\brk[\Big]{
    \sum_{\pi^*} \indicator{\pi^*(X_{1:h} = \pi(X_{1:h})}}
    \\
    & \leq
    (h-d+1) |\Pi| \prn[\Big]{\frac{2peh}{d}}^{d-1} + 
    \bbE_{\textrm{unif}}\brk[\Big]{
    \sum_{\pi^*} \indicator{\pi^*(X_{1:h} = \pi(X_{1:h})}}
    ~,
\end{align*}
where the last inequality applies \pref{lem:G_lb}.
Note that by the construction of $\Pi$, there can only be one policy in $\Pi$ which agrees with $\pi$ on more than $\frac{7}{8}N$ observations in $\cX$. 
With all other policies, $\pi$ has to disagree on at least $1/8$ fraction of all possible observations.
To see this, assume that there were two policies $\pi_1 \neq \pi_2$ in $\Pi$ for which $\nrm*{\pi - \pi_i} < N/8$. Then by triangle inequality $\nrm*{\pi_1 - \pi_2} < N/4$ which contradicts the construction of $\Pi$.
Thus, we can further bound the quantity of interest as
\begin{align}
    \bbE_{\textrm{unif}}\brk[\Big]{
    \sum_{\pi^*} \indicator{\pi^*(X_{1:G} = \pi(X_{1:G})}
    }
    &\leq (h-d+1) |\Pi| \prn[\Big]{\frac{2peh}{d}}^{d-1} + 1 + (|\Pi| - 1) \prn*{\frac{7}{8}}^h\nonumber
    \\
    &\leq  |\Pi| h\prn[\Big]{\frac{2peh}{d}}^{d-1} + 1 + |\Pi| \exp\prn*{- h \log(8/7)}~.\label{eqn:expr_bound1}
\end{align}
We use the worst-case choices of $|\Pi|$, $p$ and $h$ as
\begin{align*}
    |\Pi| &= \prn*{\frac{H}{d}}^d,
    &
    h &= \frac{\log|\Pi|}{\log(8/7)},
    &
        p &= \frac{d}{H}~.
\end{align*}
Since the RHS of \pref{eqn:expr_bound1} is non-decreasing in $p$ and $\Pi$, a bound with these exact values is also valid when $p$ and $\Pi$ are smaller.
Note that under these choices $H \geq 62 d$ which implies $\frac{H}{d} \geq 15 \ln \frac{H}{d}$ is sufficient for $h \leq H-1$.

With these choices, the last term of \pref{eqn:expr_bound1} is bounded by $1$, i.e.,  $|\Pi| \exp\prn*{- h \log\frac{8}{7}} = 1$, and the first term is bounded as 
\begin{align*} 
    |\Pi| h\prn[\big]{\frac{2peh}{d}}^{d-1} 
    &= \prn[\big]{\frac{2e}{\log(8/7)}}^d 
    \prn[\big]{\frac{p}{d}}^{d-1} |\Pi| (\log |\Pi|) ^d 
    \\ 
    &\leq 41^d \prn[\big]{\frac{Hd}{d} \log\frac{H}{d} }^d  \prn[\big]{\frac{p}{d}}^{d-1} 
    \leq \prn[\big]{41 \log \prn*{H/d}}^d H~. 
\end{align*}
\end{proof}

\begin{lemma}[Bound on expected policy matches with arbitrary $p_i$]
\label{lem:packing_lb_pi}
Let the probability for progressions be $p_1, \dots, p_{d-1}$ and let $\pi \colon \cX \mapsto \cA$ any policy (that does not need to be in the given policy class $\Pi$). Further, assume that 
\begin{align*}
    |\Pi| \leq \exp\prn[\Big]{\frac{H}{2}\log \frac{8}{7}}.
\end{align*}

Then
\begin{align*}
    \bbE_{\textrm{unif}}\brk[\Big]{
    \sum_{\pi^*} \indicator{\pi^*(X_{1:G} = \pi(X_{1:G})}
    } \leq   2 + |\Pi|  \prod_{i=1}^{d-1}\prn[\Big]{ p_i \frac{\log|\Pi|}{\log(8/7)} } ~,
\end{align*}
where $\Pr_{\textrm{unif}}$ draws all $H$ observations $X_{1:H}$ i.i.d. from $\operatorname{Uniform}(\cX)$ and $G$ as usual.
\end{lemma}
\begin{proof}For any $h \in \bbN$ with $d \leq h \leq H$, the following holds:
\begin{align*}
    \bbE_{\textrm{unif}}\brk[\Big]{
    \sum_{\pi^*} \indicator{\pi^*(X_{1:G} = \pi(X_{1:G})}
    } 
    &= 
    \Pr(G \leq  h) \bbE_{\textrm{unif}}\brk[\Big]{
    \sum_{\pi^*} \indicator{\pi^*(X_{1:G} = \pi(X_{1:G})} ~ \mid G \leq h
    }\\
    &\quad+\Pr(G > h) \bbE_{\textrm{unif}}\brk[\Big]{
    \sum_{\pi^*} \indicator{\pi^*(X_{1:G} = \pi(X_{1:G})} ~ \mid G > h
    } \\
    &\leq 
    |\Pi| \Pr(G \leq h)  + \bbE_{\textrm{unif}}\brk[\Big]{
    \sum_{\pi^*} \indicator{\pi^*(X_{1:h} = \pi(X_{1:h})}}
    \\
    & \leq
     |\Pi| h^{d-1} \prod_{i=1}^{d-1} p_i + 
    \bbE_{\textrm{unif}}\brk[\Big]{
    \sum_{\pi^*} \indicator{\pi^*(X_{1:h} = \pi(X_{1:h})}}
    ~,
\end{align*}
where the last inequality applies \pref{lem:G_lb_pi}.
Note that by the construction of $\Pi$, there can only be one policy in $\Pi$ which agrees with $\pi$ on more than $\frac{7}{8}N$ observations in $\cX$. 
With all other policies, $\pi$ has to disagree on at least $1/8$ fraction of all possible observations.
To see this, assume that there were two policies $\pi_1 \neq \pi_2$ in $\Pi$ for which $\nrm*{\pi - \pi_i} < N/8$. Then by triangle inequality $\nrm*{\pi_1 - \pi_2} < N/4$ which contradicts the construction of $\Pi$.
Thus, we can further bound the quantity of interest as
\begin{align}
    \bbE_{\textrm{unif}}\brk[\Big]{
    \sum_{\pi^*} \indicator{\pi^*(X_{1:G} = \pi(X_{1:G})}
    }
    &\leq |\Pi| h^{d-1} \prod_{i=1}^{d-1} p_i + 1 + (|\Pi| - 1) \prn*{\frac{7}{8}}^h\nonumber
    \\
    &\leq  |\Pi| h^{d-1} \prod_{i=1}^{d-1} p_i + 1 + |\Pi| \exp\prn*{- h \log(8/7)}~.\label{eqn:expr_bound1_pi}
\end{align} 
We now set $h = \frac{\log|\Pi|}{\log(8/7)}$ which gives
\begin{align}
    \bbE_{\textrm{unif}}\brk[\Big]{
    \sum_{\pi^*} \indicator{\pi^*(X_{1:G} = \pi(X_{1:G})}
    }
    &\leq  2 + |\Pi|  \prod_{i=1}^{d-1}\prn[\Big]{ p_i \frac{\log|\Pi|}{\log(8/7)} } ~.
\end{align}
\end{proof}

\begin{lemma} 
\label{lem:G_lb} 
Let the progression probabilities $p_i = p$ for all $i \in [d-1]$ where $p \in (0, 1)$.
The time step $G$ within the episode at which a goal step is reached satisfies
\begin{align*}
\Pr \prn*{G \leq h} &\leq  (h-d+1) \prn[\Big]{\frac{2peh}{d}}^{d-1}. 
\end{align*}
\end{lemma} 
\begin{proof} For $G = i$, there must be exactly $d-1$ progressions in the $i-1$ previous time steps, each happening with probability $p$. Therefore
\begin{align*}
\Pr(G = i) = {i - 1 \choose d - 1} p^{d - 1} (1 - p)^{i - d}. 
\end{align*}
Thus, 
\begin{align*}
\Pr \prn*{G \leq h}   &= \sum_{i=d}^{h} \Pr(G = i)   
	=  \sum_{i=d}^{h} {i - 1 \choose d - 1} p^{d - 1} (1 - p)^{h - d }  \\
	&\leq  
	\sum_{i=d}^{h} {i - 1 \choose d - 1} p^{d - 1}  
	\leq \sum_{i=d}^{h} \prn[\Big]{\frac{e(i-1)}{d-1}}^{d-1} p^{d - 1}, \\
	&\leq (h-d+1) \prn[\Big]{\frac{2peh}{d}}^{d-1} , 
\end{align*} where the first inequality in the above is given by ignoring terms smaller than one, and the second inequality is due to the fact that any $n, k$, we have ${n \choose k} \leq \prn*{en / k}^k$ for $0 \leq k \leq n$. 
\end{proof} 

\begin{lemma} 
\label{lem:G_lb_pi} 
Let the probability for progressions be $p_1, \dots, p_{d-1}$. 
The time step $G$ within the episode at which a goal step is reached then satisfies
\begin{align*}
\Pr \prn*{G \leq h} & \leq  h^{d-1} \prod_{i=1}^{d-1} p_i
\end{align*}
\end{lemma} 
\begin{proof}
For the event $G \leq h$ to happen, there must have been a progression in each of the $d-1$ states within $h$ trials. Therefore
\begin{align*}
\Pr \prn*{G \leq h}  
&\leq  \prod_{i=1}^{d-1} (1 - (1 - p_i)^{h-1})
 \leq \prod_{i=1}^{d-1} ((h-1) p_i) \leq \prod_{i=1}^{d-1} (h p_i)~.
\end{align*} 
\end{proof} 
 
\end{document}